\documentclass{article}

\usepackage{microtype}
\usepackage{graphicx}
\usepackage{subfigure}
\usepackage{booktabs} 

\usepackage{hyperref}

\usepackage[accepted]{icml2024}

\usepackage{amsmath}
\usepackage{amssymb}
\usepackage{mathtools}
\usepackage{amsthm}

\usepackage[capitalize,noabbrev]{cleveref}

\usepackage{macros}

\theoremstyle{plain}
\newtheorem{theorem}{Theorem}[section]

\newtheorem{lemma}[theorem]{Lemma}

\theoremstyle{definition}
\newtheorem{definition}[theorem]{Definition}
\newtheorem{assumption}[theorem]{Assumption}
\theoremstyle{remark}
\newtheorem{remark}[theorem]{Remark}

\usepackage[textsize=tiny]{todonotes}

\icmltitlerunning{SaVeR: Optimal Data Collection Strategy for Safe Policy Evaluation in Tabular MDP}

\begin{document}

\twocolumn[
\icmltitle{SaVeR: Optimal Data Collection Strategy for Safe Policy Evaluation in Tabular MDP}



\icmlsetsymbol{equal}{*}

\begin{icmlauthorlist}
\icmlauthor{Subhojyoti Mukherjee}{yyy}
\icmlauthor{Josiah P. Hanna}{comp}
\icmlauthor{Robert Nowak}{yyy}
\end{icmlauthorlist}

\icmlaffiliation{yyy}{Department of Electrical and Computer Engineering, University of Wisconsin-Madison, Madison, USA}
\icmlaffiliation{comp}{Computer Sciences Department, University of Wisconsin-Madison, Madison, USA}

\icmlcorrespondingauthor{Subhojyoti Mukherjee}{smukherjee27@wisc.edu}

\icmlkeywords{Policy Evaluation, Safe Exploration, Constraint MDP}

\vskip 0.3in
]



\printAffiliationsAndNotice{}  

\vspace*{-2.0em}
\begin{abstract}
In this paper, we study safe data collection for the purpose of policy evaluation in tabular Markov decision processes (MDPs). In policy evaluation, we are given a \textit{target} policy and asked to estimate the expected cumulative reward it will obtain. Policy evaluation requires data and we are interested in the question of what \textit{behavior} policy should collect the data for the most accurate evaluation of the target policy. While prior work has considered behavior policy selection, in this paper, we additionally consider a safety constraint on the behavior policy. Namely, we assume there exists a known default policy that incurs a particular expected cost when run and we enforce that the cumulative cost of all behavior policies ran is better than a constant factor of the cost that would be incurred had we always run the default policy. We first show that there exists a class of intractable MDPs where no safe oracle algorithm with knowledge about problem parameters can efficiently collect data and satisfy the safety constraints. We then define the tractability condition for an MDP  such that a safe oracle algorithm can efficiently collect data and using that we prove the first lower bound for this setting. We then introduce an algorithm \sav\ for this problem that approximates the safe oracle algorithm and bound the finite-sample mean squared error of the algorithm while ensuring it satisfies the safety constraint. Finally, we show in simulations that \sav\ produces low MSE policy evaluation while satisfying the safety constraint.
\end{abstract}

\section{Introduction}
\label{sec:intro}
Reinforcement learning has emerged as a powerful tool for decision-making in a wide range of applications, from robotics \citep{ibarz2021train, agarwal2022legged} and game-playing \citep{szita2012reinforcement} to autonomous driving \citep{kiran2021deep}, web-marketing \citep{bottou2013counterfactual}, healthcare \citep{fischer2018reinforcement,yu2019reinforcement} and finance \citep{hambly2021recent}. However, in these applications, it is often necessary to first evaluate the decision-making policy before its long-term deployment in the real world. 
In fact, policy evaluation is a critical step in reinforcement learning, as it allows us to assess the quality of a learned policy and to check whether it can truly achieve the desired goal for the target task.
One potential solution to this issue is off-policy evaluation (OPE) \citep{dudik2014doubly, li2015toward, swaminathan2017off, wang2017optimal, su2020doubly, kallus2021optimal, cai2021deep}. However, for OPE estimators there is no control over how the static dataset is generated, which could result in low accuracy estimates.

Hence, a natural idea is to actively gather the dataset using an adaptive behavior policy and thus increase accuracy in the evaluation of the target policy's value.  
In many real-world settings, the behavior policy itself must satisfy some side constraints (specific to the industry) \citep{wu2016conservative} or safety constraints \citep{wan2022safe} while collecting the dataset. 
For instance, in web marketing, it is common to run an A/B test with safety constraints over a subset of all users before a potential new policy is deployed for all users \citep{kohavi2017online, tucker2022variance}.
While testing autonomous vehicles it is quite natural to incorporate safety constraints in the behavior policy. So it is of great practical importance to ensure that our data collection rule is safe \citep{zhu2022safe}.

In this paper, we consider the question of optimal data collection for policy evaluation under safety constraints in the tabular reinforcement learning (RL) setting. 
Consider the following scenario that could arise in web marketing. Suppose we have a policy learned from offline data that has never been run in a real application. Moreover, we want this learned policy to be at least as good as a baseline policy that is already deployed in the application \citep{wu2016conservative, zhu2021safe, zhu2022safe}. Off-policy evaluation often has high variance, so engineers may want to have some controlled deployment where the learned policy only makes decisions for some users before letting the policy make decisions for all users. We are motivated by how to make this controlled deployment as data-efficient and safe as possible. By \textit{safe}, we mean that we want the expected return seen during data collection to remain close to the expected return under the baseline policy. A similar motivation can be found in \citet{tucker2022variance}. In this paper, we focus on finding a behavior policy that produces a minimal variance estimate while remaining safe.
We can state this formally as follows: 
We are given a target policy, $\pi$, for which we want to estimate its value denoted by $V^{\pi}(s_1)$.
To estimate $V^{\pi}(s_1)$ we will generate a set of $K$ episodes where each episodic interaction ends after $L$ timesteps. We denote the total available budget of samples as $n=KL$. Each episode is generated by following some behavior policy and collect the dataset $\D$. Let $Y^{\pi}_n(s_1)$ be the estimate of $V^{\pi}(s_1)$ computed from $\D$.
%
%
Then our objective is to determine a sequence of behavior policies that minimizes error in the estimation of $V^{\pi}(s_1)$ defined as $\E_{\D}[(Y^{\pi}_n(s_1) - V^{\pi}(s_1))^{2}]$ subject to a 
\textit{safety constraint} on the cost-value of the behavior policies (to be defined later) that must hold with high probability.

There is a growing body of literature studying this important problem of data collection for policy evaluation in both constrained and unconstrained setups.
The work of \citet{antos2008active,carpentier2011finite,carpentier2012minimax,carpentier2015adaptive, fontaine2021online, mukherjee2022revar, mukherjee2024speed} studies this problem in the bandit setting without any constraints under the finite sample regime. 
A common metric of performance that these works consider is the difference between the loss of the agnostic algorithm that does not know problem-dependent parameters, and the oracle loss (which has access to problem-dependent parameters).
This metric is termed \textit{regret} and these works show that in the bandit setting the regret of the agnostic algorithm scales as $\tO(n^{-3/2})$ where $\tO(\cdot)$ hides log factors.
One might be tempted to just run the target policy $\pi$, build $\D$ and then estimate $Y^{\pi}_n(s_1)$. This is called \textit{on-policy} data collection. However, these works show that the on-policy regret degrades at a much slower rate of $\tO(n^{-1})$ compared to active agnostic algorithms. 
%
%
Hence, a natural question arises, can we achieve similar performance for policy evaluation in the MDP setup under a finite sample regime even when we must conform to safety constraints? 
Thus, the goal of our work is to answer the following questions:
\begin{quote}
    \textit{1) Is there a class of MDPs where it is possible to  incur a regret that degrades at a faster rate than $\tO(n^{-1})$?} while satisfying safety constraints?\\\\
    \textit{2) If the answer is yes to (1), can we design an adaptive algorithm (for this class of MDPs) to collect data for policy evaluation that does not violate the safety constraints (in expectation), and its regret degrades at a faster rate than $\tO(n^{-1})$?}
\end{quote}
In this paper, we answer these questions affirmatively. Regarding the first question, we state the tractability condition on the class of MDPs which enables the optimal behavior policy to gather data for policy evaluation without violating the safety constraint and suffer a regret of $\tO(n^{-3/2})$. 
This condition leads to the first lower bound for this setting.

We also note that safe data collection for policy evaluation has also been studied in the bandit setting in \citet{zhu2021safe,zhu2022safe}. However, these works provide asymptotic guarantees whereas we are the first to provide finite-time regret guarantees when per-step constraints must be maintained in expectation. We also show that in the bandit setup, our method empirically outperforms the adaptive importance sampling based algorithms in these works.
Our formulation is also related to constrained MDPs though we specify that the constraint must be satisfied \textit{throughout learning} and not just by the final policy \citep{efroni2020exploration,vaswani2022near}. We discuss further related works in \Cref{app:related-app}.

Our main contributions are as follows:
\par
\textbf{(1)} We formulate the problem of safe data collection for policy evaluation. We introduce the safety constraint such that at the end of $n$ trajectories, the cumulative cost is above a constant factor of the baseline cost. To our knowledge, this is the first work to study this setting under such a safety constraint in the MDP setup with the goal of minimizing the estimate of the MSE of the target policy's expected reward. 
\par
\textbf{(2)} We then show that even in the special case of finite tree-structured MDPs the safe data collection for policy evaluation can be intractable. Then we come up with a condition on MDPs that enables any behavior policy to collect data without violating safety constraints. We also provide the first regret lower bound for the bandit and Tree MDP setting and show that it scales with $\Omega(n^{-3/2})$.
\par
\textbf{(3)} We then consider an oracle strategy that knows the reward variances (problem-dependent parameter) of the reward distributions and derives its sampling strategy.
We then introduce the agnostic algorithm \textbf{Sa}fe \textbf{V}arianc\textbf{e} \textbf{R}eduction (\sav\!) that does not know the problem-dependent parameters and show that its regret scales as $\tO(n^{-3/2})$. We evaluate its performance against other baseline approaches and show that \sav\ reduces MSE faster while satisfying the safety constraint.

\section{Preliminaries}
\label{sec:prelim}
%

We consider the standard finite-horizon Markov Decision process, $\M$, with both a reward and constraint function.
Formally, $\M$, is a tuple $(\S, \A, P, R, C, \gamma, d_0, L)$, where $\S$ is a finite set of states, $\A$ is a finite set of actions, $P: \S \times \A \times \S \rightarrow [0,1]$ is a state transition function, $R$ is the reward function (formalized below), $C$ is the constraint function (formalized below), $\gamma \in[0,1)$ is the discount factor, $d_{0}$ is the starting state distribution, and $L$ is the maximum episode length.
A (stationary) policy, $\pi: \S \times \A \rightarrow [0,1]$, is a probability distribution over actions conditioned on a given state.
We assume data can only be collected through episodic interaction: an agent begins in state $s_1 \sim d_0$ and then at each step $t$ takes an action $a_t \sim \pi(\cdot | s_t)$ and proceeds to state $s_{t+1} \sim P(\cdot | s_t, a_t)$.
%
%

When the agent takes an action, $a$, in state, $s$, it receives both a reward $R \sim R(s,a)$ and a constraint value $C \sim C(s,a)$.
%
%
We assume the transition model $P$ is known but the reward distributions and constraint values are unknown. 
%
%
%
%
We define the reward value of a policy as:
$
V^{\pi}(s_1) \coloneqq \E_\pi[\sum_{t=1}^{n} \gamma^{t-1} R_t]$, where $\E_\pi$ is the expectation w.r.t.\ trajectories sampled by following $\pi$ from the initial state $s_1$. 
We define a constraint-value of $\pi$ similarly:
$
V^{\pi}_c(s_1) \coloneqq \E_\pi[\sum_{t=1}^{n} \gamma^{t-1} C_t]$. 
For simplicity, let the initial state distribution has probability mass on a single state $s_1$. 

Our goal is to efficiently estimate $V^\pi(s_1)$ for a given policy $\pi$ and this estimation requires data from the environment MDP.
Past work has approached this problem by designing a sequence of \textit{behavior policies} which are ran to produce informative data for evaluating $\pi$.
However, in practical applications, it is often infeasible to simply run \textit{any} behavior policy as doing so may violate domain constraints.
We formalize this constraint by first assuming the existence of a safe \textit{baseline policy}, $\pi_0$ that provides an acceptable constraint-value $V_c^{\pi_0}(s_1)$.
Our objective is to determine a sequence of behavior policies, $\{\bb_1,..,\bb_K\}$, that will produce a set of $K$ episodes that lead to the most accurate estimate of $V^\pi(s_1)$ subject to the constraint that the cumulative expected constraint-value $V_c^{\bb}(s_1)$ always exceeds a fixed percentage of $V_c^{\pi_0}(s_1)$.
We consider the objective:
\begin{align}
\label{eq:safety-constraint-MDP}
    &\min_\textbf{b} \E_{\D}[\left(Y^{\pi}_n(s_1) - V^{\pi}(s_1)\right)^{2}]\quad \\
    \text { s.t.} &
    \sum_{k'=1}^k V^{\bb^{k'}}_c(s_1) \geq (1-\alpha)k V_c^{\pi_0}(s_1) 
    \text { for all } k \in[K] \nonumber
\end{align}
where $Y_n(s_1)$ is our estimate of $V^\pi(s_1)$, $\alpha \in (0,1]$ is the risk parameter, and the expectation is over the collected data set $\D$. We also make the following simplifying assumption. 
We assume $\pi_0$ is deterministic, i.e., will only select one action in any given state. W.l.o.g., we give this action the index $0$ and refer to it as the \textit{safe action}. The entire action set is $\A = \{0,1,…,A\}$. This assumption is reasonable in applications where existing, safe policies were created through non-learning methods or manually designed.

For analysis, we will estimate $V^\pi(s_1)$ with a certainty-equivalence estimator.
We define the random variable representing the estimated future reward from state $s$ at time-step $\ell$ as  
$Y^{\pi}_n(s,\ell) \!\coloneqq\! \sum_a \pi(a|s) \wmu_n(s,a) + \gamma \sum_{s'} \wP_n(s'|s,a) Y^{\pi}_n(s',\ell+1)$
where $Y^{\pi}_n(s,\ell+1) \!\!\coloneqq\!\! 0$ if $\ell\!\geq\! L$, and $\wmu_n(s,a)$ is an estimate of $\mu(s,a)$, 
both computed from $\mathcal{D}$.
Finally, the estimate of $V^{\pi}(s_1)$ is computed as $Y^{\pi}_n(s_1) \!\coloneqq\! \sum_{s} d_0(s_1) Y^{\pi}_n(s_1,0)$.
Note that the total available budget of samples is $n$.
We assume that there are $K$ episodes and each episodic interaction terminates in at most $L$ steps which implies $n=KL$.

We assume $V^{\bb}_c(s_1)$ is known for $\bb=\pi_0$ but not for any other policy. The constraint in \eqref{eq:safety-constraint-MDP} 
implies that the total constraint value over all deployed behavior policies should be above the total constraint value that can be obtained from the baseline policy $\pi_0$ till episode $k$ with high probability. 
Observe that small values of $\alpha$ force the learner to be highly conservative, whereas larger $\alpha$ values correspond to a weaker constraint. 
A similar setting has been studied for policy improvement by \citet{wu2016conservative, yang2021reduction} for a variety of sequential decision-making settings. However, our objective is policy evaluation and we formulate a more general safety constraint in terms of $C(\cdot)$ while these prior works define the constraint in terms of $R(\cdot)$.

Similar to the recent works of \citet{chowdhury2021reinforcement,ouhamma2023bilinear, agarwal2019reinforcement, lattimore2020bandit} we assume the reward function $R(s,a) = \N(\mu(s, a),\sigma^2(s, a))$, where $\N$ denotes a Gaussian distribution with mean $\mu(s,a)$ and variance $\sigma^2(s, a)$. Similarly we assume a constraint function $C(s,a) = \N(\mu^c(s, a),\sigma^{c,(2)}(s, a))$, where $\mu^c(s,a)$ and $\sigma^{c,(2)}(s, a)$ are the mean and variance of $\N(\cdot)$.
%
%
Note that this sub-Gaussian distribution assumption is required only for theoretical analysis, whereas our algorithm works for any bounded reward and cost functions.
We assume that we have bounded reward and constraint mean $\mu(s,a),\mu^c(s,a) \in [0,\eta]$ respectively. Finally, we define the MSE of a behavior policy $\bb$ for the target policy $\pi$ at the end of budget $n$ as
\begin{align}
    \L_n(\pi, \bb) = \E_{\D}[\left(Y^{\pi}_n(s_1) - V^{\pi}(s_1)\right)^{2}] \label{eq:mse-policy}
\end{align}
where the expectation is over dataset $\D$ which is collected by $\bb$. Our main objective is to minimize the cumulative regret $\cR_n$ subject to the safety constraint defined in \eqref{eq:safety-constraint-MDP}. To define $\cR_n$ we first define the MSE of a safe oracle behavior policy $\bb^k_*$ that collects the dataset $\D$ as $\L^*_n(\pi, \bb^k_*)$. We will formally describe such oracle policies in \Cref{sec:lower}. Then the regret $\cR_n$ is defined as 
\begin{align}
    \cR_n = \L_n(\pi, \bb) - \L^*_n(\pi, \bb^k_*). \label{eq:regret}
\end{align}

\section{Intractability and Lower Bounds}
\label{sec:lower}

In this section, we first define an oracle data collection strategy that ignores the constraints. We call this the unconstrained oracle. This oracle data collection algorithm can reach a regret bound of $\tO(n^{-3/2})$ in the unconstrained setting \citep{carpentier2012minimax, carpentier2015adaptive, mukherjee2022revar}.
We then show how data collection for policy evaluation under safety constraints in MDPs is challenging compared to standard policy improvement challenges in constrained MDPs \citep{efroni2020exploration,vaswani2022near} as well as safe data collection for policy evaluation in bandits \citep{zhu2021safe,wan2022safe,zhu2022safe}. To show this challenging aspect, we first discuss how the unconstrained oracle fails to satisfy the constraint and achieve the desired regret of $\widetilde{O}(n^{-3/2})$ in the constraint MDP setting. We then propose a safe variant of the oracle policy and finally, discuss a tractability condition that enables the safe oracle algorithm to achieve a regret bound of $\widetilde{O}(n^{-3/2})$.

\subsection{Unconstrained Oracle}
\label{sec:oracle-sampler}
In this section, we discuss the unconstrained oracle data collection strategy that knows the variances of reward and constraint value but does not know the mean of either. Moreover, this oracle does not take into account the safety constraints in \eqref{eq:safety-constraint-MDP}.
%
%
%
%
After observing $n$ samples (state-action-reward tuples), the oracle computes the estimate of $V^{\pi}(s^1_1)$ as $Y^{\pi}_n(s^1_1) =\sum_{a=1}^A\pi(a|s^1_1)\big(\wmu_n(s^1_1,a) + \sum_{s^{\ell+1}_j} P(s^{2}_j|s^1_1, a)Y_{n}(s^2_j)\big)$.
Note that 
we defined $Y^{\pi}_n(s,\ell)$ before, but now we use $Y^{\pi}_n(s)$ and assume the timestep is implicit in the state for this finite-horizon MDP.
%
%
%
%
%
%
%
\citet{mukherjee2022revar} shows that in the unconstrained setting, to reduce the $\Var(Y^{\pi}_n(s^1_1))$ the optimal sampling proportion of the oracle 
for any state $s^\ell_i$ is:
%
\begin{align}
    \bb_*(a|s^{\ell}_i)  &\propto 
     \big( \pi^2(a|s_i^{\ell}) \big[\sigma^2(s^{\ell}_i, a) \nonumber\\
     & + \sum\limits_{s^{\ell+1}_j}P(s^{\ell + 1}_j|s_i^{\ell}, a) M^2(s^{\ell+1}_j) \big]\big)^{\frac{1}{2}} \label{eq:mdp-b-def}
\end{align}
where, $M(s^{\ell}_j)$ is the normalization factor defined as follows:
\begin{align}
     M(s^{\ell}_{i}) &=
    \sum\limits_a \big(\pi^2(a|s^{\ell}_{i})\big(\sigma^2(s^{\ell}_{i}, a)\nonumber\\
&+\sum\limits_{s^{\ell+1}_j}\!\!P(s^{\ell+1}_j|s^{\ell}_i, a) M^2(s^{\ell+1}_j)\big)\big)^{\frac{1}{2}}
    \label{eq:M-def}.
\end{align}
Observe from the definition of $\bb_*(a|s^\ell_i)$ that the optimal proportion in the terminal states, i.e. $\bb_*(a|s^L_j)$, 
do not affect subsequent states and only depends on the target probability $\pi^2(a|s^{\ell}_{i})$ and variance $\sigma^2(s^{\ell}_{i},a)$.
The key difference is in the non-terminal states, $s^{L-1}_i$, where the optimal action proportion, $\bb_*(a|s^{L-1}_i)$ depends on the expected terminal state normalization factor $M(s^{L}_j)$ where $s^L_j$ is a state sampled from $P(\cdot|s^{L-1}_i,a)$. 
The normalization factor, $M(s^L_j)$, captures the total contribution of state $s^L_j$ to the variance of $Y^{\pi}_n(s^{L-1}_j)$ and thus actions in the starting state must be chosen to 1) reduce variance in the immediate reward estimate and to 2) get to states that contribute more to the variance of the estimate. This observation is also noted in \citet{mukherjee2022revar}. Finally, since $\bb_*(a|s)$ also depends on $P(s'|s,a)$, it will put a low sampling proportion on actions $a$ leading to such $s'$ which has low transition probabilities.

\subsection{Safe Oracle Algorithm for Safe Data Collection}
\label{sec:safe-oracle-policy}

The behavior policy defined in the previous section ignores the safety constraint and is thus inapplicable to our problem setting. In this section, we describe a safe variant of this oracle.
%
We define a few notations before introducing the safe algorithm.
Let $T^k_\ell(s,a) \!\coloneqq\! \sum_{k'=1}^{k-1}\sum_{\ell' = 1}^{\ell - 1} \mathbf{1}\{S^{k'}_{\ell'} \!=\! s, A^{k'}_{\ell'} \!=\! a\}$ be the number of times $(s,a)$ is visited before episode $k$. 
Let the mean reward estimate of $(s,a)$ till episode $k$ be computed as $\wmu^k_\ell(s,a) \coloneqq (T^k_\ell(s,a) )^{-1}\sum_{k'=1}^{k-1}\sum_{\ell' = 1}^{\ell-1} \mathbf{1}\{S^{k'}_{\ell'} \!=\! s, A^{k'}_{\ell'} \!=\! a\} R^{k'}_{\ell'}$, where $R^{k'}_{\ell'}$ is the observed reward. Similarly define the constraint-values estimate $\wmu^k_{c,\ell}(s,a)$ based on constraint value $C^{k}_{\ell}$. 
Define the confidence interval at the timestep $L$ of $k$-th episode as $\beta^k_L(s,a)\coloneqq L\sqrt{\log(SAn(n+1))/T^k_L(s,a)}$ \citep{agarwal2019reinforcement}. 

Let
$Y_{c,L}^{\bb^k}(s^1_1) =\!\!\sum_{a=1}^A\bb^k(a|s^1_1)\!\big(\wmu_{c,L}^{k}(s^1_1,a) \!+ \!\sum_{s^{\ell+1}_j}\! P(s^{2}_j|s^1_1, a)Y_{c,L}^{\bb^k}(s^2_j)\!\big)$ denote the empirical estimate of $V_c^{\bb^k}(s^1_1)$ at the end of the $k$-th episode, and $\wmu_{c,L}^{k}(s,a)$ is the empirical estimate of $\mu^c(s,a)$ at the end of the $k$-th episode. 
Note that the oracle algorithm knows the variances of reward $R(\cdot)$ and constraint-value $C(\cdot)$.
%
Using this knowledge, it maintains a safety budget $\wZ^{k-1}_L$ where  $\wZ^{k-1}_L \coloneqq \sum_{k'=1}^{k-1}(Y_{c,L}^{\bb^{k'}}(s^1_1) - \beta^{k'}_L(s,a)) - (1-\alpha)(k-1)V_{c}^{\pi_0}(s^1_1)$ is the safety budget at the end the $k-1$-th episode.
The $\underline{Y}_{c,L}^{\bb^k}(s^1_1) = Y_{c,L}^{\bb^k}(s^1_1) - \beta^k_L(s,a)$ is the lower confidence bound to the $Y_{c,L}^{\bb^k}(s^1_1)$. 

\textbf{Exploration policy $\pi_x$:} We require an exploration policy $\pi_x$ as the oracle algorithm needs a good estimation of the constraint-value $\mu^c(s,a)$ and following the oracle proportion $\bb_*(a|s)$ may not lead to a good estimation of $\mu^c(s,a)$. 
This exploration policy should ensure with high probability that the estimation error of $\mu^c(s,a)$ is low in each $(s,a)$ for which $\pi(a|s) > 0$ and 
can be an optimal design based policy like PEDEL that explores the state space informatively \citep{wagenmaker2022instance} or other exploration policies (e.g., \citet{dann2019policy, menard2020fast, uehara2021representation}).
%
%
%

We now state the following safe oracle algorithm: At the $k$-th episode run the policy
\begin{align}
    \bb^k_* = \begin{cases}  
    \bb_{*}, &\text{ if } \wZ^{k-1}_L \geq 0, k > \sqrt{K}\\
    \pi_0 &\text{ if } \wZ^{k-1}_L < 0 \\
    \pi_x, &\text{ if } \wZ^{k-1}_L \geq 0, k \leq \sqrt{K}
    \end{cases} \label{eq:mdp-sampling-rule}.
\end{align}
The safe oracle algorithm in \eqref{eq:mdp-sampling-rule} alternates between the optimal oracle policy $\bb_*$ in \eqref{eq:mdp-b-def} when the safety budget $\wZ^{k-1}_L$ at the start of the episode $k$ is greater than $0$, otherwise it falls back to running the baseline policy $\pi_0$. Additionally, the safe oracle conducts forced exploration for at most $\sqrt{K}$ episodes when $\wZ^{k-1}_L \geq 0$ using the exploration policy $\pi_x$ to estimate $\mu^c(s,a)$. This is because following the oracle proportion $\bb_*$ in \eqref{eq:mdp-b-def} that samples high variance state-action tuples may not lead to a good estimate of $\mu^c(s,a)$.

\subsection{An Intractable MDP}

In this section, we now show that there exist MDPs where even a safe oracle algorithm may not be able to reach the desired $\widetilde{O}(n^{-3/2})$ regret bound.
We then introduce the tractability condition which depends on the budget as the $\bb_*$ needs to be run sufficient number of times to reach a regret of $\tO(n^{-3/2})$. So a more benign MDP allows one to run $\bb_*$ most of the time whereas a less benign MDP allows you to play $\bb_*$ less. Hence tractability depends on the budget being sufficiently large and also depends on properties of the MDP  and the risk parameter $\alpha$.
To show this challenging aspect of safe data collection, we first define a Tree MDP. 
Using Tree MDPs to analyze the hardness of learning in MDPs and deriving lower bounds is common in the literature \citep{jiang2016doubly,weisz2021exponential,wagenmaker2022beyond,jin2022policy}.
The tree MDP is defined as follows:
\begin{definition}\textbf{(Tree MDP)}
\label{def:tree-mdp}
An MDP is a discrete tree MDP $\T\subset \M$ in which:
\textbf{(1)} There are $L$ levels indexed by $\ell$ where $\ell = 1,2,\ldots,L$. 
\textbf{(2)} Every state is represented as $s^\ell_i$ where $\ell$ is the level of the state $s$ indexed by $i$. 
\textbf{(3)} The transition probabilities are such that one can only transition from a state in level $\ell$ to one in level $\ell + 1$ and each non-initial state can only be reached through one other state and only one action in that state. Formally, $\forall s'$, $P(s' | s,a) \neq 0$ for only one state-action pair $s,a$ and if $s'$ is in level $\ell + 1$ then $s$ is in level $\ell$. Finally, $P(s^{L+1}_j|s^L_i, a) = 0, \forall a$.
\textbf{(4)} For simplicity, we assume that there is a single starting state $s^1_1$ (called the root). It is easy to extend our results to multiple starting states with a starting state distribution, $d_0$, by assuming that there is only one action available in the root that leads to each possible start state, $s$, with probability $d_0(s)$. The leaf states are denoted as $s^L_i$. 
\textbf{(5)} The interaction stops after $L$ steps in state $s^L_i$ after taking an action $a$.
\end{definition}
\begin{customproposition}{1}
\label{prop:intractable}
    Fix an arbitrary $n > 0$. Then 
    there exists an environment where no algorithm (including the safe oracle $\bb^k_*$) can be run that will result in a regret $\cR_n = \L_n(\pi, \bb^*_k) - \L^*_n(\pi, \bb_*)$ of $\tO(n^{-3/2})$ while satisfying the safety constraint, where $\bb_*$ is the unconstrained oracle.
\end{customproposition}
\textbf{Proof (Overview)} We first construct a worst-case $3$ armed bandit environment (MDP with single state) such that  $\mu^c(0) = 0.5$, $\mu^c(1) = 0.5 + \alpha$, $\mu^c(2) = 0$ and variance of $\sigma^{r,(2)}(0) = 0.001$, $\sigma^{r,(2)}(1) = 0.001$ and $\sigma^{r,(2)}(2) = 0.25$. So action $\{2\}$ has low constraint value (unsafe) but has high variance. So the safe oracle policy must sample the action $2$ a large number of times to reach a regret of $\tO(n^{-3/2})$. However, since action $\{2\}$ is unsafe, the safe oracle has to sample baseline action $0$ a sufficient number of times to accrue some safety budget. Combining these two observations we show that achieving a regret rate of $\tO(n^{-3/2})$ is impossible. The full proof is in \Cref{app:intractable-mdp}.
%


The key reason the above environment is intractable is that some trajectories taken by safe oracle has very less constraint value associated with them, compared to the trajectory taken by the baseline policy. To rule out such pathological MDPs, we define the \textit{tractability} condition as follows: 
%
%
%
Let $\bb^{-}$ be any behavior policy that minimizes $V^c_{\bb}(s_1)$. 
%
Define $V^{c}_{\bb^{-}}(s_1)$ as the value of the policy $\bb^{-}$ starting from state $s_1$. This policy $\bb^{-}$ suffers a value $V^{c}_{\bb^{-}}(s_1)$ that is lower than any other behavior policy $\bb$. 
So this policy $\bb^{-}$ can be thought of as the worst possible behavior policy that can be followed by the agent during an episode.
%
%
%
%
%
Then the tractability condition states that
\begin{align}
    \sqrt{n} \geq \dfrac{ \frac{1}{\alpha}\left(1-\frac{V^{c}_{\bb^{-}}(s_1)}{V^{c}_{\pi_0}(s_1)} \right) }{\frac{C_\sigma}{\alpha}\left(1-\frac{V^{c}_{\bb^{-}}(s_1)}{V^{c}_{\pi_0}(s_1)} \right) - 1} \label{eq:safe-state}
\end{align}
where $C_\sigma\in (0,1)$ is a MDP dependent parameter that depends on the reward variance of state-action pairs such that $\frac{C_\sigma}{\alpha}\left(1-\frac{V^{c}_{\bb^{-}}(s_1)}{V^{c}_{\pi_0}(s_1)} \right) - 1 > 0$. 
The quantity $C_\sigma = \max_{s,a}\frac{\bb_*(a|s)}{M(s)}$ where $\bb_*(a|s)$ and $M(s)$ are defined in \eqref{eq:mdp-b-def} and \eqref{eq:M-def} respectively. So $C_\sigma \in (0,1)$ and it captures the worst case trajectory that can be followed by $\bb_*$.

%
%
This condition gives us (1) the lower bound to the budget $n$ to run the behavior policy $\bb^{-}$ to achieve a regret bound of $\tO(n^{-3/2})$ and satisfy the safety constraint; (2) $V^{c}_{\bb^{-}}(s_1) < V^{c}_{\pi_0}(s_1)$ so that the RHS is positive, (3) depends on the reward variance of state action pairs in the MDP so that $\frac{C_\sigma}{\alpha}\left(1-\frac{V^{c}_{\bb^{-}}(s_1)}{V^{c}_{\pi_0}(s_1)} \right) - 1 > 0$, and (4) for smaller $\alpha$ (high risk) the R.H.S increases which increases the required budget $n$.
We further discuss how this condition in \eqref{eq:safe-state} is derived in \Cref{rem:safe-state}.
%
Then we define the following assumption.
\begin{assumption}\textbf{(Tractability)}
\label{assm:tractable-MDP}
    We assume a sufficiently large budget $n$ and an MDP $\M$ that satisfies the constraint in \eqref{eq:safe-state}.  
    We call such an MDP $\M$  \textit{tractable}.
\end{assumption}
\Cref{assm:tractable-MDP} ensures that even the worst possible behavior policy $\bb^{-}$ that can reach a regret of $\tO(n^{-3/2})$ has sufficient budget $n$ to satisfy the safety constraint.  
Moving forward, we will define regret relative to this safe oracle $\bb^K_*$ instead of the unconstrained oracle. Furthermore, we assume tractability in \eqref{assm:tractable-MDP} such that the safe oracle decreases MSE at a comparable rate to the unconstrained oracle $\bb_*$.
Define the reward regret as $\cR_n = \L_n(\pi, \bb) - \L^*_n(\pi, \bb^k_*)$ 
where $\L^*_n(\pi, \bb^k_*)$ is the safe oracle MSE, and $\L_n(\pi, \bb)$ is the agnostic algorithm MSE that does not know reward or constraint-value variances.
Now we present the first general lower bound theorem for the safe data collection strategy in MDPs.
\begin{customtheorem}{1}
\label{thm:lower-bound}\textbf{(Lower Bounds)}
Let $\pi(a|s) = \tfrac{1}{A}$ for each state $s\in\S$. Under \Cref{assm:tractable-MDP} the regret $\cR_n = \L_n(\pi, \bb) - \L^*_n(\pi, \bb^k_*)$ is lower bounded by 
\begin{align*}
    \E\left[\cR_n\right]\geq\begin{cases}\Omega\left(\max \left\{\frac{A^{1/3}}{n^{3/2}}, \left(\frac{H_{*, (1)}^2 A^{2/3}}{n^{3/2}}\right)\right\}\right), \textbf{(MAB)}\\
    \Omega\left(\max \left\{\frac{\sqrt{SAL^2}}{n^{3/2}}, \left(\frac{H_{*, (1)}^2 SAL^2}{n^{3/2}}\right)\right\}\right) \textbf{(MDP)}
    \end{cases}
\end{align*}
where, $\Delta_0=V^{\bb_*}_c(s^1_1)-V^{\pi_0}_c(s^1_1)$ and $H_{*, (1)} = \frac{1}{\alpha V^{\pi_0}_{c}(s^1_1)}(\alpha V^{\pi_0}_{c}(s^1_1)+\Delta_0)$ is the hardness parameter.
\end{customtheorem}

\textbf{Discussion:} \Cref{thm:lower-bound} shows that in the constrained setting the lower bound scales as $\Omega(H_{*, (1)}^2 n^{-3/2})$. 
Note that we can recover the lower bound for the unconstrained setting using this result.
In the unconstrained bandit setting the bound scales as $O\left({A^{1 / 3}}{n^{- 3 / 2}}\right)$ which matches the lower bound of \citet{carpentier2012minimax} (see their Theorem 5). We also establish the first lower bound for the unconstrained setting in data collection for policy evaluation in the tabular MDP setup that scales as $O\left({\sqrt{S A L^2}}{n^{-3 / 2}}\right)$. The $H_{*, (1)}$ captures the hardness in learning in the MDP and consists of the gap $\Delta_0$, $V^{\pi_0}_{c}(s^1_1)$ and $\alpha$. Note that $H_{*, (1)}$ increases with $\alpha$, and the $\Delta_0$ captures how much constraint value the $\bb_*$ can obtain compared to $\pi_0$. Finally, the smaller value of $\pi_0$ increases the hardness as the $\pi_0$ has to be run more times so that the safety constraint is not violated.

\textbf{Proof (Overview)} We first build two deterministic tree MDPs $\T$ and $\T'$ which differ in the variances at only one state. This leads to different optimal oracle behavior policies in $\T$ and $\T'$. Then using the divergence decomposition lemma for MDPs from \citet{garivier2016optimal, wagenmaker2022beyond} we show in \Cref{lemma:tab-RL-lower-bound} that in $\T$ the regret lower bound scales as $\Omega({\sqrt{SAL^2 \log(n)}}/{n^{3/2}})$.
Next, we follow a reduction-based proof technique to prove the reward regret lower bound in the constrained setting. 
Consider any sequential decision-making problem $\mathfrak{A}$ (for instance a multi-armed bandit problem, tabular RL) such that there exists a problem-dependent constant $\xi \in \mathbb{R}$ that only depends on on the number of actions in bandits, or state-action-horizon in tabular RL. Then for a large budget $n$ and any algorithm we have from \Cref{lemma:lower-bound-bandit} and \Cref{lemma:tab-RL-lower-bound} that $\E[\cR_n^{}] \geq \tfrac{\xi}{n^{3/2}}$ for an MDP dependent parameter $\xi$.
Then we lower bound how many times under the budget $n$ the algorithm can run the baseline policy. This is lower bounded in step 2 as $\E[\cR^{}_n] \!\gtrsim\! \min \big\{\tfrac{\xi}{n^{3/2}}, \tfrac{(\alpha V^{\pi_0}_{c}(s^1_1)+\Delta_0)^2\xi^2}{(\alpha V^{\pi_0}_{c}(s^1_1))^2 n^{3/2}}\big\}$. 
We finish off the proof by noting that the quantity $H_{*, (1)} \!\!=\!\! \tfrac{1}{\alpha V^{\pi_0}_{c}(s^1_1)}(\alpha V^{\pi_0}_{c}(s^1_1)+\Delta_0)$ is the hardness parameter when $\pi(a|s) = 1/A$, and substituting the value of $\xi = A^{1/3}$ for bandits (\Cref{lemma:lower-bound-bandit}) and $\xi=\sqrt{SAL^2}$ for $\T$ (\Cref{lemma:tab-RL-lower-bound}). Since $\T\subset \M$, this result is a lower bound to $\M$ as well. 
%
The full proof is in \Cref{app:tract}. \hfill $\blacksquare$

\section{Agnostic Algorithm for Safe Policy Evaluation}
\label{sec:tree}
In this section, we introduce the more realistic agnostic algorithm that does not know the mean and variances of the reward and constraint values of the actions.
We then analyze this algorithm and establish its finite-time MSE. We call this algorithm \textbf{Sa}fe \textbf{V}arianc\textbf{e} \textbf{R}eduction algorithm (abbreviated as \sav) as it reduces the variance of the estimated value of the target policy by following \eqref{eq:mdp-b-def} while simultaneously satisfying the safety constraint \eqref{eq:safety-constraint-MDP} with high probability.
%
%
%

We introduce a few notations before presenting the algorithm. Define the upper confidence bound on the empirical reward variance as $\overline{\wsigma}^{k}_L(s,a) \coloneqq \wsigma^{k}_L(s,a) +\beta^k_L(s,a)$, where $\beta^k_L(s,a)$ is the confidence interval defined in \Cref{sec:oracle-sampler}.
%
%
%
We define the empirical sampling proportion for an arbitrary state-action $(s^\ell_i,a)$ as $\wb^k_\ell(a|s^{\ell}_i)$. 
%
Define the policy $\wb^k_{*,\ell}(a|s^{\ell}_i)$ as similar to $\bb_*(a|s^\ell_i)$ defined in \eqref{eq:mdp-b-def}, but it uses plug-in estimate $\overline{\wsigma}^{k}_\ell(s,a)$ instead of $\sigma^{k}_\ell(s,a)$. This is because the agnostic algorithm does not know the reward and constraint-value variances. 
We define $\wZ^{k-1}_L$ similar to \eqref{eq:mdp-sampling-rule}. 
%
%
%
%
Finally, we define our algorithm, \sav, as follows: At episode $k$ run the policy:
\begin{align}
\quad \wb^k = \begin{cases} 
        \wb^k_* & \text{ if } \wZ^{k-1} \geq 0, k > \sqrt{K}\\
        \pi_0 &\text{ if } \wZ^{k-1} < 0 \\
        \pi_x & \text{ if }             \wZ^{k-1} \geq 0, k \leq \sqrt{K}
        \label{eq:mdp-saver-sample}
\end{cases} 
\end{align}
where $\wb^k_*$ for the episode $k$ is defined as follows: For each timestep $\ell=1,2,\ldots,L$ sample action $A^k_\ell \!\!=\!\! \argmax_a\tfrac{\wb^k_*(a|s^\ell_j)}{T^k_\ell(s^\ell_j,a)}$, where $\wb^k_*(a|s^\ell_j)$ is the plug-in estimate of $\bb_*(a|s^\ell_j)$ as defined in \eqref{eq:mdp-b-def}.
 \sav\ alternates between the exploration policy $\pi_x$, plugin optimal policy $\wb^k_*$, and baseline policy based on the safety budget $\wZ^k$ and the number of episodes $K$. 
In contrast to \eqref{eq:mdp-saver-sample} the oracle policy in \eqref{eq:mdp-sampling-rule} uses the true oracle proportions $\bb^*$ when $\wZ^{k-1} \geq 0, k \!\!>\!\! \sqrt{K}$. 
%
Also, observe that the action selection rule ensures that the ratio $\wb^k_{*,\ell}(a|s)/T^k_{\ell}(s,a) \!\!\approx\!\! 1$.
It is a deterministic action selection rule and thus avoids inadvertently violating the safety constraint due to random sampling from the optimal proportions $\wb^k_\ell(a)$.
Now we formally state the \sav\ for the tree MDP. 
At every episode $k\in[K]$ it generates a sampling history $\H^k\coloneqq\{S^k_\ell, A^k_\ell, R(S^k_\ell, A^k_\ell), C(S^k_\ell, A^k_\ell)\}_{\ell=1}^L$ by selecting $A^k_\ell$ according to \eqref{eq:mdp-saver-sample} and appends it to the dataset $\D$. 
After observing the feedback it updates the model parameters and estimates $\wb^{k+1}_{1}(a|s)$ for each $s,a$. It returns the dataset $\D$ to evaluate $\pi$. The pseudocode is in \Cref{alg:tree-track}.
\begin{algorithm}
\caption{\textbf{Sa}fe \textbf{V}arianc\textbf{e} \textbf{R}eduction (\sav) for $\T$}
\label{alg:tree-track}
\begin{algorithmic}[1]
\STATE \textbf{Input:} Risk Parameter $\alpha>0$, target policy $\pi$.
\STATE \textbf{Output:} Dataset $\D$.
\STATE Initialize $\D = \emptyset$, $\wb^{}_1(a|s)$ uniform over all actions.
\FOR{$k=1,2,\ldots,K$}

\FOR{$\ell=1,2,\ldots,L$}
\STATE Get $\H^k\coloneqq\{S^k_\ell, A^k_\ell, R(S^k_\ell, A^k_\ell), C(S^k_\ell, A^k_\ell)\}_{\ell=1}^L$ by selecting $\bb^k$ according to \eqref{eq:mdp-saver-sample}.
\STATE $\D \leftarrow \D \cup \{(\H^k, \wb^k)\}$
\STATE Update model parameters and estimate $\wb^{k+1}_{1}(a|s)$ for each $s,a$ 
\ENDFOR
\ENDFOR
\STATE \textbf{Return} Dataset $\D$ to evaluate policy $\pi$.
\end{algorithmic}
\end{algorithm}

We now present a theorem that gives the MSE of the agnostic algorithm \sav\ in the tree MDP in the following theorem.
We define the problem complexity parameters
%
$M =\sum_{\ell=1}^L\sum_{s^\ell_j}M(s^\ell_j)$ 
summed over all stated $s\!\in\![S]$. 
Define predicted agnostic constraint violation 
\begin{align*}
    \C_n(\pi, \wb^k) \coloneqq \sum_{k=1}^K\indic{\wZ^k < 0}
\end{align*}
when taking actions according to \eqref{eq:mdp-saver-sample}.
For scalars $x,y \in \mathbb{R}$ define $\min ^{+}(x, y):=|\min (x, y)|$. 
Define the problem complexity parameter $H_{*, (2)} \!\!=\!\!\sum_{\ell=1}^L\sum_{s^\ell_j}H_{*, (2)}(s^\ell_j)$ where 
\begin{align}
H_{*, (2)}(s^\ell_j) \!\!&=\!\! \frac{1}{\alpha\mu^{c}(s^\ell_j,0)}\!\!\!\sum_{a\in\A\setminus\{0\}}\!\!\!\!\pi(a|s^\ell_j)\sigma(s^\ell_j,a)\min^{+}\left\{\Delta_c(s^\ell_j,a),\right.\nonumber\\
&\left.\Delta_c(s^\ell_j,0)-\Delta_c(s^\ell_j,a) \}\right\}. \label{eq:prob-complexity}
\end{align}
%
\begin{remark}\label{rem:H}
The quantity $H_{*, (2)}(s^\ell_j)$ signifies the total cost of maintaining the safety constraint at state $s^\ell_j$ by sampling action $0$ instead of sampling based on $\pi(a)\sigma(a)$. Observe that $\Delta_c(s^\ell_j,0)-\Delta_c(s^\ell_j,a) = \mu_c(s^\ell_j,a) - \mu_c(s^\ell_j,0)$. So $\min^+\{\Delta_c(s^\ell_j,a),\Delta_c(s^\ell_j,0)-\Delta_c(s^\ell_j,a)\}$ depends on how close is the action $a$ to the best cost action $ \mu^{*,c}(s^\ell_j)$ or the baseline action $0$. Also observe that because of the $\min^+$ operator, this quantity cannot be $0$. Further, observe that the gap is weighted by $\pi(a|s^\ell_j)\sigma(s^\ell_j,a)$ signifying that actions with low variance and target probability contribute less to the constraint violation MSE. Also, observe that higher risk setting $(\alpha \rightarrow 0)$ leads to higher $H_{*, (2)}(s^\ell_j)$. Finally, it can be easily verified that $H_{*,(2)} > H_{*, (1)}$. 
\end{remark}


Now we present a theorem that we will use to bound the regret of \sav\ in Tree MDP $\T$ under \Cref{assm:tractable-MDP}. 


\begin{customtheorem}{2}\textbf{(informal)}
\label{thm:mdp-agnostic-loss}
The MSE of the \sav\ in $\T$ for $\frac{n}{\log(SAn(n+1)/\delta)} \!\!\geq\!\! O((LSA^2)^2 + \tfrac{SA}{\Delta^{c,(2)}_{\min}} + \tfrac{1}{4 H_{*, (2)}^2})$ is bounded by $\L_n(\pi, \wb^k) \leq \tO\big(\tfrac{M^2(s^1_1)}{n} + \frac{M^2(s^1_1)}{n}(M LSA^2 + H_{*, (2)})^2   + \tfrac{(LSA^2)^2 H_{*, (2)}^2 M^2}{\min_{s}\bb^{*,k,(3/2)}(s)n^{3/2}}\big)$ with probability $(1-\delta)$. 
    The total predicted constraint violations are bounded by
        $\C_n(\pi, \wb^k) \leq \tO\big(\frac{H_{*, (2)}}{2} \frac{n}{M_{\min}} + LSA^2 +  \frac{(LSA^2)^2 H_{*, (2)}^2 M^2}{n^{1/2}}\big)$
    with probability $(1-\delta)$, where $M_{\min} \coloneqq \min_s M(s)$.
    \vspace*{-0.5em}
\end{customtheorem}

\textbf{Discussion:} In \Cref{thm:mdp-agnostic-loss} the first quantity upper bounding $\L_n(\pi, \wb^k)$ is denoted as the \textit{safe MSE} when the safety budget $\wZ^k \geq 0$ and scales as $M^2(s^1_1)/n$. The second quantity is denoted as the \textit{unsafe MSE} which is accumulated due to constraint violation ($\wZ^k < 0$) and sampling of the safe action $0$. Finally, the third quantity is the MSE suffered due to estimation error of the variances $\sigma^2(s,a)$.
Comparing the result of the \Cref{thm:mdp-agnostic-loss} with the unconstrained setting of \citet{mukherjee2022revar} we have the additional quantity of ${(M LSA^2 + H_{*, (2)})^2}/{n}$
where $H_{*, (2)}$ is the problem-dependent quantity summed over all states. 
Observe that if all actions are safe then we have that $\L^*_n(\pi, \wb^k) = {M^2(s^1_1)}/{n}$ which recovers the MSE of the unconstraint setting in 
\citet{carpentier2011finite, carpentier2012minimax, carpentier2015adaptive, mukherjee2022revar}.
%



\textbf{Proof (Overview)}
    The agnostic \sav\ does not know the reward variances. The sampling rule in \eqref{eq:mdp-saver-sample} ensures that the good variance event $\xi_{v,K}$ defined in \eqref{eq:mdp-good variance-event} (step $2$) holds such that \sav\ has good estimates of reward variances. 
    Then, note that in the tree MDP $\T$ we have a closed form expression of $\bb_*(s^\ell_j|a)$. 
     We divide the total budget $n = n_f + n_u$ where $n_f$ are the samples allocated when safety budget $\wZ^k \geq 0$. The $n_f$ samples are also used by the exploration policy $\pi_x$ to ensure a good estimate of the constraint means as stated in the event $\xi_{c,K}$ \eqref{eq:mdp-good-constraint-event-agnostic}. This is ensured by $\pi_x$ and noting that $n > SA\log(1/\delta)/\Delta^{2}_{c,\min}$. The remaining samples from $n_f$ are allocated for reducing the MSE by sampling according to $\argmax_a(\bb_*(a|s)/T^k_\ell(s,a))$. We again prove an upper and lower bound to $T_{n}(s,a)$ in \eqref{eq:mdp-saver-6-1} in step $4$ and \eqref{eq:mdp-saver-6-2} in step $5$. Finally using \Cref{lemma:wald-variance} we can bound the MSE for the duration $n_f$ for all actions $a\in\A\setminus\{0\}$ for each state $s^\ell_j$ in step $6$. Now for an upper bound to constraint violations, we use the gap $\Delta^{\alpha}_c(s,a)\coloneqq(1-\alpha) \mu_{c,0}(s,a)-\mu_c(s,a)$ to bound how much each $a\in\A\setminus\{0\}$ in $s^\ell_j$ is underpulled and their pulls replaced by action $\{0\}$ weighted by $\pi(a|s^\ell_j)\sigma(s^\ell_j,a)$. This is captured by $H_{*, (2)}(s)$. Summing over all $s$, and horizon $L$ gives the upper bound to the violations as shown in step $7$. Finally, we also show a lower bound to constraint violations to bound the MSE for the duration when actions $a\in\A\setminus\{0\}$ are underpulled. This is shown in steps $8$ and $9$ where we equate the safety budget to $0$ to obtain a lower bound to $T_n(s^\ell_j, 0)$ for each state $s^\ell_j$. Combining everything in step $10$ gives the result. The proof is in \Cref{app:mdp-saver-loss}.  
    \hfill$\blacksquare$

Note that we do not have a closed-form solution to $\bb^k_*$ that both minimizes MSE as well as upholds \eqref{eq:safety-constraint-MDP} for all $k\in[K]$ (as opposed to \citet{carpentier2011finite, mukherjee2022chernoff}).
Therefore, we now define two additional notions of regret. The first is the regret defined as 
    $\overline{\cR}_n = \L_n(\pi, \wb^k) - \overline{\L}^*_n(\pi, \bb^k_*)$ 
where $\overline{\L}^*_n(\pi, \bb^k_*)$ is the upper bound to the safe oracle MSE. 
The second is the constraint regret defined as follows:
    $\overline{\cR}^c_n = \C_n(\pi, \wb^k) - \overline{\C}^*_n(\pi, \bb^k_*)$ 
where $\overline{\C}^*_n(\pi, \bb^k_*)$ is the upper bound to the oracle constraint violations. Note that the oracle knows the variances of reward and constraint-values for all state-action tuples (but does not know the mean of either).
%
The following corollary bounds \sav\ regret.
\begin{customcorollary}{1}
\vspace*{-0.0em}
Under \Cref{assm:tractable-MDP}, the constraint regret of \sav\ is bounded by
$\overline{\cR}^c_n \leq O\big(\frac{\log(n)}{n^{1/2}}\big)$ and the regret is bounded by $\overline{\cR}_n \leq O\big(\frac{\log(n)}{n^{3/2}}\big)$.
\vspace*{-0.5em}
\end{customcorollary}

The proof is in \Cref{app:mdp-regret-corollary} and directly follows from \Cref{thm:mdp-agnostic-loss}, and \Cref{prop:mdp-oracle-loss}. In \Cref{prop:mdp-oracle-loss} in \Cref{app:mdp-oracle-loss} we prove the MSE upper bound of the oracle. 
Observe, that the regret decreases at a rate of $\tO(n^{-3/2})$, faster than the rate of decrease of on-policy MSE of  $\tO(n^{-1})$.
Thus, we have been able to answer the second main question of this paper affirmatively.
We also state a constraint and regret upper bound in the bandit setting in \Cref{corollary:bandit-regret} in \Cref{app:mdp-regret-corollary}. Also, observe that our upper bound matches the rate in the lower bound shown in \Cref{thm:lower-bound}.




\section{Extension to DAG}
\label{app:dag}
In this section, we approximate the solution in $\T$ to DAG $\G$ and formulate the safe algorithm for policy evaluation. We first define the DAG MDP in the following definition.

\begin{definition}\textbf{(DAG MDP)}
\label{def:dag-mdp}
A DAG MDP follows the same definition as the tree MDP in \Cref{def:tree-mdp} except $P(s'|s,a)$ can be non-zero for any $s$ in layer $\ell$, $s'$ in layer $\ell+1$, and any $a$, i.e., one can now reach $s'$ through multiple previous state-action pairs.
\end{definition}
Then we state the following lemma from \citet{mukherjee2022revar}.
\begin{lemma}
\label{lemma:dag}\textbf{\textbf{(Proposition 3 of \citet{mukherjee2022revar})}}
Let $\G$ be a $3$-depth, $A$-action DAG defined in \Cref{def:dag-mdp}. The minimal-MSE sampling proportions $\bb_*(a|s^1_1), \bb_*(a|s^2_j)$ depend on themselves such that $\bb(a|s^1_1) \propto f(1/\bb(a|s^1_1))$ and $ \bb(a|s^2_j) \propto f(1/\bb(a|s^2_j))$ where $f(\cdot)$ is a function that hides other dependencies on variances of $s$ and its children. 
\end{lemma}
The \Cref{lemma:dag} \citep{mukherjee2022revar} shows that one cannot derive a closed-form solution to $\bb_*$ in $\G$ because of the existence of multiple paths to the same state resulting in a cyclical dependency. Note that in $\T$ there is only a single path to each state and this cyclical dependency does not arise. 
%
%
%
%
%
%
If we ignore the multiple path problem, we can approximate the optimal sampling proportion in $\G$ by using the tree formulation in the following way: At every time $t$ during an episode $k$ call the \Cref{alg:revar-g-1} to estimate $M_0(s)$ where $M_{t'}(s)\in\mathbb{R}^{L\times |\S|}$ stores the expected standard deviation of the state $s$ at iteration $t'$. 
After $L$ such iteration we use the value $B_0(s)$ to estimate $\bb(a|s)$ as follows:
\begin{align*}
    \bb_*(a|s) \!&\propto\! 
    \! \sqrt{\! \pi^2(a|s)\! \bigg[\sigma^2(s, a)\!  +\!    \gamma^2\!\!  \sum\limits_{s'}\!\!P(s'|s, a) M^2_0(s)\! \bigg]}.
\end{align*}
Note that for a terminal state $s$ we have the transition probability $P(s'|s, a) = 0$ and  then the $b(a|s) = \pi(a|s)\sigma(s,a)$. This iterative procedure follows from the tree formulation in \Cref{lemma:L-step-tree} and is necessary in $\G$ to take into account the multiple paths to a particular state.  \Cref{alg:revar-g-1} gives pseudocode for this procedure which takes inspiration from value-iteration for the episodic setting.
%
%
\begin{algorithm}
\caption{Estimate $B_0(s)$ for $\G$}
\label{alg:revar-g-1}
\begin{algorithmic}[1]
\STATE Initialize $B_L(s)=0$ for all $s\in\S$ 
\FOR{$t'\in L-1, \ldots, 0$} 
\STATE $B_{t'}(s) = \sum\limits_{a} \big(\pi^2(a|s)\big(\sigma^2(s, a) $ \\
                \qquad\qquad\qquad $ + \gamma^2\sum\limits_{s'}P(s'|s, a) B_{t'+1}^2(s)\big)\big)^{\tfrac{1}{2}}$
\ENDFOR
\STATE \textbf{Return} $B_0$.
\end{algorithmic}
\end{algorithm}

Finally, the safe algorithm in $\G$ can be stated as follows: At episode $k$ 
\begin{align}
\quad \text{Play } \bb^k = \begin{cases} \pi_e & \text{ if }             \wZ^k \geq 0, k \leq \sqrt{K}\\
            \pi_{\wb^k} & \text{ if } \wZ^k \geq 0, k > \sqrt{K}\\
            \pi_0 &\text{ if } \wZ^k < 0 \label{eq:mdp-saver-dag-sample}
\end{cases} 
\end{align}
where $\pi_{\wb^k}$ for the episode $k$ is defined as follows: For each time $\ell=1,2,\ldots,L$ sample action $A^k_\ell = \argmax_a\tfrac{\wb^k(a|s^\ell_j)}{T^k_\ell(s^\ell_j,a)}$, where $\wb^k(a|s^\ell_j)$ is the plug-in estimate of $\bb_*(a|s^\ell_j)$ that is obtained using \Cref{alg:revar-g-1}.

\vspace*{-0.5em}
\section{Experiments}
\label{sec:expt}

\begin{figure}[!hbt]
\vspace*{-0.5em}
  \begin{center}
    \begin{tabular}{cc}
    \subfigure[Bandit setting]{\label{fig:expt-bandit}\includegraphics[scale = 0.27]{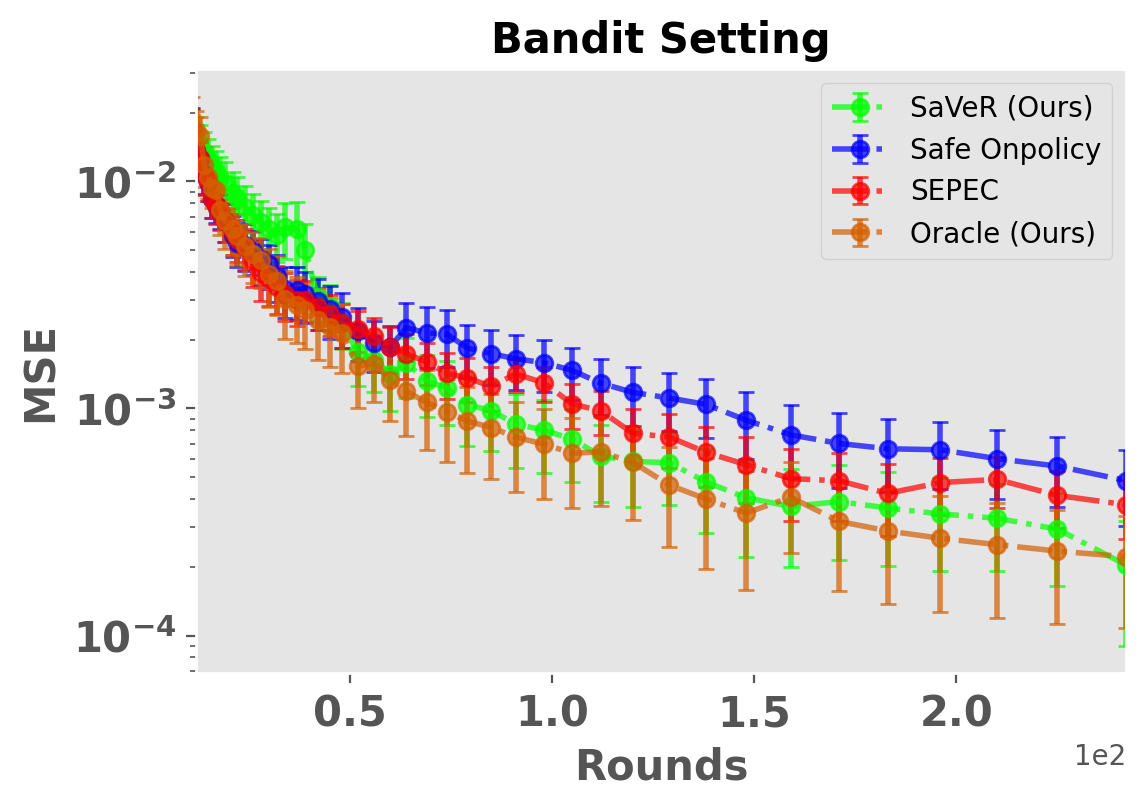}} &
    \subfigure[Movielens setting]{\label{fig:expt-movielens}\includegraphics[scale = 0.23]{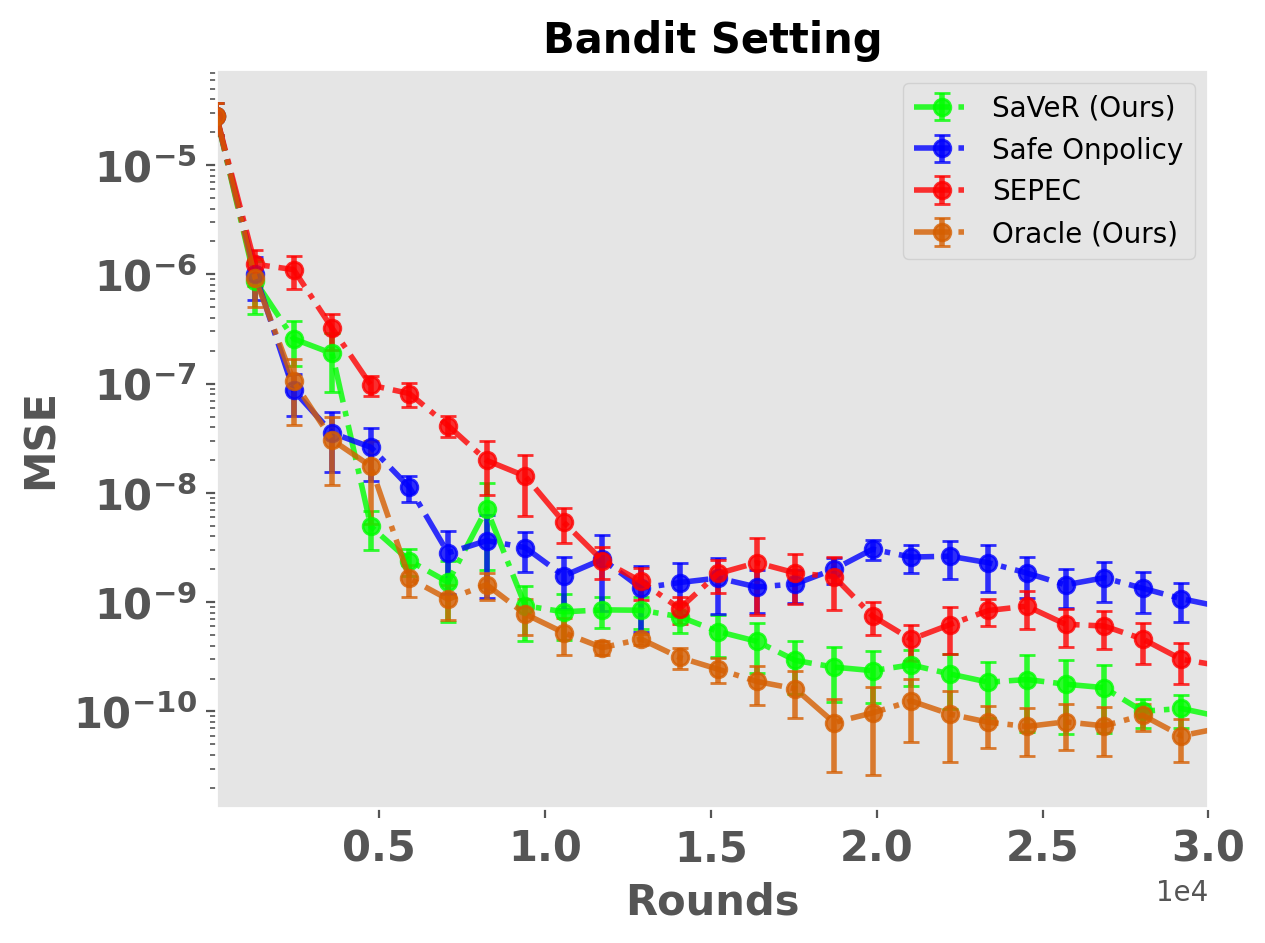}} \\
    \subfigure[Tree MDP]{\label{fig:expt-tree-mse}\includegraphics[scale = 0.27]{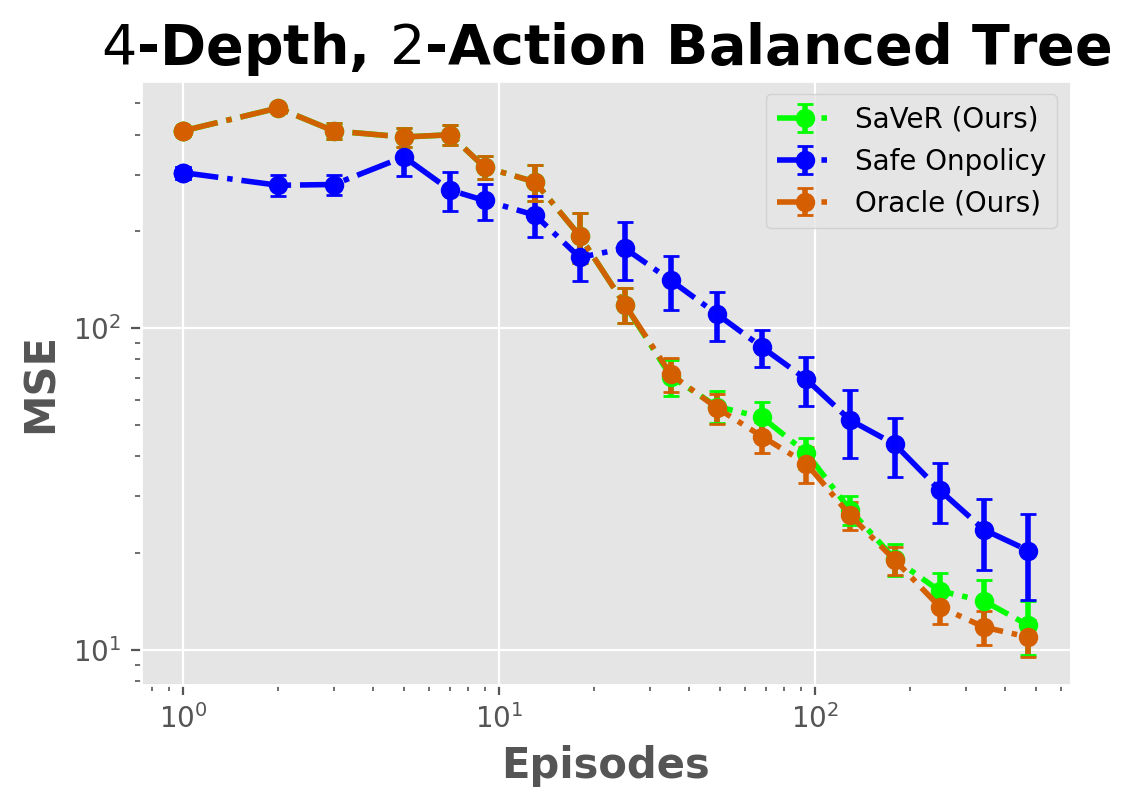}} &
    \subfigure[Grid MDP]{\label{fig:expt-grid-mse}\includegraphics[scale = 0.24]{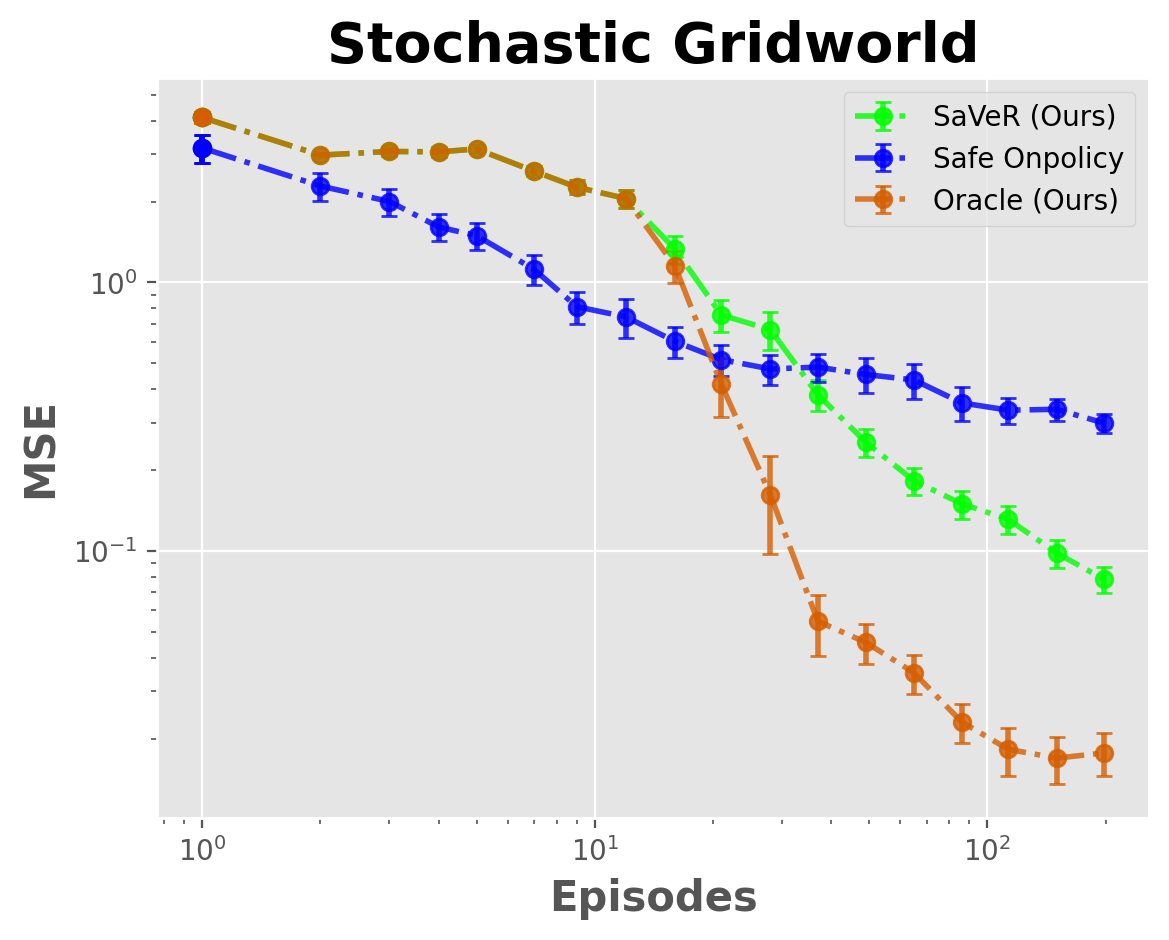}}
    \end{tabular}
    \vspace{-1.0em}
    \caption{MSE in different settings. The vertical axis (log-scaled) gives MSE and the horizontal axis is the number of episodes (or rounds for bandits). Confidence bars show one standard error.}
    \label{fig:expt}
  \end{center}
  \vspace*{-0.5em}
\end{figure}


In this section, we show numerical experiments validating our theoretical results. The full experimental details and numerical results are in \Cref{app:addl-expt-1}. 
We test the oracle, and \sav\ algorithm and introduce a method that we call safe on-policy. The safe on-policy algorithm follows the target policy $\pi$ when the safety budget is positive and plays baseline policy $\pi_0$ when the safety budget is negative. We also test against the SEPEC \citep{wan2022safe} algorithm for the bandit setting which uses importance sampling to safely collect data for policy evaluation.
Note that the bandit setting consists of a single state and every episode $K$ consists of a single timestep $L=1$.
\Cref{fig:expt} shows the MSE obtained by each algorithm for a varying number of episodes.
In \Cref{app:fig-expt-2}, we show that all algorithms respect the constraint but that the oracle and \sav\ are not excessively conservative.

\textbf{Experiment 1 (Bandit):} We implement a general bandit environment with $A=11$ and show that \sav\ achieves lower MSE than SEPEC and safe on-policy algorithm as the number of rounds increases. The performance is shown in \Cref{fig:expt-bandit}. From \Cref{fig:expt-bandit-safe} we see that \sav, and oracle do not oversample the safe action but allocate the right amount to be just safe. They allocate more samples to reduce the MSE, whereas the safe on-policy and SEPEC over-sample the safe action instead of focusing on reducing the MSE.

\textbf{Experiment 2 (Movielens):} We conduct this experiment on the real-life Movielens 1M dataset \citep{movielens} for $A=30$ actions and show that \sav\ achieves lower MSE than safe on-policy and SEPEC algorithm as the number of rounds increases. The performance is shown in \Cref{fig:expt-movielens}. From \Cref{fig:expt-movielens-safe}, we see that \sav\ and oracle \sav, and the oracle do not oversample the safe action compared to SEPEC.

\begin{figure}[!ht]
\vspace*{-0.5em}
  \begin{center}
    \begin{tabular}{cc}
    \subfigure[Bandit violation]{\label{fig:expt-bandit-safe}\includegraphics[scale = 0.27]{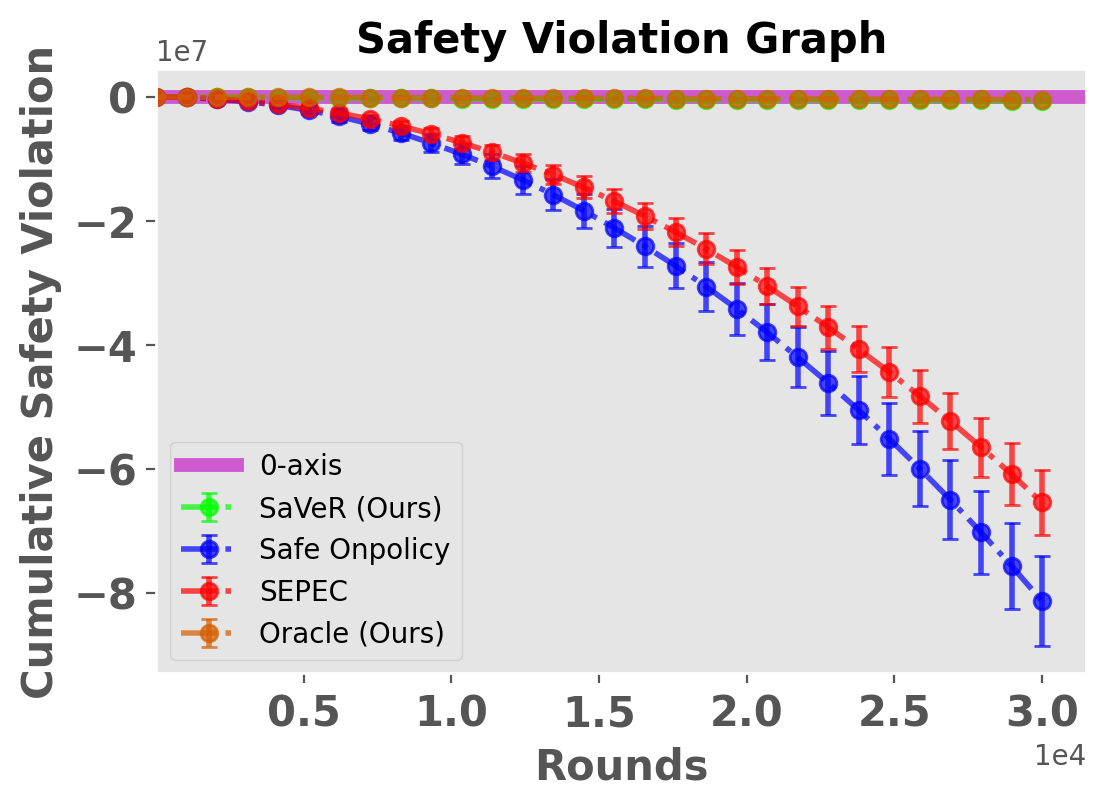}} &
    \subfigure[Movielens violation]{\label{fig:expt-movielens-safe}\includegraphics[scale = 0.23]{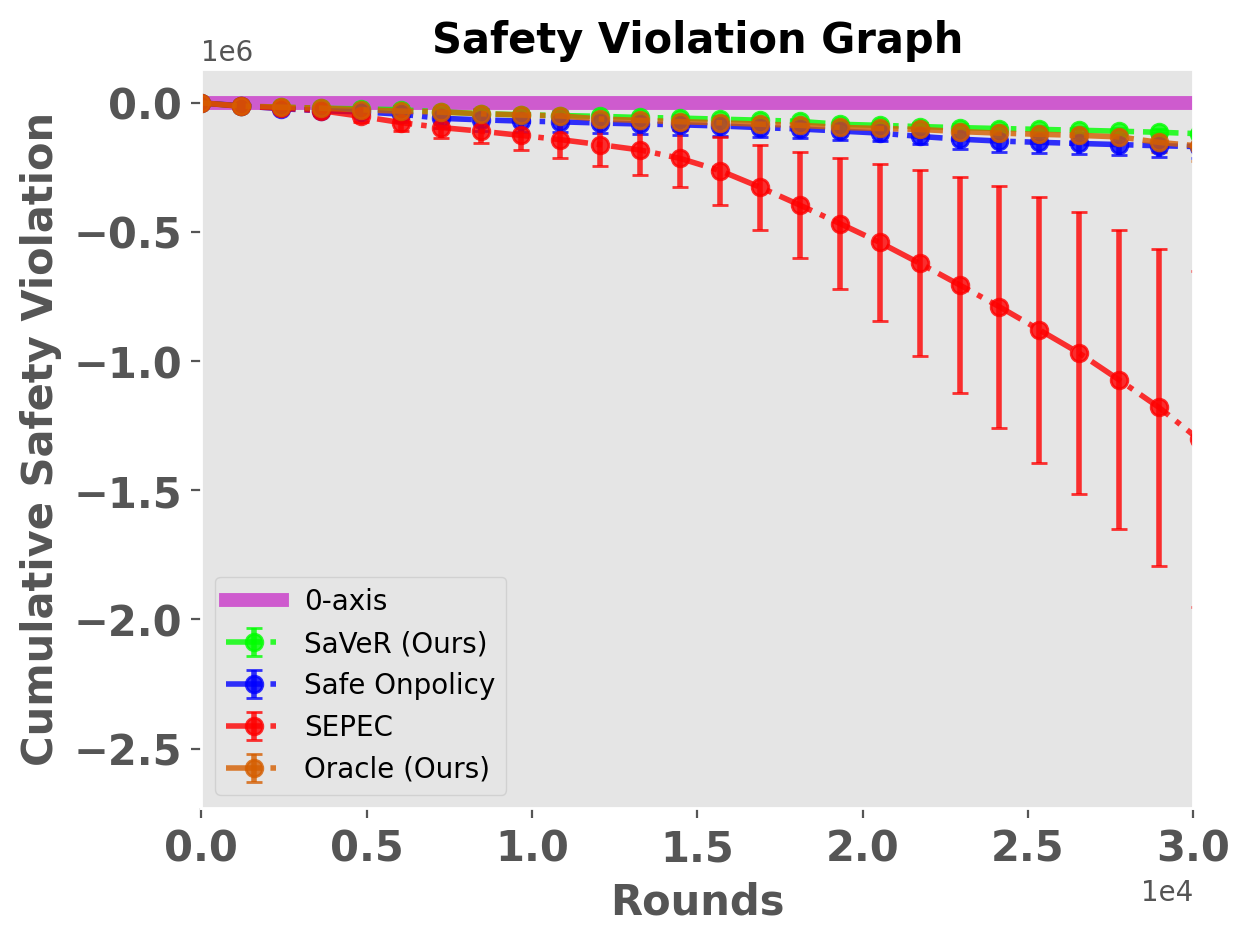}} \\
    \subfigure[Tree MDP violation]{\label{fig:expt-tree-mse-safe}\includegraphics[scale = 0.27]{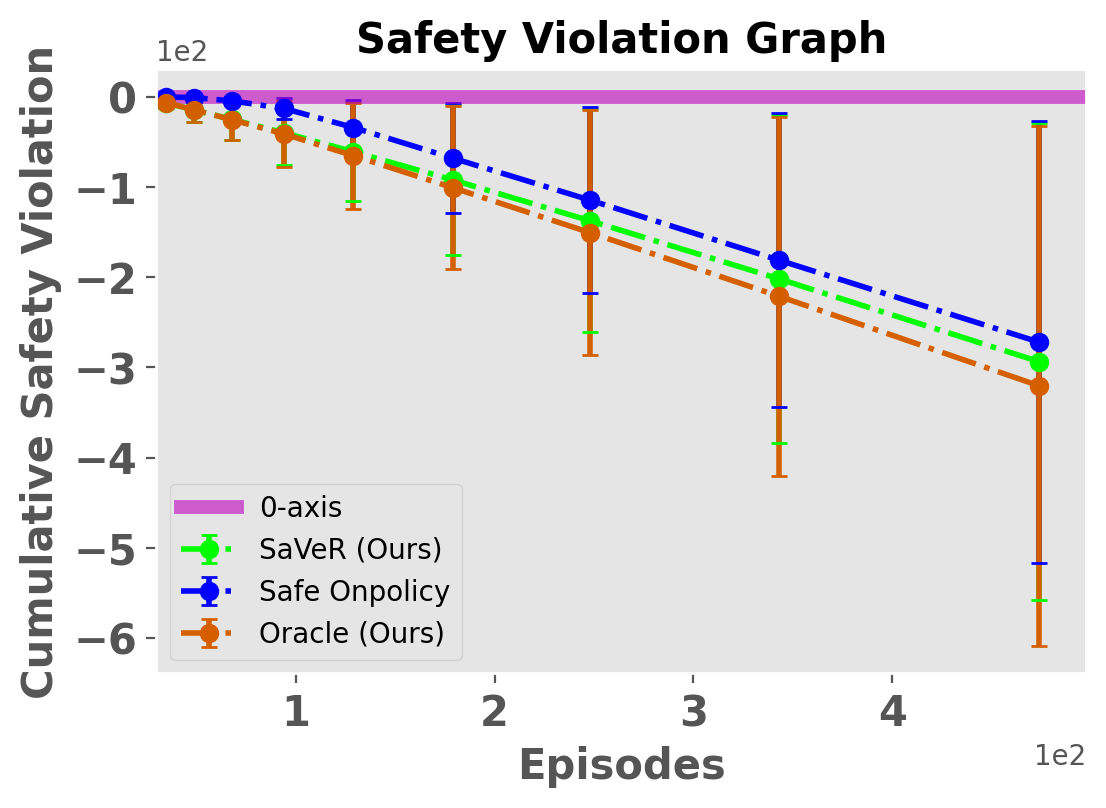}} &
    \subfigure[Grid MDP violation]{\label{fig:expt-grid-mse-safe}\includegraphics[scale = 0.24]{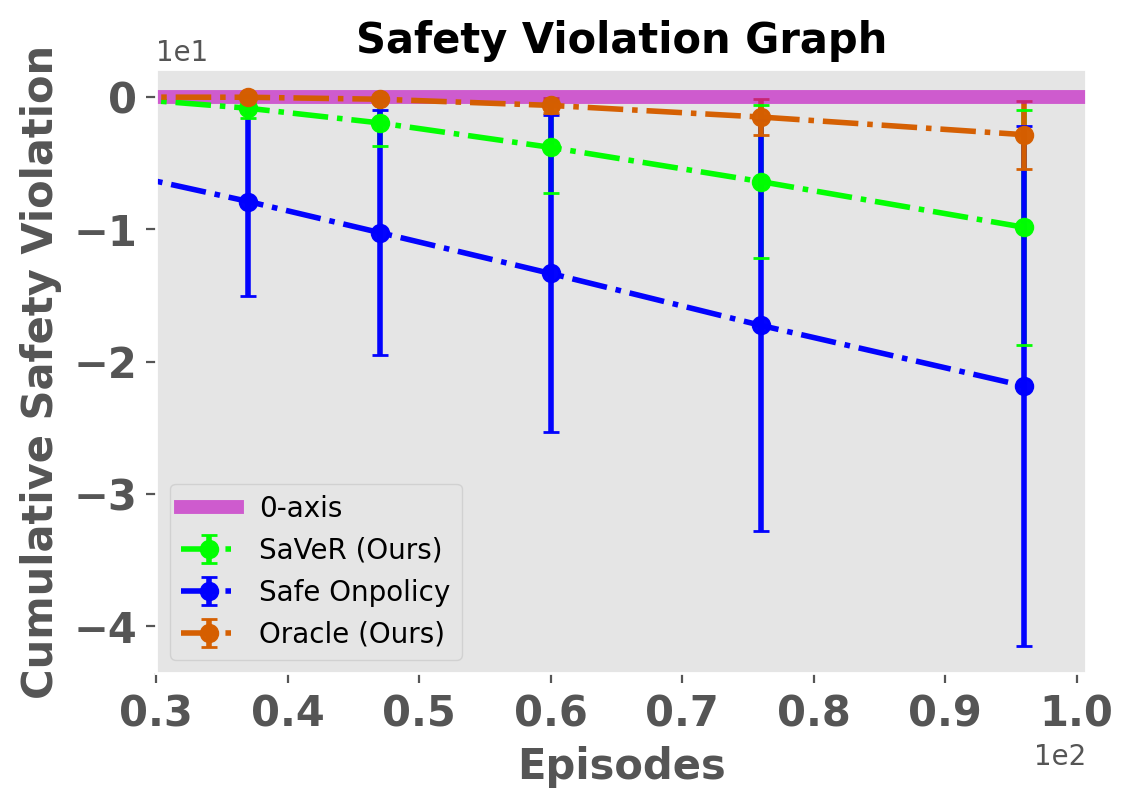}}
    \end{tabular}
    \vspace{-1.0em}
    \caption{The vertical axis gives cumulative constraint violation and the horizontal axis is the number of episodes/rounds. The $0$-axis is shown in pink. A safe algorithm has its plot below the $0$-axis with the plot showing the cumulative unsafe budget.}
    \label{app:fig-expt-2}
  \end{center}
  \vspace*{-1em}
\end{figure}


\textbf{Experiment 3 (Tree):} We experiment with a $4$-depth $2$-action deterministic tree MDP consisting of $15$ states. With increasing episodes \sav\ reaches lower MSE than safe on-policy and eventually matches the oracle's MSE in \Cref{fig:expt-tree-mse}. In \Cref{fig:expt-tree-mse-safe} the  \sav\ and oracle run the baseline policy almost similar number of times compared to the safe on-policy.


\textbf{Experiment 4 (Gridworld):} This setting consist of a $4\times 4$ stochastic gridworld of $16$ grid cells. We point out that Gridworld has a DAG structure (due to the finite horizon) which violates the tree structure assumption under which the oracle and \sav\ bounds were derived. 
Nevertheless, both \sav\ and oracle reach lower MSE with increasing episodes compared to safe onpolicy in \Cref{fig:expt-grid-mse}. We use \eqref{eq:mdp-saver-dag-sample} to estimate $\wb$ in this setting. In \Cref{fig:expt-grid-mse-safe} we see that \sav\ allocates more samples to reduce the MSE, whereas the safe on-policy runs the baseline policy more instead of focusing on reducing the MSE.

\vspace*{-0.5em}
\section{Conclusions}
\vspace*{-0.3em}
\label{sec:conc}
In this paper, we studied the question of how to take action to build a dataset for minimal-variance policy evaluation of a fixed target policy under a safety constraint \eqref{eq:safety-constraint-MDP}.
We developed a theoretical foundation for data collection in policy evaluation by showing that there exists a class of MDPs (namely tree-structured MDPs $\T$) where safe policy evaluation is intractable. 
We then showed the necessary condition for $\T$ to be tractable such that the optimal behavior policy can collect data without violating safety constraints. 
We then proved the first lower bound for this setting under the tractability conditions that scales as $\widetilde{\Omega}(n^{-3/2})$, where $\widetilde{\Omega}$ hides log factors.
%
%
We then introduced a practical algorithm, \sav, that approximates the optimal behavior strategy by computing an upper confidence bound on the variance of the cumulative cost in place of the true cost variances in the optimal behavior strategy.
We bound the finite-sample regret (excess MSE) of \sav\ and show that it scales as $\widetilde{O}(n^{-3/2})$ matching the lower bound.
Hence, we answer both the questions raised in the introduction positively.
%
%
In the future, we would like to extend our derivation of optimal data collection strategies and regret analysis of \sav\ to linear/contextual bandits and more general MDPs.
%


\newpage
\textbf{Acknowledgement:} J.\ Hanna was supported in part by American Family Insurance through a research partnership
with the University of Wisconsin—Madison’s Data Science Institute.

\textbf{Impact Statement}
In this paper, we study the safe data collection for policy evaluation in an RL setting under safety constraints. Our paper proposes a new adaptive data collection policy and addresses the theoretical challenges posed by this setting.
We focus on algorithmic and theoretical contributions and we do not address the challenges that might stem from incorrect feedback, human bias in feedback, false information, or social disparity in gathering the feedback (or dataset).
We therefore leave it to the users who apply our algorithm to use it responsibly and ethically.


\bibliographystyle{plainnat}
\bibliography{biblio}

\newpage
\appendix
\onecolumn
\section{Appendix}
\label{sec:appendix}
\subsection{Related Works}
\label{app:related-app}
Our work lies at the intersection of two areas: $1)$ optimal data collection for policy evaluation, and $2)$ safe sequential decision-making. Optimal data collection for policy evaluation has been studied in reinforcement learning \citep{antos2008active,carpentier2012minimax, carpentier2011finite, carpentier2015adaptive, hanna2017data, mukherjee2022revar, riquelme2017active, fontaine2021online,mukherjee2024speed,zhong_robust_2022} 
without considering the safety constraints. In the bandit setting the optimal data collection has been studied in the context of estimating a weighted sum of the mean reward associated with each arm. 
\cite{antos2008active} study estimating the mean reward of each arm equally well and show that the optimal solution is to pull each arm proportional to the variance of its reward distribution.
Since the variances are unknown a priori, they introduce an algorithm that pulls arms in proportion to the empirical variance of each reward distribution. 
A similar set of works by \citet{carpentier2012minimax,carpentier2015adaptive} extend the above work by introducing a weighting on each arm that is equivalent to the target policy action probabilities in our work.
They show that the optimal solution is then to pull each arm proportional to the product of the standard deviation of the reward distribution and the arm weighting.
The work of \citet{riquelme2017active, fontaine2021online, mukherjee2024speed} considers the linear bandit setting to study the policy evaluation setup where actions have different variances. Finally, \citet{mukherjee2022revar} study the policy evaluation setting for tabular MDP.
However, these works only look into the policy evaluation setting without considering the safety constraint introduced in \eqref{eq:safety-constraint-MDP}.

The safe sequential decision-making setup has recently attracted much attention in machine learning \citep{amodei2016concrete, TurchettaB019} and reinforcement learning \citep{efroni2020exploration, wachi2020safe, camilleri2022active}. In reinforcement learning, and specifically in the bandit setting, safety has been studied in the context of policy improvement. In the bandit literature regret minimization under safety constraints has been studied in \citet{wu2016conservative, KazerouniGAR17, AmaniAT19, GarcelonGLP20}. In these works the safety requirements are encoded in the form of constraints on the cumulative rewards observed by the learner. These works refer to the setup as conservative bandits because exploration is limited by the constraints on the cumulative reward. The work of \citet{wu2016conservative} consider the setting of stochastic bandits for policy improvement with a safety constraint similar to \eqref{eq:safety-constraint-MDP}. However, \citet{KazerouniGAR17, AmaniAT19, GarcelonGLP20, moradipari2021safe, pacchiano2021stochastic, hutchinson2024directional} study the linear bandit setting under safety constraints where the actions have features associated with them. Note that none of the above works study policy evaluation under safety constraints. 
\citet{wan2022safe, zhu2021safe, zhu2022safe} analyzes off policy evaluation in the context of designing a non-adaptive policy using inverse probability weighting estimator (as opposed to designing an adaptive policy using certainty equivalence estimator in this work). 

In the MDP setting the works of \citet{efroni2020exploration, altman2021constrained, wachi2024safe, li2024survival, zheng2024safe, xiong2024provably, ding2024reduced, wang2024safety, mazumdar2024safe} study different variations of the safe exploration in constraint MDPs in both offline and online policy improvement settings. The work of \citet{yang2024risk} studies the safe policy improvement in constraint MDP setting under non-stationary policies. The work of \citet{gupta2024model} proposed a safe policy improvement approach for variable horizon setting such that the safe reinforcement learning agent uses a variable look-ahead horizon to avoid unsafe states. The constrained MDP problems have also been looked into from the lens of optimization where \citet{chen2021primal, chen2022near, qiu2020upper, ding2020natural, vaswani2022near, ding2021provably, liang2018accelerated, ying2024scalable} have proposed a primal-dual sampling-based algorithm to solve CMDPs for the policy improvement setting.

\subsection{Previous results and Probability Tools}
\label{app:prev-result}

\begin{customproposition}{1}\textbf{(Restatement from \citet{carpentier2011finite})}
\label{prop:bandit}
In an $A$-action bandit setting, the estimated return of $\pi$ after $n$ action-reward samples is denoted by $Y_n$. 
Note that the expectation of $Y_n$ after each action has been sampled once is given by $V^\pi$. 
Minimal MSE, $\E_{\D}\left[\left(Y_n - V^\pi\right)^{2}\right]$, is obtained by taking actions in the proportion:
\begin{align}
    \bb_*(a) \coloneqq \dfrac{\pi(a)\sigma(a)}{\sum_{a'=1}^A \pi(a')\sigma(a')}.\label{eq:bandit-optimal-prop}
\end{align}
where $\bb^*(a)$ denotes the optimal sampling proportion.
\end{customproposition}

\begin{lemma}\textbf{(Wald's lemma for variance)}
\label{lemma:wald-variance}\citep{resnick2019probability}
Let $\left\{\mathcal{F}_{t}\right\}$ be a filtration and $R_{t}$ be a $\mathcal{F}_{t}$-adapted sequence of i.i.d. random variables with variance $\sigma^{2}$. Assume that $\mathcal{F}_{t}$ and the $\sigma$-algebra generated by $\left\{R_{t'}: t' \geq t+1\right\}$ are independent and $T$ is a stopping time w.r.t. $\mathcal{F}_{t}$ with a finite expected value. If $\mathbb{E}\left[R_{1}^{2}\right]<\infty$ then
\begin{align*}
\mathbb{E}\left[\left(\sum_{t'=1}^{n} R_{t'}-n \mu\right)^{2}\right]=\mathbb{E}[n] \sigma^{2}
\end{align*}
\end{lemma}

\begin{lemma}\textbf{(Restatement of Theorem 1 of \citet{mukherjee2022revar})}
\label{lemma:L-step-tree}
Assume the underlying MDP is an $L$-depth tree MDP as defined in \Cref{def:tree-mdp}. 
Let the estimated return of the starting state $s^1_1$ after $n$ state-action-reward samples be defined as $Y_{n}(s^1_1)$. 
Let $\D$ be the observed data over $n$ state-action-reward samples. To minimize MSE $\E_{\D}[(Y_n(s^1_1) - V^{\pi}(s^1_1))^2]$ the optimal sampling proportions for any arbitrary state is given by:
\begin{align*}
    \bb_*(a|s^{\ell}_i) \!&\propto
     \bigg(\! \pi^2(a|s_i^{\ell}) \bigg[\sigma^2(s^{\ell}_i, a)  + \sum\limits_{s^{\ell+1}_j}P(s^{\ell + 1}_j|s_i^{\ell}, a) M^2(s^{\ell+1}_j) \bigg]\bigg)^{1/2}, 
\end{align*}
where, $M(s^{\ell}_j)$ is the normalization factor defined as follows:
\begin{align*}
    M(s^{\ell}_{i})  \coloneqq \sum\limits_a\bigg(\pi^2(a|s^{\ell}_{i})\big(\sigma^2(s^{\ell}_{i}, a) + \sum\limits_{s^{\ell+1}_j}P(s^{\ell+1}_j|s^{\ell}_i, a) M^2(s^{\ell+1}_j)\big)\bigg)^{1/2}
\end{align*}
\end{lemma}





\section{Intractable MDP}
\label{app:intractable-mdp}
\begin{customproposition}{1}
    Fix an arbitrary $n > 0$. Then 
    there exists an environment where no algorithm (including the safe oracle $\bb^k_*$) can be run that will result in a regret $\cR_n = \L_n(\pi, \bb) - \L^*_n(\pi, \bb_*)$ of $\tO(n^{-3/2})$ while satisfying the safety constraint, where $\bb_*$ is the unconstrained oracle.
\end{customproposition}


\begin{proof}
We first consider a bandit setting where there are $3$ arms, action $\{0\}$ which is the safe action, and actions $1$ and $2$. Assume $\pi(a) = 1/A$ so that we can ignore its effect on optmal sampling policy $\bb_*$

\textbf{Case 1 (All actions safe):} First consider an environment when all actions are safe. That is $\mu^c(0) = 0$ and $\mu^c(1) = 1$ and $\mu^c(2) = 1-\epsilon$ and reward distributions are bounded between $[0,1]$. Therefore at round $\ell\in [L]$ we can guarantee for any $\alpha\in(0,1]$ that
\begin{align*}
    \sum_{\ell'=1}^\ell\sum_{a=0}^2\pi(a)\wmu_{c,\ell'}(a) \geq (1-\alpha)\ell \underbrace{\pi_0(0)\mu^c(0)}_{0}, \quad \forall \ell\in [L]
\end{align*}
where, $\pi_0$ always samples safe action $0$. Assume a safe oracle that knows the variances of the actions but does not know the means of the actions (both reward and cost means). Therefore from \citet{carpentier2011finite} we know that the optimal way to reduce the MSE $\min_\textbf{b} \E_{\D}[\left(Y^{\pi}_n(s_1) - V^{\pi}(s_1)\right)^{2}]$ is to run the policy $\bb_*(a) \propto \pi(a)\sigma(a)$. 
We also know from \citet{carpentier2011finite} that there exists an algorithm $\A^{safe}$ (like MC-UCB that tracks $\bb_*$) that achieves a regret after $n$ rounds as $\cR^{\mathrm{safe}}_n = \widetilde{O}(\frac{K\log(n)}{n^{3/2}})$ where $\widetilde{O}$ hides logarithmic factors and problem dependent factors like $\bb_{\min}$. 

\textbf{Case 2 (Some actions are unsafe):} In this case, we now analyze a safe oracle algorithm $\bb^k_*$.
%
Consider an environment where $\mu^c(0) = 0.5$, $\mu^c(1) = 0.5 + \alpha$, and $\mu^c(2) = 0$. Let the rewards be bounded in $[0,1]$ again. So action $\{2\}$ is unsafe.
Therefore safe oracle policy which first runs action $1$ for $C_1 n$ number of times for some $C_1 > 0$. Then it runs the safe action $0$ for $C_0 n$ number of times (for some $C_0>0$) such that it has enough safety budget and then it runs action $2$ for $n(1- (C_0 + C_1))$ number of times. 
Let the variance of $\sigma^{r,(2)}(0) = 0.001$, $\sigma^{r,(2)}(1) = 0.001$ and $\sigma^{r,(2)}(2) = 0.25$. 

The cost cumulative value over rounds for the algorithm for $\alpha= \frac{1}{4}$ is given by
\begin{align*}
    V^c_{\A} &= (C_1 n) (0.5 + \alpha) + n(1 - C_0 -C_1)0 +  (C_0 n) 0.5 = (C_1 n)\cdot\frac{3}{4} +  (C_0 n) \frac{2}{4} = \frac{n}{4}\left(3C_1 + 2C_0\right) .
\end{align*}
 Then to satisfy the safety budget we have to show that
\begin{align*}
    & V^c_{\A} \geq n(1-\alpha)0.5 \\
    \overset{(a)}{\implies} & \frac{n}{4}\left(3C_1 + 2C_0\right)  \geq \frac{3n}{8}\\
    \implies & 3C_1 + 2C_0 \geq \frac{3}{2}
\end{align*}
Say we just want to satisfy the safety constraint, then setting $C_1 = \frac{1}{4}$ and $C_0 = \frac{3}{8}$ in the above equation we can achieve that. Therefore we have that $T_n(1) = \frac{n}{4}$ and $T_n(0) = \frac{3n}{8}$. This implies that $T_n(2) = n - \frac{n}{4} - \frac{3n}{8} = \frac{3n}{8}$. Therefore we get that the loss of $\bb^k_*$ is given by 
\begin{align*}
    \L_n(\pi, \bb^k_*) =  \sum_{a, T_n(a)>0}\dfrac{\sigma^{r, (2)}(a)}{T_n(a)} = \frac{8(0.001)^2}{3n} + \frac{4(0.001)^2}{n} + \frac{8(0.25)^2}{3n} 
\end{align*}
Now we calculate the loss of the optimal data collection algorithm following the unconstrained $\bb_*$. Note that now $T^*_n(0) = \frac{0.001}{0.001+0.001+0.25}n = \frac{n}{252}$, $T^*_n(1) = \frac{n}{252}$ and  $T^*_n(2) = \frac{250n}{252}$. Then the loss of the optimal data collection algorithm following $\bb_*$ is given by
\begin{align*}
    \L^*_n(\pi, \bb_*) = \sum_{a, T^*_n(a)>0}\dfrac{\sigma^{r, (2)}(a)}{T^*_n(a)} = \frac{252(0.001)^2}{n} + \frac{252(0.001)^2}{n} + \frac{252(0.25)^2}{250n} \approx \dfrac{2}{4000n} + \dfrac{15}{n}.
\end{align*}
It follows then that the regret scales as
\begin{align*}
    \cR_n = \L_n(\pi, \bb^k_*) - \L^*_n(\pi, \bb_*) =   \sum_{a, T_n(a)>0}\dfrac{\sigma^{r, (2)}(a)}{T_n(a)} - \sum_{a, T^*_n(a)>0}\dfrac{\sigma^{r, (2)}(a)}{T^*_n(a)} = O\left(\dfrac{K}{n}\right) \geq \cR^{\mathrm{safe}}_n = \widetilde{O}(\frac{K\log(n)}{n^{3/2}}).
\end{align*}
Note that this regret rate holds for any $C_1 < C_0$ and we cannot shift any more proportion to action $\{2\}$. Therefore the algorithm will choose the sub-optimal safe action $\{0\}$ more than the action that reduces the MSE (to satisfy safety constraint) most resulting in a regret that scales as $n^{-1}$. So any algorithm (including the safe oracle algorithm) will not be able to achieve the desired regret rate of $\widetilde{O}(n^{-3/2})$.
%
%
%
The claim of the proposition follows.
\end{proof}

\begin{remark}\textbf{(Tractability condition)}
\label{rem:safe-state}
%
%
%
Let $\bb$ be any behavior policy that minimizes MSE. 
However, running $\bb$ only once is not enough to guarantee a regret of $\tO(n^{-3/2})$. Let $\bb$ be run for $K_b$ episodes to guarantee a regret of $\tO(n^{-3/2})$.
Note that $K_b$ is the number of rounds in the bandit setting.
Observe that 
the number of rounds (or episodes in case of MDP) $K_b$ is behavior policy specific. 



\textbf{Case 1 (Two action bandits):} Consider two action bandit setting such that $A=2$. Further, let $\pi(a) = 1/A$ and the left action has a constraint-value of $C_1$ while the right action has a constraint-value of $C_2$. Let the deterministic baseline policy $\pi_0$ always choose the left action, while the behavior policy $\bb$ chooses the right action. Note that $\bb$ may or may not be $\bb_*$. Then to satisfy the safety constraint \eqref{eq:safety-constraint-MDP} we need that
\begin{align*}
    (n-K_b)C_1 + K_bC_2 \geq (1-\alpha)n C_1 &\implies nC_1 - K_bC_1 + K_bC_2 \geq nC_1 - \alpha nC_1\\
    &\implies K_b(C_1 - C_2) \leq nC_1\alpha\\
    &\implies 1-\frac{C_2}{C_1} \leq \frac{n\alpha}{K_b}\\
    &\implies \frac{K_b}{\alpha}(1 - \frac{C_2}{ C_1}) \leq n\\
     &\implies n \geq \frac{K_b}{\alpha}\left( 1- \frac{C_2}{ C_1}\right)
\end{align*}
The above inequality shows two things, (1) the lower bound to the budget $n$ to run the behavior policy $\bb$ for $K_b$ rounds and satisfy the safety constraint; (2) The condition $C_1 > C_2$ has to be satisfied so that the RHS is positive. 

\textbf{Case 2 (General multi-armed bandits):} Now generalizing this to $A \geq 2$ 
we can show that the above condition can be modified into
\begin{align*}
    & (n-K_b)\mu^c(0) + K_b\min_{a\in\A\setminus\{0\}}\mu^c(a) \geq (1-\alpha)n\mu^c(0)\\
    \implies & n\mu^c(0) - K_b\mu^c(0) + K_b\min_{a\in\A\setminus\{0\}}\mu^c(a) \geq n\mu^c(0) -\alpha n\mu^c(0)\\
    \implies & K_b(\mu^c(0) - \min_{a\in\A\setminus\{0\}}\mu^c(a) ) \leq \alpha n\mu^c(0)\\
    \implies & 1 - \frac{\min_{a\in\A\setminus\{0\}}\mu^c(a)}{\mu^c(0)}  \leq \frac{\alpha n}{K_b}\\
    \implies &  \frac{K_b}{\alpha}\left(1 - \frac{\min_{a\in\A\setminus\{0\}}\mu^c(a)}{\mu^c(0)}\right)  \leq n\\
    \implies &  n \geq \frac{K_b}{\alpha}\left(1-\frac{\min_{a\in\A\setminus\{0\}}\mu^c(a)}{\mu^c(0)}\right)  
\end{align*}
The above inequality shows two things, (1) the lower bound to the budget $n$ to run the behavior policy $\bb$ for $K_b$ rounds and satisfy the safety constraint for a general $K_b$ armed bandit; (2) The condition $\min_{a\in\A\setminus\{0\}}\mu^c(a) < \mu^c(0)$ has to be satisfied  so that the RHS is positive. 



\textbf{Case 3 (Tabular MDP):} Define $V^{\bb^{-}}_{c}(s_1)$ as the value of the policy $\bb^{-}$ starting from state $s_1$. 
%
So this policy $\bb^{-}$ can be thought of as the worst possible policy that can be followed by the agent during an episode.
Let this policy be run for $K_{b^{-}}$ episodes. 
Also, recall that $V^{\pi_0}_{c}(s_1)$ is the value of the baseline policy $\pi_0$ starting from state $s_1$.
It can easily shown following a similar line of argument as case 2 that we need a budget of
\begin{align*}
    n \geq \frac{K_{b^{-}}}{\alpha}\left(1-\frac{V^{\bb^{-}}_{c}(s_1)}{V^{\pi_0}_{c}(s_1)} \right).
\end{align*}
Again the above inequality shows two things for a general Tree MDP: (1) the lower bound to the budget $n$ to run the behavior policy $\bb^{-}$ for $K_{b^{-}}$ episodes and satisfy the safety constraint for a Tree MDP; (2) $V^{\bb^{-}}_{c}(s_1) < V^{\pi_0}_{c}(s_1)$ so that the RHS is positive. 

Now observe that in the first two cases of the bandit setting the $V^{\bb^{-}}_{c}(s_1)$ yields $\min_{a\in\A\setminus\{0\}}\mu^c(a)$. Therefore combining all three cases we can state the budget $n \geq \frac{K_{b^{-}}}{\alpha}\left(1-\frac{V^{\bb^{-}}_{c}(s_1)}{V^{\pi_0}_{c}(s_1)} \right)$.
Now from \citep{carpentier2012minimax, mukherjee2022revar} we know that $K_{b^{-}} \geq C_{\sigma}(n-\sqrt{n})$ where $C_{\sigma} \in (0,1]$  is an MDP dependent parameter that depends on the reward variance of state-action pairs to achieve a regret bound of $\tO(n^{-3/2})$. 
We define the quantity $C_\sigma = \max_{s,a}\frac{\bb_*(a|s)}{M(s)}$ where $\bb_*(a|s)$ and $M(s)$ are defined in \eqref{eq:mdp-b-def} and \eqref{eq:M-def} respectively. Observe that $C_\sigma \in (0,1)$.
%
Then we have that
\begin{align*}
    & n \geq \frac{K_{b^{-}}}{\alpha}\left(1-\frac{V^{\bb^{-}}_{c}(s_1)}{V^{\pi_0}_{c}(s_1)} \right)
    \implies n \geq \frac{C_{\sigma}(n-\sqrt{n})}{\alpha}\left(1-\frac{V^{\bb^{-}}_{c}(s_1)}{V^{\pi_0}_{c}(s_1)} \right) \\
    \implies & n \geq \frac{C_{\sigma}n}{\alpha}\left(1-\frac{V^{\bb^{-}}_{c}(s_1)}{V^{\pi_0}_{c}(s_1)} \right)-\frac{\sqrt{n}}{\alpha}\left(1-\frac{V^{\bb^{-}}_{c}(s_1)}{V^{\pi_0}_{c}(s_1)} \right)\\
    \implies & n\left(1 -\frac{C_{\sigma}}{\alpha}\left(1-\frac{V^{\bb^{-}}_{c}(s_1)}{V^{\pi_0}_{c}(s_1)} \right)\right) + \frac{\sqrt{n}}{\alpha}\left(1-\frac{V^{\bb^{-}}_{c}(s_1)}{V^{\pi_0}_{c}(s_1)} \right) \geq 0 \\
    \implies & \sqrt{n}\left(\sqrt{n} -\frac{C_{\sigma}\sqrt{n}}{\alpha}\left(1-\frac{V^{\bb^{-}}_{c}(s_1)}{V^{\pi_0}_{c}(s_1)} \right) + \frac{1}{\alpha}\left(1-\frac{V^{\bb^{-}}_{c}(s_1)}{V^{\pi_0}_{c}(s_1)} \right)\right) \geq 0.
\end{align*}
This implies that
\begin{align*}
    &\sqrt{n} -\frac{C_{\sigma}\sqrt{n}}{\alpha}\left(1-\frac{V^{\bb^{-}}_{c}(s_1)}{V^{\pi_0}_{c}(s_1)} \right) + \frac{1}{\alpha}\left(1-\frac{V^{\bb^{-}}_{c}(s_1)}{V^{\pi_0}_{c}(s_1)} \right) \geq 0 \\
    \implies &\sqrt{n}\left(1 -\frac{C_{\sigma}}{\alpha}\left(1-\frac{V^{\bb^{-}}_{c}(s_1)}{V^{\pi_0}_{c}(s_1)} \right)\right) \geq -\frac{1}{\alpha}\left(1-\frac{V^{\bb^{-}}_{c}(s_1)}{V^{\pi_0}_{c}(s_1)} \right) \\
    \implies &\sqrt{n} \geq \dfrac{ -\frac{1}{\alpha}\left(1-\frac{V^{\bb^{-}}_{c}(s_1)}{V^{\pi_0}_{c}(s_1)} \right) }{\left(1 -\frac{C_{\sigma}}{\alpha}\left(1-\frac{V^{\bb^{-}}_{c}(s_1)}{V^{\pi_0}_{c}(s_1)} \right)\right)}\\
    \implies &\sqrt{n} \geq \dfrac{ \frac{1}{\alpha}\left(1-\frac{V^{\bb^{-}}_{c}(s_1)}{V^{\pi_0}_{c}(s_1)} \right) }{\frac{C_{\sigma}}{\alpha}\left(1-\frac{V^{\bb^{-}}_{c}(s_1)}{V^{\pi_0}_{c}(s_1)} \right) - 1}.
\end{align*}
This yields the tractability condition.
\end{remark}

\section{Tractable MDP and Lower Bounds}
\label{app:tract}
\textbf{Some Definitions for proving Lower Bound:} 
%
%
%
%
These definitions follow similar definitions in \citet{wagenmaker2022beyond}.
Define the $Q$-function that satisfies the Bellman equation as
$$
Q_\ell^\pi(s, a)=R_\ell(s, a)+\sum_{s^{\prime}} P_\ell\left(s^{\prime} \mid s, a\right) V_{\ell+1}^\pi\left(s^{\prime}\right)
$$
and $Q_{L+1}^\pi(s, a)=0$. 
Define the optimal $Q$-function as $Q_\ell^{\pi_*}(s, a) \coloneqq \sup _\pi Q_\ell^\pi(s, a), V_\ell^{\pi_*}(s)\coloneqq\sup _\pi V_\ell^\pi(s)$, and let $\pi^{\star}$ denote an optimal policy.
A policy $\widehat{\pi}$ is called $\epsilon$-optimal which satisfies the following
$$
V^{\pi_*}(s_1)-V^{\widehat{\pi}}(s_1) \leq \epsilon
$$
with probability greater than $1-\delta$ using as few episodes as possible. 
We further define a few more notations for proving the lower bound.
Define the suboptimality gap as
$$
\Delta_\ell(s, a)\coloneqq V_\ell^{\pi_*}(s)-Q_\ell^{\pi_*}(s, a) .
$$
such that $\Delta_\ell(s, a)$ denotes the suboptimality of taking action $a$ in $(s, h)$, and then playing the optimal policy henceforth. 
Define the state-action visitation distribution as:
$$
w_\ell^\pi(s, a)\coloneqq\Pb_\pi\left[s_\ell=s, a_\ell=a\right], \quad w_\ell^\pi(s)\coloneqq\Pb_\pi\left[s_\ell=s\right].
$$
Note that $w_\ell^\pi(s, a)=\pi_\ell(a | s) w_\ell^\pi(s)$. We denote the maximum reachability of $(s, \ell)$ by 
\begin{align*}
    W_\ell(s)\coloneqq\sup _\pi w_\ell^\pi(s).
\end{align*}
This is the maximum probability with which we could hope to reach $(s, \ell)$.
Define the best-policy gap-visitation complexity as $\mathcal{C}^{\star}(\T)$. Finally, recall that tree MDP is a subset of general MDPs which let us restate the following lemmas on lower bound for unconstrained tree MDPs from \citet{wagenmaker2022beyond}.

\begin{lemma}\textbf{(Divergence Lemma, Restatement of Lemma 4.1 from \citet{wagenmaker2022beyond})}
\label{lemma:divergence}
Consider tree MDPs $\T$ and $\T^{\prime}$ with the same state space $\mathcal{S}$, actions space $\mathcal{A}$, horizon $L$, and initial state distribution $P_0$. Fix some $(s, \ell) \in \mathcal{S} \times[L]$, and for any $a \in \mathcal{A}$ let $\nu_\ell(s, a)$ denote the law of the joint distribution of $\left(s^{\prime}, R\right)$ where $s^{\prime} \sim P_{\T}(\cdot \mid s, a)$ and $R \sim R_{\T}(s, a)$. 
Define the law $\nu_\ell^{\prime}(s, a)$ analogously with respect to $\T^{\prime}$. 
Fix some policy $\pi$ and let $\mathbb{P}_\T=\mathbb{P}_{\nu \pi}$ and $\mathbb{P}_{\T^{\prime}}=\mathbb{P}_{\nu^{\prime} \pi}$ be the probability measures on $\T$ and $\T'$ induced by the $\tau$-episode interconnection of $\pi$ and $\nu$ (respectively by $\pi'$ and $\nu'$).
For any almost-sure stopping time $\tau$ with respect to filtration $\left(\mathcal{F}_\tau\right)$,
\begin{align*}
\sum_{s, a, h} \mathbb{E}_{\T}\left[N_\ell^\tau(s, a)\right] \operatorname{KL}\left(\nu_\ell(s, a), \nu_\ell^{\prime}(s, a)\right) \geq \sup _{\xi \in \mathcal{F}_\tau} d\left(\Pb_{\T}(\xi), \Pb_{\T^{\prime}}(\xi)\right)
\end{align*}
where $d(x, y)=x \log \frac{x}{y}+(1-x) \log \frac{1-x}{1-y}$ and $N_\ell^\tau(s, a)$ denotes the number of visits to $(s, a, \ell)$ in the $\tau$ episodes.
\end{lemma}

\begin{lemma} \textbf{(Proposition 12 from \citet{wagenmaker2022beyond}}
\label{lemma:wagenmaker-prop-12}
Fix some tree MDP $\T$. Then:
\begin{enumerate}
    \item The set of valid state-action visitation distributions on $\T$ is convex.
    \item For any valid state-action visitation distribution on $\T$, there exists some policy that realizes it.
\end{enumerate}
\end{lemma}

\begin{lemma}\textbf{(Restatement of Lemma F.3 from \citet{wagenmaker2022beyond})} 
\label{lemma:wagenmaker-lemma-F3}
In the tree MDP $\T$, fix some $\bar{\ell}\in [L]$. Then
\begin{align*}
\mathcal{C}^{\star}(\T) \leq \inf _\bb \max _{s, a} \frac{1}{w_{\bar{\ell}}^\bb(s, a) \Delta_\ell(s, a)^2}+\max _{s, \ell} \frac{S A L}{W_\ell(s)} .
\end{align*}
is the complexity of the Tree MDP $\T$.
\end{lemma}

\begin{lemma}\textbf{(Proposition 4 from \citet{wagenmaker2022beyond})} 
\label{lemma:wagenmaker-prop-4}
The following bounds hold for any unconstrained tree MDP $\T$:
\begin{enumerate}
    \item $\mathcal{C}^{\star}(\T) \leq \frac{L^3 S A}{\epsilon^2}$
    \item $\mathcal{C}^{\star}(\T) \leq \sum_{\ell=1}^L \sum_{s, a} \min \left\{\frac{1}{W_\ell(s) \Delta_\ell(s, a)^2}, \frac{W_\ell(s)}{\epsilon^2}\right\}+\frac{L^2|\mathrm{OPT}(\epsilon)|}{\epsilon^2}$
    \item $\mathcal{C}^{\star}(\T) \leq \sum_{\ell=1}^L \sum_{s, a} \frac{1}{\epsilon \max \left\{\Delta_\ell(s, a), \epsilon\right\}}+\frac{L^2|\mathrm{OPT}(\epsilon)|}{\epsilon^2}$.
\end{enumerate}
where, $\mathcal{C}^{\star}(\T)$ is the complexity of the Tree MDP $\T$. The second term in $\mathcal{C}^{\star}(\T), L^2|\mathrm{OPT}(\epsilon)| / \epsilon^2$, captures the complexity of ensuring that after eliminating $\epsilon / W_\ell(s)$-suboptimal actions, sufficient exploration is performed to guarantee the returned policy is $\epsilon$-optimal. 
\end{lemma}


\begin{lemma}\textbf{(Restatement of Theorem 5 in \citet{carpentier2012minimax})}
\label{lemma:lower-bound-bandit}
Let $A \in \mathbb{N}$ be a set of actions for a bandit setting. Let $\inf$ be the infimum taken over all online sampling algorithms that reduce the MSE and $\sup$ represent the supremum taken over all environments. Define the regret of the algorithm over the target policy $\pi$ as $\cR_n \coloneqq \L_n(\pi) - \L^*_n(\pi)$ where $\L_n(\pi)$ is the MSE of the target policy following the algorithm.  Then:
\begin{align*}
\inf \sup \E\left[\cR_{n}\right] \geq C \frac{A^{1 / 3}}{n^{3 / 2}},
\end{align*}
where $C$ is a numerical constant, and $n$ is the total budget,
\end{lemma}

\setcounter{figure}{3}
\begin{figure}
    \centering
    \includegraphics[scale=0.6]{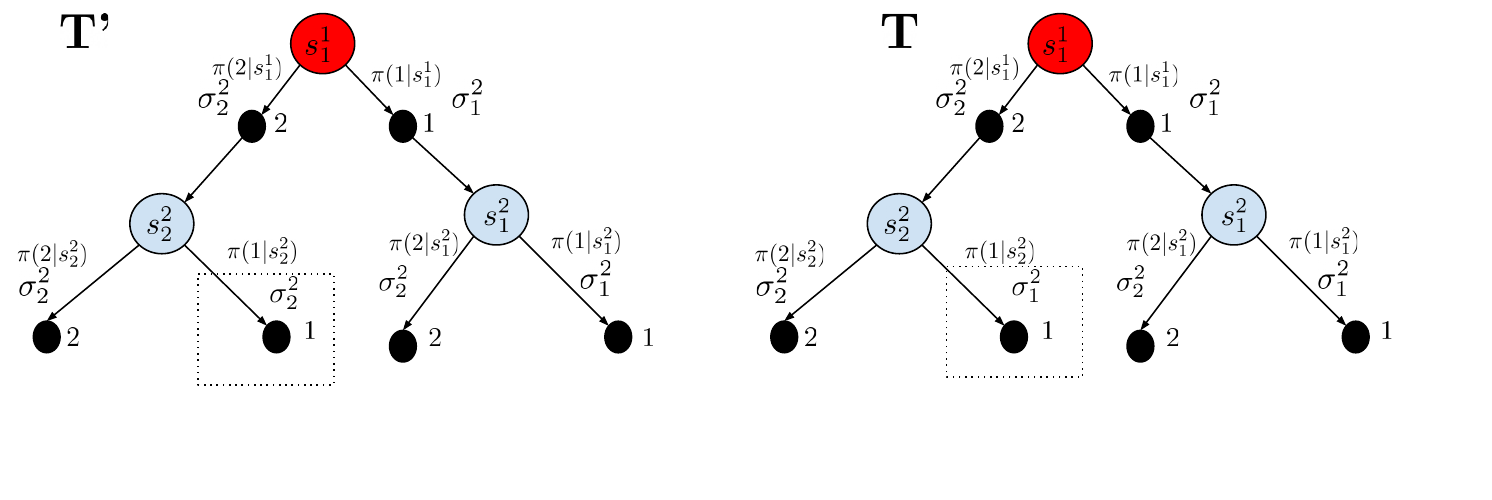}
    \caption{Tractable Tree MDPs $\T$ and $\T'$. The difference between the two Tree MDPs is highlighted in the square box.}
    \label{fig:Tractable-MDP}
\end{figure}

\begin{lemma}
\label{lemma:tab-RL-lower-bound}
Define the regret of the algorithm over the target policy $\pi$ as $\cR_n \coloneqq \L_n(\pi, \bb) - \L^*_n(\pi, \bb_*)$ where $\L_n(\pi, \bb)$ is the MSE of the target policy following the algorithm and $\bb_*$ is the unconstrained oracle behavior policy. 
The reward regret in tree MDP $\T$ is lower bounded by 
$$
\inf \sup \E\left[\cR_n\right]\geq \Omega\left(\frac{\sqrt{SAL^2 \log(1/\delta)}}{n^{3/2}}\right).
$$
\end{lemma}

\begin{proof}
We prove this lemma in two steps. In the first step, we prove the minimum number of episodes required by an $\epsilon$-optimal policy $\bb$ in tree MDP $\T$ (\cref{fig:Tractable-MDP}) such that $V^{\bb_*}(s_1)-V^{\bb}(s_1) \leq \epsilon$. Next in step $2$ we show that given this minimum number of episodes, what is the loss suffered by $\bb$ against $\bb_*$ at the end of episode $K$.

\textbf{Step 1 (Minimum episodes):}
We consider the two tree MDPs $\T$ and $\T'$ shown in \Cref{fig:Tractable-MDP}. 
We will apply \Cref{lemma:divergence} on our MDP, $\T$, and MDP $\T^{\prime}$ which is identical to $\T$ except in state $(s^2_2,1)$ where we have $\sigma^2(s^2_2,1) = (\mu-\Delta)(1-\mu+\Delta)$ in $\T'$ and $\sigma^2(s^2_1,1) = (\mu+\alpha)(1-\mu-\alpha)$ for $\T$ and some $\Delta > 0$. This yields a different $\bb_*$ for MDP $\T$ than $\bb_*$ for $\T'$. 

Fix some $\bar{\ell}\in[L]$. Since $\T$ and $\T^{\prime}$ are identical at all points but this one, we have
\begin{align*}
\sum_{s, a, \ell} & \E_{\T}\left[N_\ell^\tau(s, a)\right] \mathrm{KL}\left(\operatorname{Bernoulli}(\mu -\Delta), \operatorname{Bernoulli}(\mu +\alpha)\right) \\
&=\E_{\T}\left[N_{\bar{\ell}}^\tau(s, a)\right] \mathrm{KL}\left(\operatorname{Bernoulli}(\mu -\Delta), \operatorname{Bernoulli}(\mu +\alpha)\right) .
\end{align*}
where, $\E_{\T}, \E_{\T'}$ denotes the expectation over the data collected in tree MDP $\T$ and $\T'$ respectively following policy $\bb_*$.


Let $\bb_*$ denote the optimal policy on $\T$, and $\bb$ denote the $\epsilon$ -optimal policy by any other algorithm. Let the event $\xi=\{\bb=\bb_*\}$. Since we assume  algorithm is $\delta$-correct, and since the optimal policies on $\T$ and $\T^{\prime}$ differ, we have $\mathbb{P}_{\T}(\xi) \geq 1-\delta$ and $\mathbb{P}_{\T^{\prime}}(\xi) \leq \delta$. By \citet{garivier2016optimal}, we can then lower bound
$$
d\left(\mathbb{P}_{\T}(\xi), \mathbb{P}_{\T^{\prime}}(\xi)\right) \geq \log \frac{1}{2.4 \delta}
$$
Thus, by \Cref{lemma:divergence}, we have shown that, for any $(s, a), a \neq \bb_{*,\bar{\ell}}(s)$,
$$
\E_{\T}\left[N_{\bar{\ell}}^\tau(s, a)\right] \geq \frac{1}{\mathrm{KL}\left(\operatorname{Bernoulli}(\mu -\Delta), \operatorname{Bernoulli}(\mu +\alpha)\right)} \cdot \log \frac{1}{2.4 \delta}
$$
%
%
For small $\alpha > 0$, we can bound (see e.g. Lemma 2.7 of \citet{Tsybakov2009springer})
$$
\mathrm{KL}\left(\operatorname{Bernoulli}(\mu -\Delta), \operatorname{Bernoulli}(\mu +\alpha)\right) \leq 6(\Delta-\alpha)^2 .
$$
Taking $\alpha \rightarrow 0$, we have
$$
\E_{\T}\left[N_{\bar{\ell}}^\tau(s, a)\right] \geq \frac{1}{6 \Delta^2} \cdot \log \frac{1}{2.4 \delta}.
$$
We can write $\E_{\T}\left[N_{\bar{\ell}}^\tau(s, a)\right]=\E_{\T}\left[\sum_{k=1}^\tau w_{\bar{\ell}}^{\bb^k}(s, a)\right]$ where $\bb^k$ denotes the policy the algorithm played at episode $k$. Note that all state-visitation distributions lie in a convex set in $[0,1]^{S A}$ and that for any valid state-visitation distribution, there exists some policy that realizes it, by \Cref{lemma:wagenmaker-prop-12}. By Caratheodory's Theorem, it follows that there exists some set of policies $\Pi$ with $|\Pi| \leq S A+1$ such that, for any $\bb$ and all $s, a, w_{\bar{\ell}}^\bb(s, a)=\sum_{\bb^{\prime} \in \Pi} \lambda_{\bb^{\prime}} w_{\bar{\ell}}^{\bb^{\prime}}(s, a)$, for some $\lambda \in \triangle_{\Pi}$. Note that $\lambda$ is a distribution over the policies in $\Pi$. Letting $\lambda^k$ denote this distribution satisfying the above inequality for $\bb^k$, it follows that
\begin{align*} \E_{\T}\left[\sum_{k=1}^\tau w_{\bar{\ell}}^{\bb^k}(s, a)\right] & =\E_{\T}\left[\sum_{k=1}^\tau \sum_{\bb \in \Pi} \lambda_\bb^k w_{\bar{\ell}}^\bb(s, a)\right] \\ 
& =\sum_{\bb \in \Pi} \E_{\T}\left[\sum_{k=1}^\tau \lambda_\bb^k\right] w_{\bar{\ell}}^\bb(s, a) \\ 
& =\E_{\T}[\tau] \sum_{\bb \in \Pi} \frac{\E_{\T}\left[\sum_{k=1}^\tau \lambda_\bb^k\right]}{\E_{\T}[\tau]} w_{\bar{\ell}}(s, a) .
\end{align*}
Note that $\sum_{\bb \in \Pi} \E_{\T}\left[\sum_{k=1}^\tau \lambda_\bb^k\right]=\E_{\T}\left[\sum_{k=1}^\tau \sum_{\bb \in \Pi} \lambda_\bb^k\right]=\E_{\T}[\tau]$ so it follows that $\left(\frac{\E_{\T}\left[\sum_{k=1}^\tau \lambda_\bb^k\right]}{\E_{\T}[\tau]}\right)_{\bb \in \Pi} \in$ $\triangle_{\Pi}$. Thus, a $\delta$-correct algorithm must satisfy, for all $s, a$ and some $\lambda \in \triangle_{\Pi}$,
$$
\E_{\T}[\tau] \geq \frac{1}{6 \Delta^2 \cdot \sum_{\bb \in \Pi} \lambda_\bb w_{\bar{\ell}}^\bb(s, a)} \cdot \log \frac{1}{2.4 \delta}.
$$
Since the set of state visitation distributions is convex, and since for any state-visitation distribution we can find some policy realizing that distribution, for any $\lambda \in \triangle_{\Pi}$, it follows that there exists some $\bb^{\prime}$ such that, for all $s, a, \sum_{\bb \in \Pi} \lambda_\bb w_{\bar{\ell}}^\bb(s, a)=w_{\bar{\ell}}^{\bb^{\prime}}(s, a)$. So, we need, for all $s, a$
$$
\E_{\T}[\tau] \geq \frac{1}{6 \Delta^2 \cdot w_{\bar{\ell}}^\bb(s, a)} \cdot \log \frac{1}{2.4 \delta}.
$$

It follows that every $\delta$-correct algorithm must satisfy
\begin{align*}
\E_{\T}[\tau] &\geq \inf _\bb \max _{s, a} \frac{1}{6 \Delta^2 \cdot w_{\bar{\ell}}^\bb(s, a)} \cdot \log \frac{1}{2.4 \delta},\\
& \gtrsim \mathcal{C}^{\star}(\T) \cdot \log \frac{1}{2.4 \delta}-\max _{s, \ell} \frac{S A L}{W_\ell(s)}
\end{align*}
from which the first inequality follows, and the second inequality follows from \Cref{lemma:wagenmaker-lemma-F3}.

The second term in $\mathcal{C}^{\star}(\T), L^2|\mathrm{OPT}(\epsilon)| / \epsilon^2$, captures the complexity of ensuring that after eliminating $\epsilon / W_\ell(s)$-suboptimal actions, sufficient exploration is performed to guarantee the returned policy is $\epsilon$-optimal. 
Using \Cref{lemma:wagenmaker-prop-4} we have that $\mathcal{C}^{\star}(\T), L^2|\mathrm{OPT}(\epsilon)| / \epsilon^2$
will be no worse than $L^3 S A / \epsilon^2$, it could be much better, if in the MDP the number of $(s, a, \ell)$ with $\Delta_\ell(s, a) \lesssim \epsilon / W_\ell(s)$ is small (note that since $\Delta_\ell(s, a) \geq \Delta_{\min }(s, \ell)$ by definition, $\operatorname{OPT}(\epsilon)$ will only contain states for which the minimum non-zero gap is less than $\epsilon / W_\ell(s)$ ). 
\citet{wagenmaker2022beyond}  obtains the bounds on $\mathcal{C}^{\star}(\T)$ in \Cref{lemma:wagenmaker-prop-4}, providing an interpretation of $\mathcal{C}^{\star}(\T)$ in terms of the maximum reachability, and illustrating $\mathcal{C}^{\star}(\T)$ is no larger than the minimax optimal complexity. 
This implies that 
$$
\E_{\T}[\tau] \gtrsim \Omega\left(\frac{S A L^2}{\epsilon^2} \log (1 / \delta)\right).
$$
Hence the $V^{\bb^K}(s^1_1) - V^{\bb_*}(s^1_1) \leq \epsilon$ for $K\geq \dfrac{S A L^2}{\epsilon^2} \log (1 / \delta)$.

\textbf{Step 2 (Bound regret in $\T$):} In the $\T$ in \Cref{fig:Tractable-MDP} we now have
\begin{align*}
    M(s^1_1) = \sqrt{2\sigma^2_1 + \sigma^2_2} + \sqrt{2\sigma^2_2 + \sigma^2_1},\quad M(s^2_1) = M(s^2_2) = \sigma_1 + \sigma_2
\end{align*}
Define confidence interval $\beta^K_L = L\sqrt{SA\log(SAL^2/\delta)/n}$. It can be shown using pointwise uncertainty estimation from \Cref{corollary:conc-var}
%
that
\begin{align}
    |\wsigma_{K,1} - \sigma_1| \leq \beta^K_L,\quad |\wsigma_{K,2} - \sigma_2| \leq \beta^K_L \label{eq:conc-lower-bound}
\end{align}
holds with probability greater than $1-\delta$, where the $\wsigma_{K,1}$, $\wsigma_{K,2}$ denote the estimated variances after $K$ episodes.
Then the loss of the agnostic algorithm at the end of the $K$-th episode is given by
\begin{align*}
    \L^K_n(\pi, \bb) &= \dfrac{\sqrt{2\wsigma^2_{K,1} + \wsigma^2_{K,2}} + \sqrt{2\wsigma^2_{K,2} + \wsigma^2_{K,1}}}{n} \\
    &\overset{(a)}{\geq} \dfrac{\sqrt{2(\sigma^2_{1} - \beta^K_L) + \sigma^2_{2} - \beta^K_L} + \sqrt{2(\sigma^2_{2} - \beta^K_L) + \sigma^2_{1} - \beta^K_L}}{n}\\
    &= \dfrac{\sqrt{2\sigma^2_{1} + \sigma^2_{2} - 3\beta^K_L} + \sqrt{2\sigma^2_{2} + \sigma^2_{1} - 3\beta^K_L}}{n}\\
    &\overset{(b)}{\geq} \dfrac{\sqrt{2\sigma^2_{1} + \sigma^2_{2}} + \sqrt{2\sigma^2_{2} + \sigma^2_{1}}}{n} - C\dfrac{\beta^K_L}{n}
\end{align*}
where, $(a)$ follows from concentration inequality in \eqref{eq:conc-lower-bound}, and $(b)$ follows for some appropriate constant $C>0$. 
Then for $K\geq \frac{S A L^2}{\epsilon^2} \log (1 / \delta)$ (from step 1) we have the total loss as
\begin{align*}
    \L_n(\pi, \bb) = \L^K_n(\pi,\bb) &\geq \left(\dfrac{\sqrt{2\sigma^2_{1} + \sigma^2_{2}} + \sqrt{2\sigma^2_{2} + \sigma^2_{1}}}{n} - \dfrac{\beta^K_L}{n}\right)\dfrac{S A L^2}{\epsilon^2} \log (1 / \delta)\\
    &\overset{(a)}{\geq} \underbrace{\dfrac{\sqrt{2\sigma^2_{1} + \sigma^2_{2}} + \sqrt{2\sigma^2_{2} + \sigma^2_{1}}}{n}}_{\L_n(\pi, \bb_*)} + \dfrac{\beta^K_L}{n}\\
    \L_n(\pi, \bb) - \L^*_n(\pi, \bb_*) &\overset{(b)}{\geq} \dfrac{\sqrt{SAL^2\log(SAL^2/\delta)}}{n\sqrt{K}} = \Omega\left( \dfrac{\sqrt{SAL^2 \log(1/\delta)}}{n^{3/2}} \right)
\end{align*}
where, $(a)$ follows by first setting $\epsilon = 1/\sqrt{n}$ and then noting that 
\begin{align*}
    \left(\sqrt{2\sigma^2_{1} + \sigma^2_{2}} + \sqrt{2\sigma^2_{2} + \sigma^2_{1}} - \beta^K_L\right)(S A L^2) \log (1 / \delta) \geq \dfrac{\sqrt{2\sigma^2_{1} + \sigma^2_{2}} + \sqrt{2\sigma^2_{2} + \sigma^2_{1}}}{n} + \dfrac{\beta^K_L}{n}.
\end{align*}
Also note that $\L_n(\pi, \bb_*) = \dfrac{\sqrt{2\sigma^2_{1} + \sigma^2_{2}} + \sqrt{2\sigma^2_{2} + \sigma^2_{1}}}{n}$. The $(b)$ follows by substituting the value of $\beta^K_L$. 
The claim of the lemma follows.
%
%
%
%
%
\end{proof}


\begin{customtheorem}{1}\textbf{(Lower Bound, formal)}
Let $\pi(a|s) = \tfrac{1}{A}$ for each state $s\in\S$. Assume the MDP $\M$ is tractable under \Cref{assm:tractable-MDP} and satisfies \eqref{eq:safe-state}. Then the reward regret is lower bounded by 
\begin{align*}
    \E\left[\cR_n\right] = \L_n(\pi, \bb) - \L^*_n(\pi, \bb^k_*)\geq\begin{cases}\Omega\left(\max \left\{\frac{A^{1/3}}{n^{3/2}}, \left(\frac{H_{*, (1)}^2 A^{2/3}}{n^{3/2}}\right)\right\}\right), &\quad \textbf{(MAB)}\\
    \Omega\left(\max \left\{\frac{\sqrt{SAL^2}}{n^{3/2}}, \left(\frac{H_{*, (1)}^2 SAL^2}{n^{3/2}}\right)\right\}\right) &\quad \textbf{(Tabular MDP)}
    \end{cases}
\end{align*}
where, $\Delta_0=|V^{\bb^k_*}_c(s^1_1)-V^{\pi_0}_{c}(s^1_1)|$ and $H_{*, (1)} = \frac{1}{\alpha V^{\pi_0}_{c}(s^1_1)}(\alpha V^{\pi_0}_{c}(s^1_1)+\Delta_0)$ is the hardness parameter.
\end{customtheorem}

\begin{proof}
We follow a reduction-based proof technique to prove this lower bound \citep{yang2021reduction}. 

\textbf{Step 1 (Reduction):} First recall we have that the regret for any online algorithm $\al$ that minimizes the MSE $\L_n(\pi)$ is given by $\cR_n(\al) = \L_n(\pi, \bb) - \L^*_n(\pi, \bb^k_*)$, where $\L^*_n(\pi, \bb)$ is the MSE of the oracle algorithm. We also assume $\pi(a) = 1/A$ for all $a\in\A$, and $\sigma(a) \geq \frac{1}{16}$ for all $a$. 

Now consider any sequential decision-making problem $\mathfrak{A}$ (for instance a multi-armed bandit problem) such that there exists $\xi \in \mathbb{R}$ (a constant solely depending on the sequential decision-making problem, e.g., the number of actions in bandits, or state-action-horizon in tabular RL), an instance of problem $\mathfrak{A}$ where for the budget $n$ large enough and any algorithm $\al$ we have from \Cref{lemma:lower-bound-bandit} and \Cref{lemma:tab-RL-lower-bound} that:
\begin{align}
\E\left[\cR_n^{\mathfrak{A}}(\al)\right] \geq \dfrac{\xi}{n^{3/2}}, \label{eq:lower-1}
\end{align}

For instance, in the MAB case $\xi= A^{1/3}$ with $A$ the number of arms and in tabular RL $\xi = SAL^2$. Using this non-conservative (unconstraint) lower bound, we show our lower bound for the safe setting for the problem $\mathfrak{A}$ with a baseline policy $\pi_0$. We assume the MDP $\T\subset \M$ where we run the behavior policy $\bb^k_*$ satisfies \Cref{assm:tractable-MDP}. This is required because otherwise we will not be able to run the behavior policy a sufficient number of times to reach a regret bound of $\tO(n^{-3/2})$(see \Cref{prop:intractable}). To do so, let's consider any safe algorithm (that is to say it satisfies safety constraint) noted as $\al_c$. We assume this algorithms selects behavior policies $\left(\bb^t\right)_{t \in[n]}$ and let $\N_0$ denotes the set of episodes in $\{1, \ldots, K\}$ where $\al_c$ selects the safe policy $\pi_0$. Let $|\N_0| = N_0$ and $\Delta_0\coloneqq|V^{\bb^k_*}-V^{\pi_0}_{c}|$. Here we assume the budget $n$ is large such that $n \geq SAL^2/\epsilon^2$ for some $\epsilon > 0$ (see \Cref{lemma:tab-RL-lower-bound}) and 
\begin{align*}
    n &\geq \sqrt{\frac{\xi}{\alpha V^{\pi_0}_{c}(s^1_1) \cdot\left(\alpha V^{\pi_0}_{c}(s^1_1)+\Delta_0\right)}+\frac{\xi^2}{4\left(\alpha V^{\pi_0}_{c}(s^1_1)+\Delta_0\right)^2}} \\
    \implies& n^2 \geq \frac{\xi}{\alpha V^{\pi_0}_{c}(s^1_1) \cdot\left(\alpha V^{\pi_0}_{c}(s^1_1)+\Delta_0\right)}+\frac{\xi^2}{4\left(\alpha V^{\pi_0}_{c}(s^1_1)+\Delta_0\right)^2}.\\
    \implies& n \geq \frac{\xi}{n\alpha V^{\pi_0}_{c}(s^1_1) \cdot\left(\alpha V^{\pi_0}_{c}(s^1_1)+\Delta_0\right)}+\frac{\xi^2}{4n \left(\alpha V^{\pi_0}_{c}(s^1_1)+\Delta_0\right)^2}
\end{align*}


\textbf{Step 2 (Loss estimate):} Let $\L(N_0)$ be the loss suffered in first $N_0$ episodes. We now distinguish two cases:

\textbf{(a)} If $\E\left[\L(N_0)\right] \geq \frac{\xi}{n\alpha V^{\pi_0}_{c}(s^1_1) \cdot\left(\alpha V^{\pi_0}_{c}(s^1_1)+\Delta_0\right)}$, then the definition of the regret implies that:
\begin{align}
\E\left[\cR_n^{\mathfrak{A}}(\al)\right] 
=\E\left[\L(N_0)\right] \cdot \Delta_0 \geq \frac{\xi \Delta_0}{n\alpha V^{\pi_0}_{c}(s^1_1) \cdot\left(\alpha V^{\pi_0}_{c}(s^1_1)+\Delta_0\right)} . \label{eq:lower-2}
\end{align}

 \textbf{(b)} If $\E\left[\L(N_0)\right] <\frac{\xi}{n\alpha V^{\pi 0} \cdot\left(\alpha V^{\pi_0}_{c}(s^1_1)+\Delta_0\right)}$, then let's note $\N_0^C=\left\{i_1, i_2, \cdots, i_{\left|\N_0^c\right|}\right\}$ the set of episodes where $\al_c$ does not execute the baseline policy $\pi_0$. Now consider the safety budget (similar to Definition 1 of \citet{yang2021reduction}) we have:
\begin{align*}
B_{\N_0^c}\left(\al_c\right) & =\max _{t \in \N_0^c} \E \sum_{k=1}^t\left[(1-\alpha) V^{\pi_0}_{c}(s^1_1)-V^{\pi^t}(s^1_1)\right] \\
& =\max _{t \in \N_0^c} \E \sum_{k=1}^t\left[V^{\bb^k_*}(s^1_1)-V^{\pi^t}(s^1_1)-\alpha V^{\pi_0}_{c}(s^1_1)-\left(V^{\bb^k_*}_c(s^1_1)-V^{\pi_0}_{c}(s^1_1)\right)\right] \\
& =\max_{t \in \N_0^c} \E\left[\cR^{\mathfrak{A}}_{\N_0^C}\left(\mathcal{A}_c\right)(t)\right]-\left(\alpha V^{\pi_0}_{c}(s^1_1)+\Delta_0\right) t,
\end{align*}
where $\Delta_0=V^{\bb^k_*}_c(s^1_1)-V^{\pi_0}_{c}(s^1_1)$ is the difference between the constraint value of the optimal policy and the baseline policy and $\E\left[R_{\mathfrak{A}}^{\N_0^C}\left(\mathcal{A}_c\right)(t)\right]$ is the regret incurred by the episodes $\left\{i_k\right\}_{k \in[t]}$. Therefore, for any $t \in\left[\left|\mathcal{T}_0^c\right|\right]$, by \eqref{eq:lower-1} we have that there exists an instance $u$ (for instance in a bandit problem $u$ is the means of each arm) of $\mathfrak{A}$ such that $\E\left[R_{\mathfrak{A}}^{\N_0^C}\left(\mathcal{A}_c\right)(t)\right] \geq \dfrac{\xi}{t^{3/2}}$. Let $t_0=\frac{\xi^2}{4n\left(\alpha V^{\pi_0}_{c}(s^1_1)+\Delta_0\right)^2}$, then there exists an instance such that
\begin{align*}
B_{\N_0^C}\left(\al_c\right) \geq \dfrac{\xi}{t_0^{3/2}}-\left(\alpha V^{\pi_0}_{c}(s^1_1)+\Delta_0\right) t_0 = \dfrac{4\left(\alpha V^{\pi_0}_{c}(s^1_1)+\Delta_0\right)^3 n^{3/2}}{\xi^2} - \dfrac{\xi^2}{4(\alpha V^{\pi_0}_{c}(s^1_1)+\Delta_0)}\dfrac{1}{n^2}\\
\overset{(a)}{\gtrsim} \frac{(\alpha V^{\pi_0}_{c}(s^1_1)+\Delta_0)^2\xi^2}{n^{3/2}} .
\end{align*}
where, $(a)$ follows as $n^{3/2} - n^{-2} \geq n^{-3/2}$. 
Combining the safety condition in \cref{eq:safety-constraint-MDP}, we have
\begin{align*}
\E\left[\L(\N_0)\right] \geq \frac{B_{\N_0}\left(\al_c\right)}{\alpha V^{\pi_0}_{c}(s^1_1)} \gtrsim \frac{\left(\alpha V^{\pi_0}_{c}(s^1_1)+\Delta_0\right)^2\xi^2}{\alpha V^{\pi_0}_{c}(s^1_1) n^{3/2}} .
\end{align*}
By the same derivation of \cref{eq:lower-2}, we have
\begin{align}
\E\left[R_n^{\mathfrak{A}}(\al)\right] \gtrsim \frac{\xi^2 \Delta_0}{n \alpha V^{\pi_0}_{c}(s^1_1) \cdot\left(\alpha   V^{\pi_0}_{c}(s^1_1)+\Delta_0\right)} \overset{(a)}{\geq}  \frac{\xi^2}{n^{3/2} \alpha V^{\pi_0}_{c}(s^1_1) \cdot\left(\alpha   V^{\pi_0}_{c}(s^1_1)+\Delta_0\right)}. \label{eq:lower-3}
\end{align}
where, $(a)$ follows for $\Delta_0 \geq 1/\sqrt{n}$. 
Combining \cref{eq:lower-1}, \ref{eq:lower-2}, and \ref{eq:lower-3} we can show that
\begin{align*}
\E\left[R^{\mathfrak{A}}_n(\al)\right] \gtrsim \max \left\{\dfrac{\xi}{n^{3/2}}, \frac{(\alpha V^{\pi_0}_{c}(s^1_1)+\Delta_0)^2\xi^2}{(\alpha V^{\pi_0}_{c}(s^1_1))^2 n^{3/2}}\right\}.
\end{align*}

\textbf{Step 3 (Combine with MAB:)} Now considering that safe oracle $\bb^k_*$ is also an online algorithm $\al$, we can drop the notation. Then for multi-armed bandits, by \Cref{lemma:lower-bound-bandit}, we choose $\xi=A^{1/3}$. Then we have
\begin{align*}
\E\left[\cR_n\right] \gtrsim \max \left\{\dfrac{A^{1/3}}{n^{3/2}}, \frac{\left(\alpha V^{\pi_0}_{c}(s^1_1)+\Delta_0\right)^2}{(\alpha V^{\pi_0}_{c}(s^1_1))^2 }\left(\dfrac{A^{2/3}}{n^{3/2}}\right)\right\} \overset{(a)}{=} \min \left\{\dfrac{A^{1/3}}{n^{3/2}}, \left(\dfrac{H_{*, (1)}^2 A^{2/3}}{n^{3/2}}\right)\right\} .
\end{align*}
where, $(a)$ follows from the problem complexity parameter $H_{*, (1)} = \frac{1}{\alpha V^{\pi_0}_{c}(s^1_1)}(\alpha V^{\pi_0}_{c}(s^1_1)+\Delta_0)$ when $\pi(a) = 1/A$ and $\sigma(a)\geq 1/16$ for the bandit setting.
%
%

\textbf{Step 4 (Combine with tabular RL:)} For tabular RL, by \Cref{lemma:tab-RL-lower-bound}, we choose $\xi=\sqrt{SAL^2}$. Then we have
\begin{align*}
\E\left[\cR_n\right] \gtrsim \max \left\{\dfrac{\sqrt{SAL^2}}{n^{3/2}}, \frac{\left(\alpha V^{\pi_0}_{c}(s^1_1)+\Delta_0\right)^2}{(\alpha V^{\pi_0}_{c}(s^1_1))^2 }\left(\dfrac{SAL^2}{n^{3/2}}\right)\right\} \overset{(a)}{=} \min \left\{\dfrac{\sqrt{SAL^2}}{n^{3/2}}, \left(\dfrac{H_{*, (1)}^2 SAL^2}{n^{3/2}}\right)\right\} .
\end{align*}
where, $(a)$ follows from the problem complexity parameter $H_{*, (1)} = \frac{1}{\alpha V^{\pi_0}_{c}(s^1_1)}(\alpha V^{\pi_0}_{c}(s^1_1)+\Delta_0)$ when $\pi(a) = 1/A$ and $\sigma(a)\geq 1/16$.
%
%
This concludes the proof.
\end{proof}

\begin{remark}\textbf{(Comparing regret)}
\label{rem:comparing-regret}
Observe that the regret lower bound is proved on $\cR'_n = \L_n(\pi) - \L^*_n(\pi)$ which assumes that we can exactly solve for the oracle sampling solution. However, $\overline{\L}^*_n(\pi)$ in $\cR_n$ is an upper bound to $\L^*_n(\pi)$ and so we cannot directly compare $\cR_n$ with $\cR'_n$. However, since $\cR'_n$ gives a lower bound by directly solving for the oracle solution, we conjecture that this is the lower bound to $\cR_n$. Proving this conjecture we leave it to future works.
\end{remark}


\section{Proof of Tree Agnostic MSE}
\label{app:mdp-saver-loss}
\begin{customtheorem}{2}\textbf{(formal)}
Let \Cref{assm:tractable-MDP} hold. Then the MSE of the \sav\ for $\frac{n}{\log(SAn(n+1)/\delta)} \geq 32(LSA^2)^2 + \frac{SA}{\min_{s,a}\Delta^{c,(2)}(s,a)} + \frac{1}{4 H_{*, (2)}^2}$ is bounded by
    \begin{align*}
        \L_n(\pi, \wb^k) \leq &\frac{M^2(s^1_1)}{n}+\frac{8A M^2(s^1_1)}{n^2} + \frac{16A^2 M^2(s^1_1)}{n^3} + \dfrac{M^2(s^1_1)}{n}\left(32 M LSA + H_{*, (2)}\right)^2  + 2\sum_{t=1}^n\dfrac{2\eta + 4\eta^2}{n^2} \nonumber\\
        & \quad + O\left(\dfrac{(2\eta + 4\eta^2)(LSA^2)^2 H_{*, (2)}^2 M^2\sqrt{\log (SA n(n+1) / \delta)}}{\min_{s}\bb^{k,(3/2)}_{*,\min}(s)n^{3/2}}\right)
    \end{align*}
    with probability $(1-\delta)$. The $M =\sum_{\ell=1}^L\sum_{s^\ell_j}M(s^\ell_j)$, and $H_{*, (2)} =\sum_{\ell=1}^L\sum_{s^\ell_j}H_{*, (2)}(s^\ell_j)$ is the problem complexity parameter. The total predicted constraint violations is bounded by
    \begin{align*}
        \C_n(\pi, \wb^k) \leq \dfrac{H_{*, (2)}}{2} \dfrac{n}{M_{\min}} + 16  LSA^2 +  O\left(\dfrac{(2\eta + 4\eta^2)(LSA^2)^2 H_{*, (2)}^2 M^2\sqrt{\log (SA n(n+1) / \delta)}}{\min_{s}\bb^{k,(3/2)}_{*,\min}(s)n^{1/2}}\right)
    \end{align*}
    with probability $(1-\delta)$, where $M_{\min} \coloneqq \min_s M(s)$.
\end{customtheorem}

\begin{proof}
\textbf{Step 1 (Sampling rule):} First note that agnostic \sav\ 
samples by the following rule
\begin{align}
\quad \text{Play } \bb^k = \begin{cases} \pi_x & \text{ if }             \wZ^{k-1} \geq 0, k \leq \sqrt{K}\\
        \wb^k & \text{ if } \wZ^{k-1} \geq 0, k > \sqrt{K}\\
        \pi_0 &\text{ if } \wZ^{k-1} < 0 \label{eq:mdp-saver-sample-1}
\end{cases} 
\end{align}
where, $\wZ^{k-1}_L \coloneqq \sum_{k'=1}^{k-1}(Y_{c,L}^{\bb^{k'}}(s^1_1) - \beta^{k'}_L(s,a)) - (1-\alpha)(k-1)V_{c}^{\pi_0}(s^1_1)$ is the safety budget till the $k$-th episode.


\textbf{Step 2 (MSE Decomposition):} Now recall that the agnostic algorithm does not know the variances and the means. We define the good cost event when the oracle has a good estimate of the cost mean. This is stated as follows:
\begin{align}
    \xi_{c,K} \coloneqq \bigcap_{\substack{1\leq k \leq K,\\1\leq a \leq A,1\leq s \leq S}}\left\{\left|\wmu^{k}_{c,L}(s,a)-\mu^c(s,a)\right|\leq (2\eta + 4\eta^2)L \sqrt{\frac{\log (SA n(n+1) / \delta)}{2 T^k_L(s,a)}}\right\}  \label{eq:mdp-good-constraint-event-agnostic}
\end{align}
where, $n=KL$ and $K$ is the number of episodes and $L$ is the length of horizon of each episode. 
%
%
%
The exploration policy $\pi_x$ results in a good constraint estimate of state-action tuples. This is shown in \Cref{corollary:safe-conc}. 
%
We define the good variance event as 
\begin{align}
    \xi_{v,K} \coloneqq \bigcap_{\substack{1\leq k \leq K,\\1\leq a \leq A,1\leq s \leq S}}\left\{|\wsigma^{k}_L(s,a) - \sigma(s,a)|\leq (2\eta + 4\eta^2)L \sqrt{\frac{\log (SA n(n+1) / \delta)}{2 T^k_L(s,a)}}\right\}. \label{eq:mdp-good variance-event}
\end{align}
We define the safety budget event
\begin{align}
    \xi_{Z,K} \coloneqq \bigcap_{1\leq k \leq K}\left\{\wZ^k \geq 0\right\}. \label{eq:safety-budget-event-agnostic}
\end{align}
Using the definition of MSE, 
and \Cref{lemma:wald-variance} we can show that 
\begin{align}
    &\E_{\D}\left[\left(Y_{n}(s^1_1)-V_\pi(s^1_1)\right)^2\indic{\xi_{Z,K}}\cap\indic{\xi_{c,K}}\cap\indic{\xi_{v,K}}\right] \nonumber\\
    &\leq \sum_a\pi^2(a|s^{1}_{1}) \bigg[\dfrac{ \sigma^2(s^{1}_{1}, a)}{\underline{T}^{(2),K}_L(s^{1}_{1}, a)}\bigg]\E[T^K_L(s^1_1,a)\indic{\xi_{Z,K}}\cap\indic{\xi_{c,K}}\cap\indic{\xi_{v,K}}]\nonumber\\
    &\quad + \gamma^2\sum_{a}\pi^2(a|s^{1}_{1})\sum_{s^{2}_j}P(s^2_j|s^1_1,a)\Var[Y_{n}(s^{2}_j)]\E[T^K_L(s^2_j,a)\indic{\xi_{Z,K}}\cap\indic{\xi_{c,K}}\cap\indic{\xi_{v,K}}]\nonumber\\
    &\leq \sum_a\pi^2(a|s^{1}_{1}) \bigg[\dfrac{ \sigma^2(s^{1}_{1}, a)}{\underline{T}^{(2),K}_L(s^{1}_{1}, a)}\bigg]\E[T^K_L(s^1_1,a)\indic{\xi_{Z,K}\cap\indic{\xi_{v,K}}}\cap\indic{\xi_{c,K}}] \nonumber\\
    &\quad + \gamma^2\sum_{a}\pi^2(a|s^{1}_{1})\sum_{\ell=2}^L\sum_{s^{\ell}_j}P(s^{\ell}_j|s^1_1,a)\sum_{a'}\pi^2(a'|s^\ell_j)\bigg[\dfrac{ \sigma^2(s^{\ell}_{j}, a')}{\underline{T}^{(2),K}_L(s^{\ell}_{j}, a')}\bigg]\E[T^K_L(s^\ell_j,a')\indic{\xi_{Z,K}}\cap\indic{\xi_{c,K}}\cap\indic{\xi_{v,K}}] \label{eq:MSE-agnostic}
\end{align}
which implies that \sav\ does not need to know the reward means $\mu(s,a)$.
Hence, the MSE of \sav\ is bounded by

\begin{align*}
\L_n(\pi) 
&\leq \underbrace{\E_{\D}\left[\left(Y_{n}(s^1_1)-V_\pi(s^1_1)\right)^2\indic{\xi_{Z,K}}\cap\indic{\xi_{c,K}}\cap\xi_{v,K}\right]}_{\textbf{Part A, $\wZ_n \geq 0$, safety event holds}}
+ \underbrace{\E_{\D}\left[\left(Y_{n}(s^1_1)-V_\pi(s^1_1)\right)^2\indic{\xi^C_{Z,K}}\right]}_{\textbf{Part B, $\wZ_n < 0$, constraint violation}} \\
&\quad + \underbrace{ \E_{\D}\left[\left(Y_{n}(s^1_1)-V_\pi(s^1_1)\right)^2\indic{\xi^C_{c,K}}\right]}_{\textbf{Part C, Safety event does not hold}} + \underbrace{ \E_{\D}\left[\left(Y_{n}(s^1_1)-V_\pi(s^1_1)\right)^2\indic{\xi^C_{v,K}}\right]}_{\textbf{Part D, Variance event does not hold}}\\
& \leq \sum_a\pi^2(a|s^{1}_{1}) \bigg[\dfrac{ \sigma^2(s^{1}_{1}, a)}{\underline{T}^{(2),K}_L(s^{1}_{1}, a)}\bigg]\E[T^K_L(s^1_1,a)\indic{\xi_{Z,K}}\cap\indic{\xi_{c,K}}\cap\indic{\xi_{v,K}}] \\
    &\quad + \gamma^2\sum_{a}\pi^2(a|s^{1}_{1})\sum_{\ell=2}^L\sum_{s^{\ell}_j}P(s^{\ell}_j|s^1_1,a)\sum_{a'}\pi^2(a'|s^\ell_j)\bigg[\dfrac{ \sigma^2(s^{\ell}_{j}, a')}{\underline{T}^{(2),K}_L(s^{\ell}_{j}, a')}\bigg]\E[T^K_L(s^\ell_j,a')\indic{\xi_{Z,K}}\cap\indic{\xi_{c,K}}]\\
    &\quad + \underbrace{\E_{\D}\left[\left(Y_{n}(s^1_1)-V_\pi(s^1_1)\right)^2\indic{\xi^C_{Z,K}}\right]}_{\textbf{Part B, $\wZ_n < 0$, constraint violation}} 
+ \underbrace{ \E_{\D}\left[\left(Y_{n}(s^1_1)-V_\pi(s^1_1)\right)^2\indic{\xi^C_{c,K}}\right]}_{\textbf{Part C, Safety event does not hold}} +\underbrace{ \E_{\D}\left[\left(Y_{n}(s^1_1)-V_\pi(s^1_1)\right)^2\indic{\xi^C_{v,K}}\right]}_{\textbf{Part D, Variance event does not hold}}.
\end{align*}
Divide the total budget $n$ into two parts, $n_{f}$ when $\sum_{j=1}^k\indic{\wZ^j \geq 0}$ is true, then $\bb_{*}$ or $\pi_x$ is run. 
Hence define
\begin{align*}
    n_f \coloneqq \sum_{k=1}^K\sum_{\ell=1}^L\sum_{s^\ell_j}\sum_{a'=1}^A\E[T^k_\ell(s^\ell_j,a')\indic{\xi_{Z,K}}\cap\indic{\xi_{c,K}}\cap\indic{\xi_{v,K}}].
\end{align*}
The other part consist of $n_{u} = n-n_{f}$ number of samples when $\sum_{j=1}^k\indic{\wZ^k < 0}$ and only $\pi_0$ is run. 
Hence we define,
\begin{align*}
    n_u = \sum_{k=1}^K\sum_{\ell=1}^L\sum_{s^\ell_j}\sum_{a'=1}^A\E[T^k_\ell(s^\ell_j,a')\indic{\xi^C_{Z,K}}].
\end{align*}
%

\textbf{Step 3 (Sampling of \sav\ for $\wZ^k \geq 0$):} First note that when $\wZ^k\geq 0$ the \sav\ samples at episode $k$ and round $\ell+1$ the action $\argmax_a U^k_{\ell+1}(s^{\ell+1}_i,a)$ where
\begin{align}
U^k_{\ell}(s^{\ell}_i,a) \coloneqq\dfrac{\wb^k_{\ell}(a|s^{\ell}_{i})}{T^k_{\ell}(s^{\ell}_{i}, a)} 
&\leq   \frac{\pi(a|s^{\ell}_{i})}{T^k_{\ell}(s^{\ell}_{i}, a)}\bigg(\sigma(s^{\ell}_{i}, a) + (2\eta + 4\eta^2) \sqrt{\frac{\log (SA n(n+1) / \delta)}{2 T^k_\ell(s^{\ell}_{i}, a)}} \nonumber\\
&\qquad + \underbrace{\gamma^2\sum_{a'}\pi(a'|s^{\ell}_i)\sum_{s^{\ell+1}_j}P(s^{\ell+1}_j|s^{\ell}_i|a')\wM(s^{\ell+1}_j)}_{\bB(s^{\ell}_{i})}\bigg).\label{eq:mdp-saver-1}
\end{align}
Let $\ell+1>2 SA$ be the time at which a given state-action $(s^\ell_i,p')$ is visited for the last time, i.e., $T^k_{\ell}(p')=T^K_{L}(p')-1$ and $T^k_{\ell+1}(p')=T^K_{L}(p')$. Note that as $n = KL \geq 4 SA$, there is at least one state-action pair $(s^\ell_i,p')$ such that this happens, i.e. such that it is visited after the initialization phase. 
Note that under \Cref{assm:tractable-MDP} it is possible to visit each $(s,a)$ atleast once.
Since the \sav\ chooses to visit $(s^\ell_i,p')$ at time $\ell+1$, we have for any state-action pair $(s^\ell_i,p')$
\begin{align}
U^k_{\ell+1}(s^{\ell+1}_i,p) \leq U^k_{\ell+1}(s^{\ell+1}_i,p') .\label{eq:mdp-saver-2}
\end{align}
From \eqref{eq:mdp-saver-1} and using the fact that $T^k_{\ell}(s^{\ell}_i,p')=T^K_{L}(s^{\ell}_i,p')-1$, we can show that
\begin{align}
U^k_{\ell+1}(s^{\ell+1}_i,p') &\leq \frac{\bb_{*}(p'|s^{\ell+1}_i)}{T^k_{t}(s^{\ell+1}_i, p')}\left((2\eta + 4\eta^2) \sqrt{\frac{\log (SA n(n+1) / \delta)}{2 T^k_{t}(s^{\ell+1}_i, p')-1}} + \bB(s^{\ell+1}_i)\right) \nonumber\\
&=\frac{\bb_{*}(p'|s^{\ell+1}_i)}{T^K_{L}(s^{\ell+1}_i, p')-1}\left((2\eta + 4\eta^2) \sqrt{\frac{\log (SA n(n+1) / \delta)}{2 T^K_{L}(s^{\ell+1}_i, p')-1}} + \bB(s^{\ell+1}_i)\right). \label{eq:mdp-saver-3}
\end{align}
Also note that
\begin{align}
    U^k_{\ell+1}(s^{\ell+1}_i, p) = \frac{\bb_{*}(p|s^{\ell+1}_i)}{T^k_{t}(s^{\ell+1}_i, p)}\left((2\eta + 4\eta^2) \sqrt{\frac{\log (SA n(n+1) / \delta)}{2 T^k_{t}(s^{\ell+1}_i, p)-1}} + B(s^{\ell+1}_i)\right)  \overset{(a)}{\geq} \frac{\bb_{*}(p|s^{\ell+1}_i)}{T^K_{L}(s^{\ell+1}_i,p)}. \label{eq:mdp-saver-4}
\end{align}
where, $(a)$ follows as $T_{t}(p)\leq T^K_{L}(p,s^{\ell+1}_i)$ (i.e., the number of times $p$ has been visited can only increase after time $\ell$). 
Combining \eqref{eq:mdp-saver-2}, \eqref{eq:mdp-saver-3}, \eqref{eq:mdp-saver-4} we can show that for any action $p$:
\begin{align}
    \frac{\bb_{*}(p|s^{\ell+1}_i)}{T^K_{L}(p,s^{\ell+1}_i)} \leq \frac{\bb_{*}(p'|s^{\ell+1}_i)}{T^K_{L}(p',s^{\ell+1}_i)-1} \left((2\eta + 4\eta^2) \sqrt{\frac{\log (SA n(n+1) / \delta)}{2 T^K_{L}(s^{\ell+1}_i, p')-1}} + \bB(s^{\ell+1}_i)\right).\label{eq:mdp-saver-5}
\end{align}
Note that in the above equation, there is no dependency on $\ell$, and thus, the probability that \eqref{eq:mdp-saver-5} holds for any $(s^{\ell+1}_i,p)$ and for any $(s^{\ell+1}_i,p')$ such that action $(s^{\ell+1}_i,p')$ is visited after the initialization phase, i.e., such that $T^K_{L}(s^{\ell+1}_i, p')>2$ depends on the probability of event $\xi_{Z,n}$.

\textbf{Step 4. ((Lower bound on $T^K_{L}(s^\ell_i, p)$ for $\wZ^k \geq 0$):} 
If a state-action tuple $(s^\ell_i, p)$ is less visited compared to its optimal allocation without taking into account the initialization phase, i.e., $T^K_{L}(s^\ell_i, p)-2<\bb(p|s^\ell_i)(n-2 A)$, then from the constraint $\sum_{p'}\left(T^K_{L}(s,p')-2\right)=n-2 SA$ and the definition of the optimal allocation, we deduce that there exist at least another state-action tuple $s^\ell_i, p'$ that is over-visited compared to its optimal allocation without taking into account the initialization phase, i.e., $T^K_{L}(s^\ell_i, p')-2>\bb(s^\ell_i, p')(n-2 A)$. Note that for this action, $T^K_{L}(s^\ell_i, p')-2>\bb_{*}(p'|s^\ell_i)(n-2 SA) \geq 0$, so we know that this specific action is taken at least once after the initialization phase and that it satisfies \eqref{eq:mdp-saver-5}. 
Recall that we have defined $M(s^\ell_i) = \sum_a \pi(a|s^\ell_i)\sigma(s^\ell_i, a)$. Further define $M = \sum_{\ell=1}^L\sum_{s^\ell_i}M(s^\ell_i)$. Using the definition of the optimal allocation $T^{*,K}_{L}(s^\ell_i,p')=n_f \frac{\bb_{*}(p'|s^\ell_i)}{M(s^\ell_i)}$, and the fact that $T^K_{L}(s^\ell_i,p') \geq \bb_{*}(p'|s^\ell_i)(n_f- 2 SA)+2$, \eqref{eq:mdp-saver-5} may be written as for any state-action tuple $(s^\ell_i,p)$
%
%
\begin{align}
\frac{\bb_{*}(p|s^\ell_i)}{T^K_{L}(s^\ell_i,p)} & \leq \frac{\bb_{*}(p'|s^\ell_i)}{T^{*,K}_{L}(p',s^\ell_i)} \frac{n_f}{(n_f-2 SA)}\left((2\eta + 4\eta^2) \sqrt{\frac{\log (SA n(n+1) / \delta)}{2 T^K_{L}(s^{\ell+1}_i, p')-1}} + \bB(s^{\ell+1}_i)\right) \nonumber\\
&\leq \dfrac{M(s^\ell_i)}{n_f} + \dfrac{4 SA M(s^\ell_i)}{n_f^2} + \dfrac{(2\eta + 4\eta^2)\sqrt{\log (SA n(n+1) / \delta)}}{\bb^{3/2}_{*,\min}(s^{\ell}_i)n_f^{3/2}}\label{eq:mdp-saver-6}
\end{align}
because $n_f \geq 4 SA$. 
By rearranging \eqref{eq:mdp-saver-6}, we obtain the lower bound on $T^K_{L}(s^\ell_i, p)$ :
\begin{align}
T^K_{L}(s^\ell_i, p) &\geq \frac{\bb_{*}(p|s^\ell_i)}{\frac{M(s^\ell_i)}{n_f}+\frac{4 SA M(s^\ell_i)}{n_f^2} + \dfrac{(2\eta + 4\eta^2)\sqrt{\log (SA n(n+1) / \delta)}}{\bb^{3/2}_{*,\min}(s^{\ell}_i)n_f^{3/2}} } \nonumber\\
&\overset{(a)}{\geq} T^{*,K}_{L}(s^\ell_i,p)- \frac{(2\eta + 4\eta^2)\bb_*(p|s^{\ell}_i)\sqrt{\log (SA n(n+1) / \delta)}}{M(s^{\ell}_i)\bb^{3/2}_{*,\min}(s^{\ell}_i)n_f^{3/2}} -4 A \bb_*(p|s^\ell_i), \label{eq:mdp-saver-6-1}
\end{align}
where in $(a)$ we use $1 /(1+x) \geq 1-x$ (for $x>-1$ ). Note that the lower bound holds on $\xi_{c,K}$ for any state-action  $(s^\ell_i,p)$.

\textbf{Step 5. (Upper bound on $T^K_{L}(s^\ell_i, p)$ for $\wZ^k \geq 0$):} Now using \eqref{eq:mdp-saver-6-1} and the fact that $n_f$ is given by $\sum_{\ell=1}^L\sum_{s^\ell_j}\sum_{a'=1}^A\E[T^K_L(s^\ell_j,a')\indic{\xi_{Z,K}}\cap\indic{\xi_{c,K}}\cap\indic{\xi_{v,K}}]=n_f$, we obtain
\begin{align*}
T^K_{L}(s^\ell_i, p) &= n_f-\sum_{p' \neq p} T^K_{L}(s^\ell_i,p') \leq\left(n_f -\sum_{p' \neq p} T^{*,K}_{L}(s^\ell_i,p')\right)\\
&\quad +\sum_{p' \neq p}\left(\frac{(2\eta + 4\eta^2)\bb_*(p'|s^{\ell}_i)\sqrt{\log (SA n(n+1) / \delta)}}{M(s^{\ell}_i)\bb^{3/2}_{*,\min}(s^{\ell}_i)n_f^{3/2}} + 4 A \bb_*(p'|s^\ell_i) \right).
\end{align*}
Now since $\sum_{p' \neq p} \bb_*(p'|s^\ell_i) \leq 1$ we can show that
\begin{align}
T^K_{L}(s^\ell_i, p) \leq T^{*,K}_{L}(s^\ell_i, p)+ \frac{(2\eta + 4\eta^2)\bb_*(p|s^{\ell}_i)\sqrt{\log (SA n(n+1) / \delta)}}{M(s^{\ell}_i)\bb^{3/2}_{*,\min}(s^{\ell}_i)n_f^{3/2}} + 4 A . \label{eq:mdp-saver-6-2}
\end{align}

\textbf{Step 6 (Bound part A):} We now bound the part A using \eqref{eq:mdp-saver-6} 
\begin{align*}
&\sum_a\pi^2(a|s^{1}_{1}) \bigg[\dfrac{ \sigma^2(s^{1}_{1}, a)}{\underline{T}^{(2),K}_L(s^{1}_{1}, a)}\bigg]\E[T^K_L(s^1_1,a)\indic{\xi_{Z,K}}\cap\indic{\xi_{c,K}}\cap\indic{\xi_{v,K}}] \\
    &\quad + \gamma^2\sum_{a}\pi^2(a|s^{1}_{1})\sum_{\ell=2}^L\sum_{s^{\ell}_j}P(s^{\ell}_j|s^1_1,a)\sum_{a'}\pi^2(a'|s^\ell_j)\bigg[\dfrac{ \sigma^2(s^{\ell}_{j}, a')}{\underline{T}^{(2),K}_L(s^{\ell}_{j}, a')}\bigg]\E[T^K_L(s^\ell_j,a')\indic{\xi_{Z,K}}\cap\indic{\xi_{c,K}}]\\
    &\overset{(a)}{\leq} \left(\frac{M(s^1_1)}{n_f}+\frac{4 SA M(s^1_1)}{n_f^2} +  \dfrac{(2\eta + 4\eta^2)\sqrt{\log (SA n(n+1) / \delta)}}{\bb^{3/2}_{*,\min}(p|s^{\ell}_i)n_f^{3/2}}\right)^2 n_{f} \\
    &\quad + \gamma^2\sum_{a}\pi^2(a|s^{1}_{1})\sum_{\ell=2}^L\sum_{s^{\ell}_j}P(s^{\ell}_j|s^1_1,a)\left(\frac{M(s^\ell_j)}{n_f}+\frac{4 SA M(s^\ell_j)}{n_f^2} + \dfrac{(2\eta + 4\eta^2)\sqrt{\log (SA n(n+1) / \delta)}}{\bb^{3/2}_{*,\min}(p|s^{\ell}_i)n_f^{3/2}}\right)^2 n_{f}\\
    &\overset{}{=} \frac{M^2(s^1_1)}{n_f}+\frac{8A M^2(s^1_1)}{n_f^2} + \frac{16A^2 M^2(s^1_1)}{n_f^3} + O\left(\dfrac{(2\eta + 4\eta^2)\sqrt{\log (SA n(n+1) / \delta)}}{\bb^{3/2}_{*,\min}(p|s^{\ell}_i)n_f^{3/2}}\right) \\
    &\quad + \gamma^2\sum_{a}\pi^2(a|s^{1}_{1})\sum_{\ell=2}^L\sum_{s^{\ell}_j}P(s^{\ell}_j|s^1_1,a)\left(\frac{M^2(s^\ell_j)}{n_f}+\frac{8A M^2(s^\ell_j)}{n_f^2} + \frac{16A^2 M^2(s^\ell_j)}{n_f^3} \right.\\
    &\quad \left. + O\left(\dfrac{(2\eta + 4\eta^2)\sqrt{\log (SA n(n+1) / \delta)}}{\bb^{3/2}_{*,\min}(p|s^{\ell}_j)n_f^{3/2}}\right)\right)
\end{align*}
where, in $(a)$ follows from the definition of $M(s)$ and $n_{f}$.

\textbf{Step 7 (Upper Bound to Constraint Violation):} In this step we bound the quantity $\C_n(\pi) = \sum_{j=1}^k\indic{\wZ^j < 0, \bb^j\in \{\wb^k, \pi_0\}}$. 
Define the number of times the policy $\bb_*$ is played till episode $k$ is $T^k(\bb_*)$ and the number of times the baseline policy is played is given by $T^k(\pi_0)$.
Observe that $\C_n(\pi) = \sum_{j=1}^k\indic{\wZ^j < 0, \bb^j\in \{\wb^k, \pi_0\}} = T^K(\pi_0)\indic{\xi^C_{Z,K}}$ as when the constraint are violated policy $\pi_0$ is sampled. 
Let $\tau=\max \left\{k \leq K \text{ and } n_f \geq \frac{\log(SAn(n+1)/\delta)}{\min_{s,a}\Delta^{c,\alpha,(2)}(s,a)} \mid \bb^k=\pi_0\right\}$ be the last episode in which the baseline policy is played. 
We will define formally the gap $\Delta^{c,\alpha,(2)}(s,a)$ later. Observe that the constraint violation can be re-stated as follows:
\begin{align}
    &\sum_{k=1}^{\tau} Y_{\bb^k}^c(s^1_1) \coloneqq \sum_{k=1}^{\tau}\sum_{a}\bb^k(a|s^1_1)\left(\wmu^{c,k}_L(s_1,a) + \sum_{s^2_j}P(s^2_j|s^1_1,a) Y_{\bb^k}^c(s^2_j)\right) < (1-\alpha) \tau V^c_{\pi_0}(s^1_1)\nonumber\\
    \implies & \sum_{k=1}^{\tau}\sum_{a}\bb^k(a|s^1_1)\left(\underline{\wmu}^{c,k}_L(s^1_1,a) + \sum_{s^2_j}P(s^2_j|s^1_1,a) \underline{Y}_{\bb^k}^c(s^2_j)\right) < (1-\alpha) \tau V^c_{\pi_0}(s^1_1)\nonumber\\
    \overset{(a)}{\implies} & \sum_{k=1}^{\tau}\sum_{a}\bb^k(a|s^1_1)\left(\underline{\wmu}^{c,k}_L(s^1_1,a) + \sum_{s^2_j}P(s^2_j|s^1_1,a) \underline{Y}_{\bb^k}^c(s^2_j)\right) \nonumber\\
    &\quad < (1-\alpha) \sum_{k=1}^{\tau}\pi_0(0|s^1_1)\left(\mu^{c}_{}(s^1_1,0) + \sum_{s^2_j}P(s^2_j|s^1_1,0) V_{\pi_0}^c(s^2_j)\right)\nonumber
\end{align}
\begin{align}
    \overset{}{\implies} & \sum_{k=1}^{\tau}\sum_{a}T^k_L(s^1_1,a)\left(\underline{\wmu}^{c,k}_L(s^1_1,a) \!\! + \!\! \sum_{s^2_j}P(s^2_j|s^1_1,a) \underline{Y}_{\bb^k}^c(s^2_j)\right) \nonumber\\
    &\quad < (1-\alpha) \sum_{k=1}^{\tau}T^k_{L}(s^1_1,a)\left(\mu^{c}_{}(s^1_1,0) \!\! + \!\! \sum_{s^2_j}P(s^2_j|s^1_1,0) V_{\pi_0}^c(s^2_j)\right)\nonumber\\
    \overset{(b)}{\implies} & \underbrace{\sum_{a}T^\tau_L(s^1_1,a)\underline{\wmu}^{c,\tau}_L(s^1_1,a)}_{\textbf{Part A}} + \sum_{a}T^\tau_L(s^1_1,a)\sum_{s^2_j}P(s^2_j|s^1_1,a) \underline{Y}_{\bb^k}^c(s^2_j)  \nonumber\\
    &\quad < \underbrace{(1-\alpha) \sum_{a}T^\tau_{L}(s^1_1,0)\mu^{c}_{}(s^1_1,0)}_{\textbf{Part B}} + (1-\alpha)T^\tau_{L}(s^1_1,0)\sum_{s^2_j}P(s^2_j|s^1_1,0) V_{\pi_0}^c(s^2_j) .
    \label{eq:constraint-reduction-agnostic}
\end{align}
Comparing \textbf{Part A} and \textbf{Part B} for level $\ell=1$ we observe that the constraint violation must satisfy
\begin{align*}
    & \sum_{a}T^\tau_L(s^1_1,a)\underline{\wmu}^{c,\tau}_L(s^1_1,a)  < (1-\alpha) T^\tau_{L}(s^1_1,0)\mu^{c}_{}(s^1_1,0)
\end{align*}
which can be reduced as follows
\begin{align*}
     T^{\tau-1}_{L}(s^1_1, 0) \leq \dfrac{1}{\alpha\mu^c(s^1_1,0)}\left(1+\sum_{a=1}^A N(s^1_1, a)\right) .
\end{align*}
where $\Delta^{c,\alpha}(s^1_1, a)\coloneqq(1-\alpha) \mu^c(s^1_1,0)-\mu^c(s^1_1, a)$ and
\begin{align}
N(s^1_1, a) &\coloneqq T^{\tau-1}_L(s^1_1, a) \cdot\left((1-\alpha) \mu^c(s^1_1, 0)-\mu^c(s^1_1, a)+c_1\sqrt{\log(An(n+1)/\delta) / T^{\tau-1}_L(s^1_1, a)}\right) \nonumber\\
&=\Delta^{c,\alpha}(s^1_1, a) T^{\tau-1}_L(s^1_1, a)+c_1\sqrt{\log(An(n+1)/\delta) T^{\tau-1}_L(s^1_1,a)}\label{eq:NS-agnostic}
\end{align}
is a bound on the decrease in $\wZ_\tau$ in the first $\tau-1$ rounds due to choosing action $a$ in $s^1_1$. We will now bound $N(s^1_1, a)$ for each $a$. Now observe
\begin{align*}
    \Delta^{c,\alpha}(s^1_1, a) = (1-\alpha) \mu^c(s^1_1, 0)-\mu^c(s^1_1, a) &= \mu^c(s^1_1, 0) - \alpha\mu^c(s^1_1, 0) - \mu^c(s^1_1, a) \\
    &= -(\mu^{*,c}(s^1_1) - \mu^c(s^1_1, 0)) - \alpha\mu^c(s^1_1, 0) + (\mu^{*,c}(s^1_1) - \mu^c(s^1_1, a)) \\
    &= -\Delta^c(s^1_1, 0) - \alpha\mu^c(s^1_1, 0) + \Delta^c(s^1_1, a).
\end{align*}
where, $\mu^{*,c}(s^1_1) = \max_a\mu^c(s^1_1, a)$. 
Let $J(n_f) = \frac{(2\eta + 4\eta^2)\bb_*(p|s^{\ell}_i)\sqrt{\log (SA n(n+1) / \delta)}}{M(s^{\ell}_i)\bb^{3/2}_{*,\min}(s^{\ell}_i)n_f^{3/2}}$. 
The first case is $\Delta^{c,\alpha}(s^1_1, a) > 0$, i.e. $\Delta^c(s^1_1, a) > \Delta^c(0)+\alpha \mu^c(0)$. 
These are the unsafe actions as $\Delta^{c,\alpha}(s^1_1, a) \coloneqq(1-\alpha) \mu^c(0)-\mu^c(s^1_1, a) > 0$ we have from \eqref{eq:mdp-saver-6-2}
\begin{align*}
    T_n(s^1_1, a) &\leq T^*_n(s^1_1, a) + J(n_f) + 4A
    = \dfrac{\pi(s^1_1, a)\sigma(s^1_1, a)}{M}n_f + J(n_f) +  4A
\end{align*}
Plugging this back in $N(s^1_1, a)$ we get
\begin{align}
    &N(s^1_1, a) =\Delta^{c,\alpha}(s^1_1, a) T_{\tau-1}(s^1_1, a)+c_1\sqrt{\log(An(n+1)/\delta) T_{\tau-1}(s^1_1, a)} + J(n_f)\nonumber\\
    &\leq \dfrac{\pi(s^1_1, a)\sigma(s^1_1, a)}{M}n_f\Delta^{c,\alpha}(s^1_1, a) + 4A\Delta^{c,\alpha}(s^1_1, a) + c_1\sqrt{\log(An(n+1)/\delta)\left(\dfrac{\pi(s^1_1, a)\sigma(s^1_1, a)}{M}n_f + 4A\right)} + J(n_f) \nonumber\\
    &\overset{(a)}{\leq} \dfrac{\pi(s^1_1, a)\sigma(s^1_1, a)}{M}n_f\Delta^{c,\alpha}(s^1_1, a) + 4A\Delta^{c,\alpha}(s^1_1, a) + c_1\sqrt{\Delta^{c,\alpha,(2)}(s^1_1, a)\left(\dfrac{\pi(s^1_1, a)\sigma(s^1_1, a)}{M}n_f + 4A\right)} + J(n_f) \nonumber\\
    &\overset{}{\leq} 2\left(\dfrac{\pi(s^1_1, a)\sigma(s^1_1, a)}{M}n_f\Delta^{c,\alpha}(s^1_1, a) + 4A\Delta^{c,\alpha}(s^1_1, a)\right) + J(n_f).
    \label{eq:mdp-saver-9} 
\end{align}
where, $(a)$ follows for $n_f \geq \frac{\log(SAn(n+1)/\delta)}{\min_a\Delta^{c,\alpha,(2)}(s^1_1, a)}$.
The other case is $\Delta^{c,\alpha}(s^1_1, a)< 0$, i.e. $\Delta^c(s^1_1, a)<\Delta^c(s^1_1,0)+\alpha \mu^c(s^1_1,0)$ then only safe actions are pulled. Then
\begin{align}
N(s^1_1, a) &\leq -\Delta^{c,\alpha}(s^1_1, a)T_{\tau-1}(s^1_1, a) + c_1\sqrt{\log(An(n+1)/\delta) T_{\tau-1}(s^1_1, a)} + J(n_f)\nonumber\\
&= \underbrace{-\Delta^{c,\alpha}(s^1_1, a)}_{a}T_{\tau-1}(s^1_1, a) + \underbrace{c_1\sqrt{\log(An(n+1)/\delta)}}_{b}\sqrt{T_{\tau-1}(s^1_1, a)} + J(n_f)\nonumber\\
&\overset{(a)}{\leq} -\dfrac{\log(An(n+1)/\delta)}{4\Delta^{c,\alpha}(s^1_1, a)} = \dfrac{\log(An(n+1)/\delta)}{4(\Delta^c(0) + \alpha\mu^c(0) - \Delta^c(s^1_1, a))}  \nonumber\\
&\overset{(b)}{\leq} 4 \left(\dfrac{\pi(s^1_1, a)\sigma(s^1_1, a)}{M}n_f(\Delta^c(0) + \alpha\mu^c(0) - \Delta^c(s^1_1, a)) \right) 
\label{eq:mdp-saver-10}
\end{align}
where, $(a)$ follows by using $a x^2+b x \leq-b^2 / 4 a$ for $a<0$, and $(b)$ follows as $n_f \geq \frac{\log(An(n+1)/\delta)}{\min^{+}_{a\in\A\setminus\{0\}}\pi(s^1_1, a)\sigma(s^1_1, a)\Delta^{c,\alpha,(2)}(s^1_1, a)}$ which implies
\begin{align*}
&\dfrac{\log(An(n+1)/\delta)}{n_f} \leq 4 \left(\sum_{\A\setminus\{0\}}\pi(s^1_1, a)\sigma(s^1_1, a)\min^{+}\{\Delta^c(s^1_1, a),\Delta^c(0)-\Delta^c(s^1_1, a)\}\right)^2 \\
&\implies \dfrac{\log(An(n+1)/\delta)}{\sum_{\A\setminus\{0\}}\pi(s^1_1, a)\sigma(s^1_1, a)\min^{+}\{\Delta^c(s^1_1, a),\Delta^c(0)-\Delta^c(s^1_1, a)\}} \\
&\leq 4 \left(\sum_{\A\setminus\{0\}}\pi(s^1_1, a)\sigma(s^1_1, a)\min^{+}\{\Delta^c(s^1_1, a),\Delta^c(0)-\Delta^c(s^1_1, a)\}\right)n_f.
\end{align*}
Plugging everything back in \eqref{eq:mdp-saver-10}, we get
\begin{align}
n_u &=T_{\tau-1}(s^1_1,0) = \dfrac{1}{\alpha\mu^c(0)}\left(\sum_{a=1}^A N(s^1_1, a)\right)\nonumber\\
&\leq 
\frac{2}{\alpha \mu^c(0)} \sum_{a\in\uA}
\Delta^c(s^1_1, a)
\left(\dfrac{\pi(s^1_1, a)\sigma(s^1_1, a)}{M}n_f \right) 
+ \frac{4}{\alpha \mu^c(0)}\sum_{a\in\sA\setminus\{0\}}\left(\Delta^c(0)-\Delta^c(s^1_1, a) \right)\left(\dfrac{\pi(s^1_1, a)\sigma(s^1_1, a)}{M}n_f \right)\nonumber\\
&= 
\frac{6}{\alpha \mu^c(0)} \sum_{a\in\A\setminus\{0\}}\min^{+}\{\Delta^c(s^1_1, a),\Delta^c(0)-\Delta^c(s^1_1, a) \}\left(\dfrac{\pi(s^1_1, a)\sigma(s^1_1, a)}{M}(n - n_u) \right)
\overset{}{\leq} \dfrac{H_{*, (2)}}{2}\dfrac{n}{M}.
\label{eq:mdp-saver-11}
\end{align}
It follows then that 
for the state $s^1_1$
\begin{align*}
    n_u(s^1_1) \leq \dfrac{1}{\alpha\mu^c(s^1_1,0)}\left(1+\sum_{a=1}^A N(s^1_1, a)\right) \leq \dfrac{H_{*, (2)}(s^1_1)}{2}\dfrac{n}{M(s^1_1)}
\end{align*}
where 
\begin{align}
    H_{*, (2)}(s^\ell_i) &\coloneqq \sum_{a}\bb_*(a|s^\ell_i)\min^{+}\{\Delta^c(s^\ell_i, a),\Delta^c(s^\ell_i,0)-\Delta^c(s^\ell_i, a)\}, \nonumber\\
     M(s^{\ell}_{i}) &\coloneqq \sum\limits_a\!\! \sqrt{\!\!\pi^2(a|s^{\ell}_{i})\!\left(\!\!\sigma^2(s^{\ell}_{i}, a) \!+\! \!\sum\limits_{s^{\ell+1}_j}\!\!P(s^{\ell+1}_j\!|\!s^{\ell}_i, a) M^2(s^{\ell+1}_j)\!\!\right)}.
    \label{eq:HM-def-agnostic}
\end{align}
For an arbitrary level $\ell\in[L]$, we can show using \eqref{eq:constraint-reduction-agnostic} that the constraint violation must satisfy
\begin{align}
     & \sum_{\ell'=1}^{\ell} \sum_{s^{\ell'}_i} \sum_{a}T^\tau_L(s^{\ell'}_i,a)
     \underline{\wmu}^{c,\tau}_L(s^{\ell'}_i,a)  
      < (1-\alpha) \sum_{\ell'=1}^{\ell} \sum_{s^{\ell'}_i} T^\tau_{L}(s^{\ell'}_i,0)
     \mu^{c}_{0}(s^{\ell'}_i,0) \nonumber\\
     \overset{(a)}{\implies} & \sum_{\ell'=1}^{\ell} \sum_{s^{\ell'}_i} \sum_{a}\left(T^{*,K}_{L}(s^{\ell'}_i,a)-4 A \bb_*(a|s^{\ell'}_i) - O\left(\dfrac{(2\eta + 4\eta^2)\sqrt{\log (SA n(n+1) / \delta)}}{\min_{s^{\ell'}_i}\bb^{k,(3/2)}_{*,\min}(s^{\ell'}_i)n_f^{3/2}}\right)\right)\underline{\wmu}^{c,\tau}_L(s^{\ell'}_i,a)\nonumber\\
     &\quad 
     < (1-\alpha) \sum_{\ell'=1}^{\ell} \sum_{s^{\ell'}_i} \left(T^{*,K}_{L}(s^{\ell'}_i,0) + 4 A + O\left(\dfrac{(2\eta + 4\eta^2)\sqrt{\log (SA n(n+1) / \delta)}}{\min_{s^{\ell'}_i}\bb^{k,(3/2)}_{*,\min}(s^{\ell'}_i)n_f^{3/2}}\right)\right)\mu^{c}_{}(s^{\ell'}_i,0)\nonumber\\
    \implies & \sum_{\ell'=1}^{\ell} \sum_{s^{\ell'}_i} \sum_{a}\left(T^{*,K}_{L}(s^{\ell'}_i,a)\right)\underline{\wmu}^{c,\tau}_L(s^{\ell'}_i,a)
    < (1-\alpha) \sum_{\ell'=1}^{\ell} \sum_{s^{\ell'}_i} \left(T^{*,K}_{L}(s^{\ell'}_i,0)\right)\mu^{c}_{}(s^{\ell'}_i,0) \nonumber\\
    &\quad + 8LS A^2(\mu^{c}_{}(s^{\ell'}_i,0) + \underline{\wmu}^{c,\tau}_L(s^{\ell'}_i,a)) + O\left(\sum_{\ell'=1}^{\ell} \sum_{s^{\ell'}_i}\dfrac{(2\eta + 4\eta^2)\sqrt{\log (SA n(n+1) / \delta)}}{\min_{s^{\ell'}_i}\bb^{k,(3/2)}_{*,\min}(s^{\ell'}_i)n_f^{3/2}}\right)\nonumber\\
    \implies & \sum_{\ell'=1}^{\ell} \sum_{s^{\ell'}_i} \sum_{a}\left(T^{*,K}_{L}(s^{\ell'}_i,a)\right)\underline{\wmu}^{c,\tau}_L(s^{\ell'}_i,a) < (1-\alpha) \sum_{\ell'=1}^{\ell} \sum_{s^{\ell'}_i} \left(T^{*,K}_{L}(s^{\ell'}_i,0)\right)\mu^{c}_{}(s^{\ell'}_i,0)\nonumber\\
    & \quad  + 8LS A^2(\mu^{c,\tau}_{0,L}(s^{\ell'}_i,a) + \wmu^{c,\tau}_L(s^{\ell'}_i,a) - \sqrt{\dfrac{\log((SA n(n+1)/\delta)}{2 T^\tau_L(s^{\ell'}_i,a)})} \nonumber\\
    &\quad + O\left(\sum_{\ell'=1}^{\ell} \sum_{s^{\ell'}_i}\dfrac{(2\eta + 4\eta^2)\sqrt{\log (SA n(n+1) / \delta)}}{\min_{s^{\ell'}_i}\bb^{k,(3/2)}_{*,\min}(s^{\ell'}_i)n_f^{3/2}}\right)\nonumber\\
    \overset{}{\implies} & \sum_{\ell'=1}^{\ell} \sum_{s^{\ell'}_i} \sum_{a}\left(T^{*,K}_{L}(s^{\ell'}_i,a)\right)\underline{\wmu}^{c,\tau}_L(s^{\ell'}_i,a) < (1-\alpha) \max_{s}\mu^{c}_{}(s,0)\sum_{\ell'=1}^{\ell} \sum_{s^{\ell'}_i} \left(T^{*,K}_{L}(s^{\ell'}_i,0)\right) + 16LS A^2 \nonumber\\
    &\quad +O\left(\dfrac{(2\eta + 4\eta^2)L\sqrt{\log (SA n(n+1) / \delta)}}{\min_{s}\bb^{k,(3/2)}_{*,\min}(s)n_f^{3/2}}\right)\label{eq:agnostic-reduction-1}
\end{align}
where, $(a)$ follows as $\mu(s,a)\in (0,1]$ for all $s,a$ and using \eqref{eq:mdp-saver-6-1} and \eqref{eq:mdp-saver-6-2}. Summing over all states $s^\ell_j$ till level $L$ we can show that
\begin{align}
    n_u = \sum_{\ell=1}^L\sum_{s^\ell_j} T^{*,K}_{L}(s^{\ell}_j,0)&\leq \dfrac{n}{2}\sum_{\ell=1}^L\sum_{s^\ell_j}\dfrac{H_{*, (2)}(s^\ell_j)}{M(s^\ell_j)} + 16  LSA^2  +O\left(\dfrac{(2\eta + 4\eta^2)L\sqrt{\log (SA n(n+1) / \delta)}}{\min_{s}\bb^{k,(3/2)}_{*,\min}(s)n_f^{3/2}}\right)n_f \nonumber\\
    &\overset{(a)}{\leq} \dfrac{H_{*, (2)}}{2} \dfrac{n}{M_{\min}} + 16  LSA^2 + O\left(\dfrac{(2\eta + 4\eta^2)L\sqrt{\log (SA n(n+1) / \delta)}}{\min_{s}\bb^{k,(3/2)}_{*,\min}(s)n_f^{1/2}}\right)\nonumber\\
    &\overset{(b)}{\leq} \dfrac{H_{*, (2)}}{2} \dfrac{n}{M_{\min}} + 16  LSA^2 + O\left(\dfrac{(2\eta + 4\eta^2)L\sqrt{\log (SA n(n+1) / \delta)}}{\min_{s}\bb^{k,(3/2)}_{*,\min}(s)n^{1/2}}\right)
    \label{eq:mdp-agnostic-11}  
\end{align}
where, in $(a)$ we define $M_{\min} = \min_{s} M(s)$, and $H_{*, (2)} = \sum_{\ell=1}^L\sum_{s^\ell_j}H_{*, (2)}(s^\ell_j)$, and $(b)$ follows by setting $n_f = n - n_u$. Finally, observe that $16 LSA^2$ does not depend on the episode $K$, and the quantity $O\left(\frac{(2\eta + 4\eta^2)L\sqrt{\log (SA n(n+1) / \delta)}}{\min_{s}\bb^{k,(3/2)}_{*,\min}(s)n^{1/2}}\right)$ decreases with $n$.

\textbf{Step 8 (Lower Bound to Constraint Violation):} 
For the lower bound to the constraint we equate \Cref{eq:constraint-reduction-agnostic} to $0$ and show that
\begin{align*}
    & \underbrace{\sum_{a}T^\tau_L(s^1_1,a)\underline{\wmu}^{c,\tau}_L(s^1_1,a)}_{\textbf{Part A}} + \sum_{a}T^\tau_L(s^1_1,a)\sum_{s^2_j}P(s^2_j|s^1_1,a) \underline{Y}_{\bb^k}^c(s^2_j)  \nonumber\\
    &\quad = \underbrace{(1-\alpha) \sum_{a}T^\tau_{L}(s^1_1,0)\mu^{c}_{}(s^1_1,0)}_{\textbf{Part B}} + (1-\alpha)T^\tau_{L}(s^1_1,0)\sum_{s^2_j}P(s^2_j|s^1_1,0) V_{\pi_0}^c(s^2_j) .
\end{align*}
Again comparing \textbf{Part A} and \textbf{Part B} for level $\ell=1$ we observe that the lower bound to  constraint violation must satisfy
\begin{align*}
    & \sum_{a}T^\tau_L(s^1_1,a)\underline{\wmu}^{c,\tau}_L(s^1_1,a)  = (1-\alpha)  T^\tau_{L}(s^1_1,0)\mu^{c,}_{}(s^1_1,0)
\end{align*}
which can be reduced as
\begin{align*}
    \sum_a T^{\tau-1}_{L}(s^1_1, 0) \geq \dfrac{1}{\alpha\mu^c(s^1_1,0)}\left(1+\sum_{a=1}^A \underline{N}(s^1_1, a)\right) .
\end{align*}
where $\Delta^{c,\alpha}(s^1_1, a)\coloneqq(1-\alpha) \mu^c(s^1_1,0)-\mu^c(s^1_1, a)$ and
\begin{align*}
\underline{N}(s^1_1, a) &\coloneqq T^{\tau-1}_L(s^1_1, a) \cdot\left((1-\alpha) \mu^c(s^1_1, 0)-\mu^c(s^1_1, a)+c_1\sqrt{\log(An(n+1)/\delta) / T^{\tau-1}_L(s^1_1, a)}\right) \\
&=\Delta^{c,\alpha}(s^1_1, a) T^{\tau-1}_L(s^1_1, a)+c_1\sqrt{\log(An(n+1)/\delta) T^{\tau-1}_L(s^1_1,a)}\\
&\overset{(a)}{\geq} \Delta^{c,\alpha}(s^1_1, a) \left(T^{*,K}_L(s^1_1, a) - 4A\bb_*(a|s^1_1)\right)+c_1\sqrt{\log(An(n+1)/\delta) \left(T^{*,K}_L(s^1_1, a) - 4A\bb_*(a|s^1_1)\right)}
\end{align*}
where, $(a)$ follows from \eqref{eq:mdp-saver-6-1}. Then 
we can show that
\begin{align*}
     T^{\tau-1}_{L}(s^1_1, 0) \geq \dfrac{1}{\alpha\mu^c(s^1_1,0)}\left(1+\sum_{a=1}^A \underline{N}(s^1_1, a)\right) \geq \dfrac{n_f}{M(s^1_1)}\left(\dfrac{H_{*, (2)}(s^1_1)}{8} - \frac{A}{2}\dfrac{H_{*, (2)}(s^1_1)}{M(s^1_1)}  \right) - 16 SA
\end{align*}
Similarly for any arbitrary level $\ell\in [L]$  following the same way as step $7$ above it can be shown that
\begin{align*}
    &\sum_{\ell'=1}^{\ell} \sum_{s^{\ell'}_i} \sum_{a}\left(T^{*,K}_{L}(s^{\ell'}_i,a) + 4 A \right)\underline{\wmu}^{c,\tau}_L(s^{\ell'}_i,a)  
     \geq (1-\alpha) \sum_{\ell'=1}^{\ell} \sum_{s^{\ell'}_i} \left(T^{*,K}_{L}(s^{\ell'}_i,0) - 4 A \bb_*(0|s^{\ell'}_i)\right)\mu^{c}_{}(s^{\ell'}_i,0)\\
     \overset{}{\implies} & \sum_{\ell'=1}^{\ell} \sum_{s^{\ell'}_i}\sum_{a} T^{*,K}_L(s^\ell_j,a)\underline{\wmu}^{c,\tau}_L(s^\ell_j,a)  \geq (1-\alpha)\sum_{\ell'=1}^{\ell} \sum_{s^{\ell'}_i} T^{*,K}_{L}(s^\ell_i,0)\mu^{c}_{}(s^\ell_i,0) - 16  LSA^2 .
\end{align*}

For the state $s^\ell_j$ we can show that
\begin{align*}
    \sum_{\ell'=1}^{\ell} &\sum_{s^{\ell'}_j} T^{*,K}_{L}(s^{\ell'}_j, 0) \geq \dfrac{1}{\alpha\max_s\mu^c(s^\ell_j,0)}\left(1+\sum_{\ell'=1}^{\ell} \sum_{s^{\ell'}_j}\sum_{a=1}^A \underline{N}(s^{\ell'}_j, a)\right)\\
    &\geq  \sum_{\ell=1}^L\sum_{s^\ell_j}\dfrac{n_f}{M(s^\ell_j)}\left(\dfrac{H_{*, (2)}(s^\ell_j)}{8} - \frac{A}{2}\dfrac{H_{*, (2)}(s^\ell_j)}{M(s^\ell_j)}  \right) - 16  LSA^2 - O\left(\dfrac{(2\eta + 4\eta^2)L\sqrt{\log (SA n(n+1) / \delta)}}{\min_{s}\bb^{k,(3/2)}_{*,\min}(s)n_f^{3/2}}\right).
\end{align*}
Finally summing over all states $s^\ell_j$ and level $L$ we can show that
\begin{align}
    \sum_{\ell=1}^L\sum_{s^\ell_j} T^{*,K}_{L}(s^\ell_j, 0)& \geq \sum_{\ell=1}^L\sum_{s^\ell_j}\dfrac{n_f}{M(s^\ell_j)}\left(\dfrac{H_{*, (2)}(s^\ell_j)}{8} - \frac{A}{2}\dfrac{H_{*, (2)}(s^\ell_j)}{M(s^\ell_j)}  \right) \nonumber\\
    &\qquad - 16 LSA^2 - O\left(\dfrac{(2\eta + 4\eta^2)L\sqrt{\log (SA n(n+1) / \delta)}}{\min_{s}\bb^{k,(3/2)}_{*,\min}(s)n_f^{3/2}}\right).\label{eq:mdp-agnostic-11-1}
\end{align}
Again, observe that $16 LSA^2$ does not depend on the episode $K$. 

\textbf{Step 9 (Bound Part B)}: Then from \eqref{eq:mdp-agnostic-11-1} we can show that
\begin{align*}
    &\dfrac{M(s^1_1)}{\sum_{\ell=1}^L\sum_{s^\ell_j}  T^{*,K}_{L}(s^\ell_j, 0)} \\
    &\leq \dfrac{M(s^1_1)}{\sum_{\ell=1}^L\sum_{s^\ell_j}\dfrac{n_f}{M(s^\ell_j)}\left(\dfrac{H_{*, (2)}(s^\ell_j)}{8} - \dfrac{A}{2}\dfrac{H_{*, (2)}(s^\ell_j)}{M(s^\ell_j)}  \right) - 16 LSA^2 -  O\left(\dfrac{(2\eta + 4\eta^2)L\sqrt{\log (SA n(n+1) / \delta)}}{\min_{s}\bb^{k,(3/2)}_{*,\min}(s)n_f^{3/2}}\right)}\\
    &\overset{(a)}{\leq} (M(s^1_1) + 16 LSA^2)\sum_{\ell=1}^L\sum_{s^\ell_j}\dfrac{M(s^\ell_j)}{n_f}\left(\dfrac{H_{*, (2)}(s^\ell_j)}{8} + \dfrac{A}{2}\dfrac{H_{*, (2)}(s^\ell_j)}{M(s^\ell_j)}  \right) + O\left(\dfrac{(2\eta + 4\eta^2)L\sqrt{\log (SA n(n+1) / \delta)}}{\min_{s}\bb^{k,(3/2)}_{*,\min}(s)n_f^{3/2}}\right) \\
    &\leq (M(s^1_1) + 16 LSA^2)\sum_{\ell=1}^L\sum_{s^\ell_j}\dfrac{M(s^\ell_j)}{n_f}\left(2 + H_{*, (2)}(s^\ell_j)  \right) + O\left(\dfrac{(2\eta + 4\eta^2)L\sqrt{\log (SA n(n+1) / \delta)}}{\min_{s}\bb^{k,(3/2)}_{*,\min}(s)n_f^{3/2}}\right)  \\
    &\overset{(b)}{\leq} (M(s^1_1) + 16 LSA^2)\dfrac{M_{}}{n_f}\left(2 + H_{*, (2)}  \right) +  O\left(\dfrac{(2\eta + 4\eta^2)L\sqrt{\log (SA n(n+1) / \delta)}}{\min_{s}\bb^{k,(3/2)}_{*,\min}(s)n_f^{3/2}}\right)  
\end{align*}
where, $(a)$ follows for $1/(x-c) \leq x + c$ for $x^2 \geq 1 + c^2$ and $c>0$. The $(b)$ follows for $H_{*, (2)} = \sum_{\ell=1}^L\sum_{s^\ell_j}H_{*, (2)}(s^\ell_j)$ and $M = \sum_{\ell=1}^L\sum_{s^\ell_j} M(s^\ell_j)$. It follows then by setting 
$n_f = n - n_u$ that 
\begin{align*}
    \E_{\D}&\left[\left(Y_{n}(s^1_1)-V_\pi(s^1_1)\right)^2\indic{\xi^C_{Z,K}}\right] \\
    &\overset{(a)}{\leq} \left(\dfrac{ (M(s^1_1) + 16 LSA^2)  M\left(2 + H_{*, (2)} \right)}{n_f} +   O\left(\dfrac{(2\eta + 4\eta^2)L\sqrt{\log (SA n(n+1) / \delta)}}{\min_{s}\bb^{k,(3/2)}_{*,\min}(s)n_f^{3/2}}\right)  \right)^2 n_u \\
    & \overset{(b)}{\leq} \dfrac{(M(s^1_1) + 16 LSA^2)^2  n_u}{(n-n_u)^2}\left(2 + H_{*, (2)}\right)^2 + O\left(\dfrac{(2\eta + 4\eta^2)L^2S^2A^4 H_{*, (2)}^2 M^2\sqrt{\log (SA n(n+1) / \delta)}}{\min_{s}\bb^{k,(3/2)}_{*,\min}(s)(n - n_u)^{3/2}}\right)\\
    &\overset{(c)}{\leq } \dfrac{(M(s^1_1) + 16 LSA^2)^2  H_{*, (2)} n}{(n- H_{*, (2)} n)^2}\left(2 + H_{*, (2)}\right)^2 + O\left(\dfrac{(2\eta + 4\eta^2)L^2S^2A^4 H_{*, (2)}^2 M^2\sqrt{\log (SA n(n+1) / \delta)}}{\min_{s}\bb^{k,(3/2)}_{*,\min}(s)(n- H_{*, (2)}n)^{3/2}}\right)  \\
    &\leq \dfrac{M^2(s^1_1)}{n}\left(32 M LSA^2 + H_{*, (2)}\right)^2 + O\left(\dfrac{(2\eta + 4\eta^2)L^2S^2A^4 H_{*, (2)}^2 M^2\sqrt{\log (SA n(n+1) / \delta)}}{\min_{s}\bb^{k,(3/2)}_{*,\min}(s)n^{3/2}}\right)
\end{align*}
where, $(a)$ follows from \Cref{lemma:wald-variance},  $(b)$ follows from definition of $H_{*, (2)}$, and $(c)$ follows from \eqref{eq:mdp-agnostic-11}.

\textbf{Step 10 (Combine everything):} Combining everything from step 5, step 8 and setting $\delta=1/n^2$ we can show that the MSE of \sav scales as
\begin{align}
    \L_n(\pi, \wb^k) &\leq \frac{M^2(s^1_1)}{n}+\frac{8A M^2(s^1_1)}{n^2} + \frac{16A^2 M^2(s^1_1)}{n^3} + \dfrac{M^2(s^1_1)}{n}\left(32 M LSA + H_{*, (2)}\right)^2 \nonumber\\
    &\quad + \underbrace{ \E_{\D}\left[\left(Y_{n}(s^1_1)-V_\pi(s^1_1)\right)^2\indic{\xi^C_{c,K}}\right]}_{\textbf{Part C, Safety event does not hold}}  + \underbrace{ \E_{\D}\left[\left(Y_{n}(s^1_1)-V_\pi(s^1_1)\right)^2\indic{\xi^C_{v,K}}\right]}_{\textbf{Part D, Variance event does not hold}}\nonumber\\
    &\overset{(a)}{\leq} \frac{M^2(s^1_1)}{n}+\frac{8A M^2(s^1_1)}{n^2} + \frac{16A^2 M^2(s^1_1)}{n^3} + \dfrac{M^2(s^1_1)}{n}\left(32 M LSA + H_{*, (2)}\right)^2  + 2\sum_{t=1}^n\dfrac{2\eta + 4\eta^2}{n^2} \nonumber\\
    & \quad + O\left(\dfrac{(2\eta + 4\eta^2)L^2S^2A^4 H_{*, (2)}^2 M^2\sqrt{\log (SA n(n+1) / \delta)}}{\min_{s}\bb^{k,(3/2)}_{*,\min}(s)n^{3/2}}\right)\label{eq:mdp-agnostic-loss}
\end{align}
where, $(a)$ follows as $\E_{\D}\left[\left(Y_{n}(s^1_1)-V_\pi(s^1_1)\right)^2\indic{\xi^C_{c,K}}\right] \leq 2\eta + 4\eta^2$ and using the low error probability of the cost event from \Cref{lemma:conc-mu} and variance event from \Cref{corollary:conc-var}.
The claim of the theorem follows.

\end{proof}

\section{Proof of Tree Oracle MSE}
\label{app:mdp-oracle-loss}
\begin{customproposition}{2}\textbf{(formal)}
\label{prop:mdp-oracle-loss}
Let \Cref{assm:tractable-MDP} hold. Then the MSE of the oracle for $\frac{n}{\log(SAn(n+1)/\delta)} \geq 32(LSA^2)^2 + \frac{SA}{\min_{s,a}\Delta^{c,(2)}(s,a)} + \frac{1}{4 H_{*, (2)}^2}$ is bounded by
    \begin{align*}
        \L_n(\pi, \bb^k_*) \leq &\frac{M^2(s^1_1)}{n}+\frac{8A M^2(s^1_1)}{n^2} + \frac{16A^2 M^2(s^1_1)}{n^3} + \dfrac{M^2(s^1_1)}{n}\left(32 M LSA^2 + H_{*, (2)}\right)^2  + 2\sum_{t=1}^n\dfrac{2\eta + 4\eta^2}{n^2}
        + \dfrac{2}{n}
    \end{align*}
    with probability $(1-\delta)$. The $M =\sum_{\ell=1}^L\sum_{s^\ell_j}M(s^\ell_j)$, and $H_{*, (2)} =\sum_{\ell=1}^L\sum_{s^\ell_j}H_{*, (2)}(s^\ell_j)$ is the problem complexity parameter. The total predicted constraint violations is bounded by
    \begin{align*}
        \C^*_n(\pi, \bb^k_*) \leq \dfrac{H_{*, (2)}}{2} \dfrac{n}{M_{\min}} + 16  LSA^2 
    \end{align*}
    with probability $(1-\delta)$, where $M_{\min} \coloneqq \min_s M(s)$.
\end{customproposition}

\begin{proof}
\textbf{Step 1 (Sampling rule):} We follow the proof technique of \Cref{thm:mdp-agnostic-loss}. Note that the oracle tree algorithm knows the variances of reward and constraints values (but does not know the mean of either) and samples by the following rule
\begin{align}
    \bb^k_* = \begin{cases}  
    \pi_x, &\text{ if } \wZ^{k-1}_L \geq 0, k \leq \sqrt{K}\\
    \bb_{*}, &\text{ if } \wZ^{k-1}_L \geq 0, k > \sqrt{K}\\
    \pi_0 &\text{ if } \wZ^{k-1}_L < 0 
    \end{cases} \label{eq:mdp-sampling-rule-1}.
\end{align}
where, $\wZ^{k-1}_L \coloneqq \sum_{k'=1}^{k-1}(Y_{c,L}^{\bb^{k'}}(s^1_1) - \beta^{k'}_L(s,a)) - (1-\alpha)(k-1)V_{c}^{\pi_0}(s^1_1)$ is the safety budget till the $k$-th episode. 
%
%

\textbf{Step 2 (MSE Decomposition):} Now recall that the oracle knows the variances but does not know the means (constraint and reward). We define the good constraint event when the oracle has a good estimate of the constraint mean. This is stated as follows:
\begin{align}
    \xi_{c,K} \coloneqq \bigcap_{\substack{1\leq k \leq K,\\1\leq a \leq A,1\leq s \leq S}}\left\{\left|\wmu^{c,k}_L(s,a)-\mu^c(s,a)\right|\leq (2\eta + 4\eta^2) \sqrt{\frac{\log (SA n(n+1) / \delta)}{2 T^k_L(s,a)}}\right\} \label{eq:mdp-good-constraint-event}
\end{align}
where, $n=KL$ and $K$ is the number of episodes and $L$ is the length of horizon of each episode. Define $c_1 = 2\eta + 4\eta^2$. 

The exploration policy $\pi_e$ results in a good constraint estimate of state-action tuples. This is shown in \Cref{corollary:safe-conc}.

We also define the safety budget event
    $\xi_{Z,K} \coloneqq \bigcap_{1\leq k \leq K}\left\{\wZ^k \geq 0\right\}$.
Now using \Cref{lemma:wald-variance} we can show that 
\begin{align*}
    &\E_{\D}\left[\left(Y_{n}(s^1_1)-V_\pi(s^1_1)\right)^2\indic{\xi_{Z,K}}\cap\indic{\xi_{c,K}}\right] \leq \sum_a\pi^2(a|s^{1}_{1}) \bigg[\dfrac{ \sigma^2(s^{1}_{1}, a)}{\underline{T}^{(2),K}_L(s^{1}_{1}, a)}\bigg]\E[T^K_L(s^1_1,a)\indic{\xi_{Z,K}}\cap\indic{\xi_{c,K}}]\\
    &\quad + \gamma^2\sum_{a}\pi^2(a|s^{1}_{1})\sum_{s^{2}_j}P(s^2_j|s^1_1,a)\Var[Y_{n}(s^{2}_j)]\E[T^K_L(s^2_j,a)\indic{\xi_{Z,K}}\cap\indic{\xi_{c,K}}]\\
    &\leq \sum_a\pi^2(a|s^{1}_{1}) \bigg[\dfrac{ \sigma^2(s^{1}_{1}, a)}{\underline{T}^{(2),K}_L(s^{1}_{1}, a)}\bigg]\E[T^K_L(s^1_1,a)\indic{\xi_{Z,K}}\cap\indic{\xi_{c,K}}] \\
    &\quad + \gamma^2\sum_{a}\pi^2(a|s^{1}_{1})\sum_{\ell=2}^L\sum_{s^{\ell}_j}P(s^{\ell}_j|s^1_1,a)\sum_{a'}\pi^2(a'|s^\ell_j)\bigg[\dfrac{ \sigma^2(s^{\ell}_{j}, a')}{\underline{T}^{(2),K}_L(s^{\ell}_{j}, a')}\bigg]\E[T^K_L(s^\ell_j,a')\indic{\xi_{Z,K}}\cap\indic{\xi_{c,K}}]
\end{align*}
which implies that the oracle does not need to know the reward means $\mu(a)$. Hence, Using the definition of MSE we can show that the MSE of oracle is bounded by
\begin{align*}
\L_n(\pi) 
&\leq \underbrace{\E_{\D}\left[\left(Y_{n}(s^1_1)-V_\pi(s^1_1)\right)^2\indic{\xi_{Z,K}}\cap\indic{\xi_{c,K}}\right]}_{\textbf{Part A, $\wZ_n \geq 0$, safety event holds}}
+ \underbrace{\E_{\D}\left[\left(Y_{n}(s^1_1)-V_\pi(s^1_1)\right)^2\indic{\xi^C_{Z,K}}\right]}_{\textbf{Part B, $\wZ_n < 0$, constraint violation}} \\
&\quad + \underbrace{ \E_{\D}\left[\left(Y_{n}(s^1_1)-V_\pi(s^1_1)\right)^2\indic{\xi^C_{c,K}}\right]}_{\textbf{Part C, Safety event does not hold}}\\
& \leq \sum_a\pi^2(a|s^{1}_{1}) \bigg[\dfrac{ \sigma^2(s^{1}_{1}, a)}{\underline{T}^{(2),K}_L(s^{1}_{1}, a)}\bigg]\E[T^K_L(s^1_1,a)\indic{\xi_{Z,K}}\cap\indic{\xi_{c,K}}] \\
    &\quad + \gamma^2\sum_{a}\pi^2(a|s^{1}_{1})\sum_{\ell=2}^L\sum_{s^{\ell}_j}P(s^{\ell}_j|s^1_1,a)\sum_{a'}\pi^2(a'|s^\ell_j)\bigg[\dfrac{ \sigma^2(s^{\ell}_{j}, a')}{\underline{T}^{(2),K}_L(s^{\ell}_{j}, a')}\bigg]\E[T^K_L(s^\ell_j,a')\indic{\xi_{Z,K}}\cap\indic{\xi_{c,K}}]\\
    &\quad + \underbrace{\E_{\D}\left[\left(Y_{n}(s^1_1)-V_\pi(s^1_1)\right)^2\indic{\xi^C_{Z,K}}\right]}_{\textbf{Part B, $\wZ_n < 0$, constraint violation}} 
+ \underbrace{ \E_{\D}\left[\left(Y_{n}(s^1_1)-V_\pi(s^1_1)\right)^2\indic{\xi^C_{c,K}}\right]}_{\textbf{Part C, Safety event does not hold}}
\end{align*}
Divide the total budget $n$ into two parts, $n_{f}$ when $\sum_{j=1}^k\indic{\wZ^j \geq 0}$ is true, then $\bb_{*}$ is run. 
Hence define
\begin{align*}
    n_f \coloneqq \sum_{k=1}^K\sum_{\ell=1}^L\sum_{s^\ell_j}\sum_{a'=1}^A\E[T^k_\ell(s^\ell_j,a')\indic{\xi_{Z,K}}\cap\indic{\xi_{c,K}}].
\end{align*}
The other part consist of $n_{u} = n-n_{f}$ number of samples when $\sum_{j=1}^k\indic{\wZ^k < 0}$ and only $\pi_0$ is run. 
Hence we define,
\begin{align*}
    n_u = \sum_{k=1}^K\sum_{\ell=1}^L\sum_{s^\ell_j}\sum_{a'=1}^A\E[T^k_\ell(s^\ell_j,a')\indic{\xi^C_{Z,K}}].
\end{align*}

\textbf{Step 3 (Sampling of oracle for an episode $k$ when $\wZ^k \geq 0$):} First note that when $\wZ^k\geq 0$ the oracle samples at episode $k$ according to the policy $\bb_{*}$. The following the same steps as in step 3 of \Cref{thm:mdp-agnostic-loss} we can show that. At episode $k$, time $\ell+1$, the $\bb_{*}$ samples the state-action tuple, action $\argmax_a U^k_{\ell+1}(s^{\ell+1}_i,a)$ where
\begin{align}
U^k_{\ell}(s^{\ell}_i,a) &\coloneqq\dfrac{\bb_{*,\ell}(a|s^{\ell}_{i})}{T^k_{\ell}(s^{\ell}_{i}, a)} \label{eq:mdp-oracle-1}
\end{align}
Let $\ell+1>2 SA$ be the time at which a given state-action $(s^\ell_i,p')$ is visited for the last time, i.e., $T^k_{\ell}(p')=T^K_{L}(p')-1$ and $T^k_{\ell+1}(p')=T^K_{L}(p')$. Note that as $n = KL \geq 4 SA$, there is at least one state-action pair $(s^\ell_i,p')$ such that this happens, i.e. such that it is visited after the initialization phase. Since the oracle chooses to pull visit $(s^\ell_i,p')$ at time $\ell+1$, we have for any state-action pair $(s^\ell_i,p')$
\begin{align}
U^k_{\ell+1}(s^{\ell+1}_i,p) \leq U^k_{\ell+1}(s^{\ell+1}_i,p') .\label{eq:mdp-oracle-2}
\end{align}
From \eqref{eq:mdp-oracle-1} and using the fact that $T^k_{\ell}(s^{\ell}_i,p')=T^K_{L}(s^{\ell}_i,p')-1$, we can show that
\begin{align}
U^k_{\ell+1}(s^{\ell+1}_i,p') \leq \frac{\bb_{*}(p'|s^{\ell+1}_i)}{T^k_{t}(s^{\ell+1}_i, p')} =\frac{\bb_{*}(p'|s^{\ell+1}_i)}{T^K_{L}(s^{\ell+1}_i, p')-1} \label{eq:mdp-oracle-3}
\end{align}
Also note that
\begin{align}
    U^k_{\ell+1}(s^{\ell+1}_i, p) = \frac{\bb_{*}(p|s^{\ell+1}_i)}{T^k_{t}(s^{\ell+1}_i, p)} \overset{(a)}{\geq} \frac{\bb_{*}(p|s^{\ell+1}_i)}{T^K_{L}(s^{\ell+1}_i,p)}. \label{eq:mdp-oracle-4}
\end{align}
where, $(a)$ follows as $T_{t}(p)\leq T^K_{L}(p,s^{\ell+1}_i)$ (i.e., the number of times $p$ has been sampled can only increase after time $\ell$). 
Combining \eqref{eq:mdp-oracle-2}, \eqref{eq:mdp-oracle-3}, \eqref{eq:mdp-oracle-4} we can show that for any action $p$:
\begin{align}
    \frac{\bb_{*}(p|s^{\ell+1}_i)}{T^K_{L}(p,s^{\ell+1}_i)} \leq \frac{\bb_{*}(p'|s^{\ell+1}_i)}{T^K_{L}(p',s^{\ell+1}_i)-1} \label{eq:mdp-oracle-5}
\end{align}
Note that in the above equation, there is no dependency on $\ell$, and thus, the probability that \eqref{eq:mdp-oracle-5} holds for any $(s^{\ell+1}_i,p)$ and for any $(s^{\ell+1}_i,p')$ such that state-action $(s^{\ell+1}_i,p')$ is visited after the initialization phase, i.e., such that $T^K_{L}(s^{\ell+1}_i, p')>2$ depends on the probability of event $\xi_{Z,n}$.

\textbf{Step 4. (Lower bound on $T^K_{L}(s^\ell_i, p)$ for $\wZ^k \geq 0$):} If a state-action tuple $s^\ell_i, p,p$ is under-pulled compared to its optimal allocation without taking into account the initialization phase, i.e., $T^K_{L}(s^\ell_i, p)-2<\bb(p|s^\ell_i)(n-2 A)$, then from the constraint $\sum_{p'}\left(T^K_{L}(s,p')-2\right)=n-2 SA$ and the definition of the optimal allocation, we deduce that there exists at least another state-action tuple $s^\ell_i, p'$ that is over-visited compared to its optimal allocation without taking into account the initialization phase, i.e., $T^K_{L}(s^\ell_i, p')-2>\bb(s^\ell_i, p')(n-2 SA)$. Note that for this action, $T^K_{L}(s^\ell_i, p')-2>\bb_{*}(p'|s^\ell_i)(n-2 SA) \geq 0$, so we know that this specific action is pulled at least once after the initialization phase and that it satisfies \eqref{eq:mdp-oracle-5}. 
Recall that we have defined $M(s^\ell_i) = \sum_a \pi(a|s^\ell_i)\sigma(s^\ell_i, a)$. Further define $M = \sum_{\ell=1}^L\sum_{s^\ell_i}M(s^\ell_i)$. Using the definition of the optimal allocation $T^{*,K}_{L}(s^\ell_i,p')=n_f \frac{\bb_{*}(p'|s^\ell_i)}{M(s^\ell_i)}$, and the fact that $T^K_{L}(s^\ell_i,p') \geq \bb_{*}(p'|s^\ell_i)(n_f- 2 SA)+2$, \eqref{eq:mdp-oracle-5} may be written as for any state-action tuple $(s^\ell_i,p)$
\begin{align}
\frac{\bb_{*}(p|s^\ell_i)}{T^K_{L}(s^\ell_i,p)} & \leq \frac{\bb_{*}(p'|s^\ell_i)}{T^{*,K}_{L}(p',s^\ell_i)} \frac{n_f}{(n_f-2 SA)} \leq \dfrac{M(s^\ell_i)}{n_f} + \dfrac{4 A M(s^\ell_i)}{n_f^2} \label{eq:mdp-oracle-6}
\end{align}
because $n_f \geq 4 SA$. 
By rearranging \eqref{eq:mdp-oracle-6}, we obtain the lower bound on $T^K_{L}(s^\ell_i, p)$ :
\begin{align}
T^K_{L}(s^\ell_i, p) \geq \frac{\bb_{*}(p|s^\ell_i)}{\frac{M(s^\ell_i)}{n_f}+\frac{4 A M(s^\ell_i)}{n_f^2}} = \dfrac{\bb_{*}(p|s^\ell_i)}{\frac{M(s^\ell_i)}{n_f}\left(1 + \frac{4A}{n_f}\right)} \overset{(a)}{\geq} T^{*,K}_{L}(s^\ell_i,p)-4 A \bb_*(p|s^\ell_i), \label{eq:mdp-oracle-6-1}
\end{align}
where in $(a)$ we use $1 /(1+x) \geq 1-x$ (for $x>-1$ ). Note that the lower bound holds on $\xi_{c,K}$ for any action $p$.

\textbf{Step 5. (Upper bound on $T^K_{L}(s^\ell_i, p)$ for $\wZ^k \geq 0$):} Now using \eqref{eq:mdp-oracle-6-1} and the fact that $n_f$ is given by $\sum_{\ell=1}^L\sum_{s^\ell_j}\sum_{a'=1}^A\E[T^K_L(s^\ell_j,a')\indic{\xi_{Z,K}}\cap\indic{\xi_{c,K}}]=n_f$, we obtain
\begin{align*}
T^K_{L}(s^\ell_i, p)=n_f-\sum_{p' \neq p} T^K_{L}(s^\ell_i,p') \leq\bigg(n_f -\sum_{p' \neq p} T^{*,K}_{L}(s^\ell_i,p')\bigg)+\sum_{p' \neq p}4 A \bb_*(p'|s^\ell_i) .
\end{align*}
Now since $\sum_{p' \neq p} \bb_*(p'|s^\ell_i) \leq 1$ we can show that
\begin{align}
T^K_{L}(s^\ell_i, p) \leq T^{*,K}_{L}(s^\ell_i, p)+4 A . \label{eq:mdp-oracle-6-2}
\end{align}

\textbf{Step 6 (Bound part A):} We now bound the part A using \eqref{eq:mdp-oracle-6}
\begin{align*}
&\sum_a\pi^2(a|s^{1}_{1}) \bigg[\dfrac{ \sigma^2(s^{1}_{1}, a)}{\underline{T}^{(2),K}_L(s^{1}_{1}, a)}\bigg]\E[T^K_L(s^1_1,a)\indic{\xi_{Z,K}}\cap\indic{\xi_{c,K}}] \\
    &\quad + \gamma^2\sum_{a}\pi^2(a|s^{1}_{1})\sum_{\ell=2}^L\sum_{s^{\ell}_j}P(s^{\ell}_j|s^1_1,a)\sum_{a'}\pi^2(a'|s^\ell_j)\bigg[\dfrac{ \sigma^2(s^{\ell}_{j}, a')}{\underline{T}^{(2),K}_L(s^{\ell}_{j}, a')}\bigg]\E[T^K_L(s^\ell_j,a')\indic{\xi_{Z,K}}\cap\indic{\xi_{c,K}}]\\
    &\overset{(a)}{\leq} \left(\frac{M(s^1_1)}{n_f}+\frac{4 A M(s^1_1)}{n_f^2}\right)^2 n_{f}  + \gamma^2\sum_{a}\pi^2(a|s^{1}_{1})\sum_{\ell=2}^L\sum_{s^{\ell}_j}P(s^{\ell}_j|s^1_1,a)\left(\frac{M(s^\ell_j)}{n_f}+\frac{4 A M(s^\ell_j)}{n_f^2}\right)^2 n_{f}\\
    &\overset{}{=} \frac{M^2(s^1_1)}{n_f}+\frac{8A M^2(s^1_1)}{n_f^2} + \frac{16A^2 M^2(s^1_1)}{n_f^3} \\
    &\quad + \gamma^2\sum_{a}\pi^2(a|s^{1}_{1})\sum_{\ell=2}^L\sum_{s^{\ell}_j}P(s^{\ell}_j|s^1_1,a)\left(\frac{M^2(s^\ell_j)}{n_f}+\frac{8A M^2(s^\ell_j)}{n_f^2} + \frac{16A^2 M^2(s^\ell_j)}{n_f^3}\right)
\end{align*}
where, in $(a)$ follows from the definition of $M(s)$ and $n_{f}$.

\textbf{Step 7 (Upper bound to Constraint violation):} In this step we bound the quantity $\C^*_n(\pi) = \sum_{j=1}^k\indic{\wZ^j < 0, \bb^j\in \{\bb_*, \pi_0\}}$. Define the number of times the policy $\bb_*$ is played till episode $k$ is $T^k(\bb_*)$ and the number of times the baseline policy is played is given by $T^k(\pi_0)$. Observe that $\C^*_n(\pi) = \sum_{j=1}^k\indic{\wZ^j < 0, \bb^j\in \{\bb_*, \pi_0\}} = T^K(\pi_0)\indic{\xi^C_{Z,K}}$ as when the constraint are violated and policy $\pi_0$ is played. 
Let $\tau=\max \left\{k \leq K \text{ and } n_f \geq \frac{\log(SAn(n+1)/\delta)}{\min_{s,a}\bb_*(a|s)\Delta^{c,\alpha,(2)}(s,a)} \mid \bb^k=\pi_0\right\}$ be the last episode in which the baseline policy is played. We will define formally the gap $\Delta^{c,\alpha,(2)}(s,a)$ later. Observe that the constraint violation can be re-stated as follows:
\begin{align}
    &\sum_{k=1}^{\tau} Y_{\bb^k}^c(s^1_1) \coloneqq \sum_{k=1}^{\tau}\sum_{a}\bb^k(a|s^1_1)\left(\wmu^{c,k}_L(s_1,a) + \sum_{s^2_j}P(s^2_j|s^1_1,a) Y_{\bb^k}^c(s^2_j)\right) < (1-\alpha) \tau V^c_{\pi_0}(s^1_1)\nonumber\\
    \implies & \sum_{k=1}^{\tau}\sum_{a}\bb^k(a|s^1_1)\left(\underline{\wmu}^{c,k}_L(s^1_1,a) + \sum_{s^2_j}P(s^2_j|s^1_1,a) \underline{Y}_{\bb^k}^c(s^2_j)\right) < (1-\alpha) \tau V^c_{\pi_0}(s^1_1)\nonumber\\
    \overset{(a)}{\implies} & \sum_{k=1}^{\tau}\sum_{a}\bb^k(a|s^1_1)\left(\underline{\wmu}^{c,k}_L(s^1_1,a) + \sum_{s^2_j}P(s^2_j|s^1_1,a) \underline{Y}_{\bb^k}^c(s^2_j)\right) \nonumber\\
    &\quad < (1-\alpha) \sum_{k=1}^{\tau}\pi_0(0|s^1_1)\left(\mu^{c}_{}(s^1_1,0) + \sum_{s^2_j}P(s^2_j|s^1_1,0) V_{\pi_0}^c(s^2_j)\right)\nonumber\\
    \overset{}{\implies} & \sum_{k=1}^{\tau}\sum_{a}T^k_L(s^1_1,a)\left(\underline{\wmu}^{c,k}_L(s^1_1,a) \!\! + \!\! \sum_{s^2_j}P(s^2_j|s^1_1,a) \underline{Y}_{\bb^k}^c(s^2_j)\right) \nonumber\\
    &\quad < (1-\alpha) \sum_{k=1}^{\tau}T^k_{L}(s^1_1,a)\left(\mu^{c}_{}(s^1_1,0) \!\! + \!\! \sum_{s^2_j}P(s^2_j|s^1_1,0) V_{\pi_0}^c(s^2_j)\right)\nonumber\\
    \overset{(b)}{\implies} & \underbrace{\sum_{a}T^\tau_L(s^1_1,a)\underline{\wmu}^{c,\tau}_L(s^1_1,a)}_{\textbf{Part A}} + \sum_{a}T^\tau_L(s^1_1,a)\sum_{s^2_j}P(s^2_j|s^1_1,a) \underline{Y}_{\bb^k}^c(s^2_j)  \nonumber\\
    &\quad < \underbrace{(1-\alpha) \sum_{a}T^\tau_{L}(s^1_1,0)\mu^{c}_{}(s^1_1,0)}_{\textbf{Part B}} + (1-\alpha)T^\tau_{L}(s^1_1,0)\sum_{s^2_j}P(s^2_j|s^1_1,0) V_{\pi_0}^c(s^2_j) 
    \label{eq:constraint-reduction}
\end{align}
%
%
%
%
%
where $(a)$ follows as $\pi_0$ samples baseline action $0$ for each state $s\in[S]$, and in $(b)$ the $T^\tau_{L}(s^1_1,a)$ denotes the total samples of state-action tuple till episode $\tau$. 
%
%
%
%
Comparing \textbf{Part A} and \textbf{Part B} for level $\ell=1$ we observe that the constraint violation must satisfy
\begin{align*}
    & \sum_{a}T^\tau_L(s^1_1,a)\underline{\wmu}^{c,\tau}_L(s^1_1,a)  < (1-\alpha) T^\tau_{L}(s^1_1,0)\mu^{c}_{}(s^1_1,0)
\end{align*}
which can be reduced by following the same way as step $7$ as \Cref{thm:mdp-agnostic-loss} 
\begin{align*}
     T^{\tau-1}_{L}(s^1_1, 0) \leq \dfrac{1}{\alpha\mu^c(s^1_1,0)}\left(1+\sum_{a=1}^A N(s^1_1, a)\right) .
\end{align*}
where $\Delta^{c,\alpha}(s^1_1, a)\coloneqq(1-\alpha) \mu^c(s^1_1,0)-\mu^c(s^1_1, a)$ and
\begin{align}
N(s^1_1, a) &\coloneqq T^{\tau-1}_L(s^1_1, a) \cdot\left((1-\alpha) \mu^c(s^1_1, 0)-\mu^c(s^1_1, a)+c_1\sqrt{\log(An(n+1)/\delta) / T^{\tau-1}_L(s^1_1, a)}\right) \nonumber\\
&=\Delta^{c,\alpha}(s^1_1, a) T^{\tau-1}_L(s^1_1, a)+c_1\sqrt{\log(An(n+1)/\delta) T^{\tau-1}_L(s^1_1,a)}\label{eq:NS}
\end{align}
is a bound on the decrease in $\wZ_\tau$ in the first $\tau-1$ rounds due to choosing action $a$ in $s^1_1$. We will now bound $N(s^1_1, a)$ for each $a$. Now observe
\begin{align*}
    \Delta^{c,\alpha}(s^1_1, a) = (1-\alpha) \mu^c(s^1_1, 0)-\mu^c(s^1_1, a) &= \mu^c(s^1_1, 0) - \alpha\mu^c(s^1_1, 0) - \mu^c(s^1_1, a) \\
    &= -(\mu^{*,c}(s^1_1) - \mu^c(s^1_1, 0)) - \alpha\mu^c(s^1_1, 0) + (\mu^{*,c}(s^1_1) - \mu^c(s^1_1, a)) \\
    &= -\Delta^c(s^1_1, 0) - \alpha\mu^c(s^1_1, 0) + \Delta^c(s^1_1, a).
\end{align*}
where, $\mu^{*,c}(s^1_1) = \max_a\mu^c(s^1_1, a)$. It follows then that using step $7$ as \Cref{thm:mdp-agnostic-loss} for the state $s^1_1$
\begin{align*}
    n_u(s^1_1) \leq \dfrac{1}{\alpha\mu^c(s^1_1,0)}\left(1+\sum_{a=1}^A N(s^1_1, a)\right) \leq \dfrac{H_{*, (2)}(s^1_1)}{2}\dfrac{n}{M(s^1_1)}
\end{align*}
where 
\begin{align}
    H_{*, (2)}(s^\ell_i) &\coloneqq \sum_{a}\bb_*(a|s^\ell_i)\min^{+}\{\Delta^c(s^\ell_i, a),\Delta^c(s^\ell_i,0)-\Delta^c(s^\ell_i, a)\}, \nonumber\\
     M(s^{\ell}_{i}) &\coloneqq \sum\limits_a\!\! \sqrt{\!\!\pi^2(a|s^{\ell}_{i})\!\left(\!\!\sigma^2(s^{\ell}_{i}, a) \!+\! \!\sum\limits_{s^{\ell+1}_j}\!\!P(s^{\ell+1}_j\!|\!s^{\ell}_i, a) M^2(s^{\ell+1}_j)\!\!\right)}
    \label{eq:HM-def}
\end{align}
Similarly, for an arbitrary level $\ell\in[L]$, we can show using \eqref{eq:constraint-reduction} that the constraint violation must satisfy
\begin{align}
     & \sum_{\ell'=1}^{\ell} \sum_{s^{\ell'}_i} \sum_{a}T^\tau_L(s^{\ell'}_i,a)
     \underline{\wmu}^{c,\tau}_L(s^{\ell'}_i,a)  
      < (1-\alpha) \sum_{\ell'=1}^{\ell} \sum_{s^{\ell'}_i} T^\tau_{L}(s^{\ell'}_i,0)
     \mu^{c}_{}(s^{\ell'}_i,0) \nonumber\\
     \overset{(a)}{\implies} & \sum_{\ell'=1}^{\ell} \sum_{s^{\ell'}_i} \sum_{a}\left(T^{*,K}_{L}(s^{\ell'}_i,a)-4 A \bb_*(a|s^{\ell'}_i)\right)\underline{\wmu}^{c,\tau}_L(s^{\ell'}_i,a) 
    < (1-\alpha) \sum_{\ell'=1}^{\ell} \sum_{s^{\ell'}_i} \left(T^{*,K}_{L}(s^{\ell'}_i,0) + 4 A \right)\mu^{c}_{}(s^{\ell'}_i,0)\nonumber\\
    \implies & \sum_{\ell'=1}^{\ell} \sum_{s^{\ell'}_i} \sum_{a}\left(T^{*,K}_{L}(s^{\ell'}_i,a)\right)\underline{\wmu}^{c,\tau}_L(s^{\ell'}_i,a) \nonumber\\
    &< (1-\alpha) \sum_{\ell'=1}^{\ell} \sum_{s^{\ell'}_i} \left(T^{*,K}_{L}(s^{\ell'}_i,0)\right)\mu^{c}_{}(s^{\ell'}_i,0) + 8LS A^2(\mu^{c}_{}(s^{\ell'}_i,0) + \underline{\wmu}^{c,\tau}_L(s^{\ell'}_i,a))\nonumber\\
    \implies & \sum_{\ell'=1}^{\ell} \sum_{s^{\ell'}_i} \sum_{a}\left(T^{*,K}_{L}(s^{\ell'}_i,a)\right)\underline{\wmu}^{c,\tau}_L(s^{\ell'}_i,a)\nonumber\\
    & \quad < (1-\alpha) \sum_{\ell'=1}^{\ell} \sum_{s^{\ell'}_i} \left(T^{*,K}_{L}(s^{\ell'}_i,0)\right)\mu^{c}_{}(s^{\ell'}_i,0) + 8LS A^2(\mu^{c}_{}(s^{\ell'}_i,0) + \wmu^{c,\tau}_L(s^{\ell'}_i,a) - \sqrt{\dfrac{\log((SA n(n+1)/\delta)}{2 T^\tau_L(s^{\ell'}_i,a)})}\nonumber\\
    \overset{(b)}{\implies} & \sum_{\ell'=1}^{\ell} \sum_{s^{\ell'}_i} \sum_{a}\left(T^{*,K}_{L}(s^{\ell'}_i,a)\right)\underline{\wmu}^{c,\tau}_L(s^{\ell'}_i,a) < (1-\alpha) \max_{s,a}\mu^{c}_{0}(s,a)\sum_{\ell'=1}^{\ell} \sum_{s^{\ell'}_i} \left(T^{*,K}_{L}(s^{\ell'}_i,0)\right) + 16LS A^2 \label{eq:oracle-reduction-1}
\end{align}
where, $(a)$ follows from \eqref{eq:mdp-oracle-6-2} and $(b)$ follows as $\mu(s,a)\in (0,1]$ for all $s,a$. 
It follows then that using step $7$ of \Cref{thm:mdp-agnostic-loss} and definition of $N(s^\ell_j)$ from \eqref{eq:NS}
\begin{align*}
    \sum_{\ell'=1}^{\ell} \sum_{s^{\ell'}_i} T^{*,K}_{L}(s^{\ell'}_i,0) \leq \dfrac{1}{\alpha\max_{s}\mu^c(s,0)}\left(1+\sum_{\ell'=1}^{\ell} \sum_{s^{\ell'}_i}\sum_{a} N(s^\ell_j, a)\right) \leq \dfrac{n}{2}\sum_{\ell'=1}^{\ell} \sum_{s^{\ell'}_i}\sum_{a}\dfrac{H_{*, (2)}(s^{\ell'}_i)}{M(s^{\ell'}_i)} + 16  LSA^2 
\end{align*}
which gives a bound on how many times action $\{0\}$ is sampled across different states till level $\ell$. Summing over all states $s^\ell_j$ till level $L$ we can show that
\begin{align}
    n_u = \sum_{\ell=1}^L\sum_{s^\ell_j}  T^{*,K}_{L}(s^{\ell}_j,0)\leq \dfrac{n}{2}\sum_{\ell=1}^L\sum_{s^\ell_j}\dfrac{H_{*, (2)}(s^\ell_j)}{M(s^\ell_j)} + 16  LSA^2 \overset{(a)}{\leq} \dfrac{H_{*, (2)}}{2} \dfrac{n}{M_{\min}} + 16  LSA^2  \label{eq:mdp-oracle-11}  
\end{align}
where, in $(a)$ we define $M_{\min} = \min_{s} M(s)$, and $H_{*, (2)} =\sum_{\ell=1}^L\sum_{s^\ell_j}H_{*, (2)}(s^\ell_j)$. Finally, observe that $16 LSA^2$ does not depend on the episode $K$. 

\textbf{Step 8 (Lower bound to Constraint violation):} For the lower bound to the constraint we equate \Cref{eq:constraint-reduction} to $0$ and show that
\begin{align*}
    & \underbrace{\sum_{a}T^\tau_L(s^1_1,a)\underline{\wmu}^{c,\tau}_L(s^1_1,a)}_{\textbf{Part A}} + \sum_{a}T^\tau_L(s^1_1,a)\sum_{s^2_j}P(s^2_j|s^1_1,a) \underline{Y}_{\bb^k}^c(s^2_j)  \nonumber\\
    &\quad = \underbrace{(1-\alpha) \sum_{a}T^\tau_{L}(s^1_1,0)\mu^{c}_{}(s^1_1,0)}_{\textbf{Part B}} + (1-\alpha)T^\tau_{L}(s^1_1,0)\sum_{s^2_j}P(s^2_j|s^1_1,0) V_{\pi_0}^c(s^2_j) 
\end{align*}
Again comparing \textbf{Part A} and \textbf{Part B} for level $\ell=1$ we observe that the lower bound to  constraint violation must satisfy
\begin{align*}
    & \sum_{a}T^\tau_L(s^1_1,a)\underline{\wmu}^{c,\tau}_L(s^1_1,a)  = (1-\alpha)  T^\tau_{L}(s^1_1,0)\mu^{c,}_{}(s^1_1,0)
\end{align*}
which can be reduced by following the same way as step $8$ as \Cref{thm:mdp-agnostic-loss} 
\begin{align*}
    \sum_a T^{\tau-1}_{L}(s^1_1, 0) \geq \dfrac{1}{\alpha\mu^c(s^1_1,0)}\left(1+\sum_{a=1}^A \underline{N}(s^1_1, a)\right) .
\end{align*}
where $\Delta^{c,\alpha}(s^1_1, a)\coloneqq(1-\alpha) \mu^c(s^1_1,0)-\mu^c(s^1_1, a)$ and
\begin{align*}
\underline{N}(s^1_1, a) &\coloneqq T^{\tau-1}_L(s^1_1, a) \cdot\left((1-\alpha) \mu^c(s^1_1, 0)-\mu^c(s^1_1, a)+c_1\sqrt{\log(An(n+1)/\delta) / T^{\tau-1}_L(s^1_1, a)}\right) \\
&=\Delta^{c,\alpha}(s^1_1, a) T^{\tau-1}_L(s^1_1, a)+c_1\sqrt{\log(An(n+1)/\delta) T^{\tau-1}_L(s^1_1,a)}\\
&\overset{(a)}{\geq} \Delta^{c,\alpha}(s^1_1, a) \left(T^{*,K}_L(s^1_1, a) - 4A\bb_*(a|s^1_1)\right)+c_1\sqrt{\log(An(n+1)/\delta) \left(T^{*,K}_L(s^1_1, a) - 4A\bb_*(a|s^1_1)\right)}
\end{align*}
where, $(a)$ follows from \eqref{eq:mdp-oracle-6-1}. Then following the same way as step $8$ of \Cref{thm:mdp-agnostic-loss} we can show that
\begin{align*}
     T^{\tau-1}_{L}(s^1_1, 0) \geq \dfrac{1}{\alpha\mu^c(s^1_1,0)}\left(1+\sum_{a=1}^A \underline{N}(s^1_1, a)\right) \geq \dfrac{n_f}{M(s^1_1)}\left(\dfrac{H_{*, (2)}(s^1_1)}{8} - \frac{A}{2}\dfrac{H_{*, (2)}(s^1_1)}{M(s^1_1)}  \right) - 16 SA
\end{align*}
Similarly for any arbitrary level $\ell\in [L]$  following the same way as step $7$ above it can be shown that
\begin{align*}
    &\sum_{\ell'=1}^{\ell} \sum_{s^{\ell'}_i} \sum_{a}\left(T^{*,K}_{L}(s^{\ell'}_i,a) + 4 A \right)\underline{\wmu}^{c,\tau}_L(s^{\ell'}_i,a)  
     \geq (1-\alpha) \sum_{\ell'=1}^{\ell} \sum_{s^{\ell'}_i} \left(T^{*,K}_{L}(s^{\ell'}_i,0) - 4 A \bb_*(0|s^{\ell'}_i)\right)\mu^{c}_{}(s^{\ell'}_i,0)\\
     \overset{}{\implies} & \sum_{\ell'=1}^{\ell} \sum_{s^{\ell'}_i}\sum_{a} T^{*,K}_L(s^\ell_j,a)\underline{\wmu}^{c,\tau}_L(s^\ell_j,a)  \geq (1-\alpha)\sum_{\ell'=1}^{\ell} \sum_{s^{\ell'}_i} T^{*,K}_{L}(s^\ell_i,0)\mu^{c}_{}(s^\ell_i,0) - 16  LSA^2 
\end{align*}
Again following the same way as step $8$ of \Cref{thm:mdp-agnostic-loss} for the state $s^\ell_j$, the lower bound to the total number of times the baseline actions are sampled across states till level $\ell$ is given by we can show that
\begin{align*}
    \sum_{\ell'=1}^{\ell} \sum_{s^{\ell'}_j} T^{{*,K}}_{L}(s^{\ell'}_j, 0) &\geq \dfrac{1}{\alpha\max_{s^\ell_j}\mu^c(s^\ell_j,0)}\left(1+\sum_{\ell'=1}^{\ell} \sum_{s^{\ell'}_j}\sum_{a=1}^A \underline{N}(s^{\ell'}_j, a)\right) \\
    &\geq  \sum_{\ell=1}^L\sum_{s^\ell_j}\dfrac{n_f}{M(s^\ell_j)}\left(\dfrac{H_{*, (2)}(s^\ell_j)}{8} - \frac{A}{2}\dfrac{H_{*, (2)}(s^\ell_j)}{M(s^\ell_j)}  \right) - 16  LSA^2 
\end{align*}
Finally summing over all states $s^\ell_j$ and level $L$ we can show that 
\begin{align}
    \sum_{\ell=1}^L\sum_{s^\ell_j} T^{{*,K}}_{L}(s^\ell_j, 0)\geq \sum_{\ell=1}^L\sum_{s^\ell_j}\dfrac{n_f}{M(s^\ell_j)}\left(\dfrac{H_{*, (2)}(s^\ell_j)}{8} - \frac{A}{2}\dfrac{H_{*, (2)}(s^\ell_j)}{M(s^\ell_j)}  \right) - 16 LSA^2 \label{eq:mdp-oracle-11-1}
\end{align}
Again, observe that $16 LSA^2$ does not depend on the episode $K$. 

\textbf{Step 9 (Bound Part B)}: Then from \eqref{eq:mdp-oracle-11-1} we can show that
\begin{align*}
    \dfrac{M(s^1_1)}{\sum_{\ell=1}^L\sum_{s^\ell_j}  T^{{*,K}}_{L}(s^\ell_j, 0)} &\leq \dfrac{M(s^1_1)}{\sum_{\ell=1}^L\sum_{s^\ell_j}\dfrac{n_f}{M(s^\ell_j)}\left(\dfrac{H_{*, (2)}(s^\ell_j)}{8} - \dfrac{A}{2}\dfrac{H_{*, (2)}(s^\ell_j)}{M(s^\ell_j)}  \right) - 16 LSA^2 }\\
    &\overset{(a)}{\leq} (M(s^1_1) + 16 LSA^2)\sum_{\ell=1}^L\sum_{s^\ell_j}\dfrac{M(s^\ell_j)}{n_f}\left(\dfrac{H_{*, (2)}(s^\ell_j)}{8} + \dfrac{A}{2}\dfrac{H_{*, (2)}(s^\ell_j)}{M(s^\ell_j)}  \right)  \\
    &\leq (M(s^1_1) + 16 LSA^2)\sum_{\ell=1}^L\sum_{s^\ell_j}\dfrac{M(s^\ell_j)}{n_f}\left(2 + H_{*, (2)}(s^\ell_j)  \right)  \\
    &\overset{(b)}{\leq} (M(s^1_1) + 16 LSA^2)\dfrac{M_{}}{n_f}\left(2 + H_{*, (2)}  \right)  
\end{align*}
where, $(a)$ follows for $1/(x-c) \leq x + c$ for $x^2 \geq 1 + c^2$ and $c>0$. The $(b)$ follows for $M =\sum_{\ell=1}^L\sum_{s^\ell_j}M(s^\ell_j)$, and $H_{*, (2)} =\sum_{\ell=1}^L\sum_{s^\ell_j}H_{*, (2)}(s^\ell_j)$.
%
%
%
%
It follows then by setting 
$n_f = n - n_u$ that 
\begin{align*}
    \E_{\D}\left[\left(Y_{n}(s^1_1)-V_\pi(s^1_1)\right)^2\indic{\xi^C_{Z,K}}\right] &\overset{(a)}{\leq} \left(\dfrac{ (M(s^1_1) + 16 LSA^2)  M\left(2 + H_{*, (2)} \right)}{n_f}\right)^2 n_u \\
    & \overset{(b)}{=} \dfrac{(M(s^1_1) + 16 LSA^2)^2  n_u}{(n-n_u)^2}\left(2 + H_{*, (2)}\right)^2\\
    &\overset{(c)}{\leq } \dfrac{(M(s^1_1) + 16 LSA^2)^2  H_{*, (2)} n}{(n- H_{*, (2)} n)^2}\left(2 + H_{*, (2)}\right)^2 \\
    &\leq \dfrac{M^2(s^1_1)}{n}\left(32 M LSA^2 + H_{*, (2)}\right)^2
\end{align*}
where, $(a)$ follows from \Cref{lemma:wald-variance},  $(b)$ follows from the definition of $H_{*, (2)}$, and $(c)$ follows from \eqref{eq:mdp-oracle-11}.

\textbf{Step 10 (Combine everything):} Combining everything from step 5, step 8 and setting $\delta=1/n^2$ we can show that the MSE of oracle scales as
\begin{align}
    \L_n(\pi, \!\!\bb^k_*) &\!\!\leq \!\!\frac{M^2(s^1_1)}{n} \!\!+ \!\! \frac{8A M^2(s^1_1)}{n^2} \!\!+ \!\!\frac{16A^2 M^2(s^1_1)}{n^3} + \dfrac{M^2(s^1_1)}{n}\left(32 M LSA^2 \!\!+\!\! H_{*, (2)}\right)^2 \!+\! \underbrace{ \E_{\D}\left[\left(Y_{n}(s^1_1)-V_\pi(s^1_1)\right)^2\indic{\xi^C_{c,K}}\right]}_{\textbf{Part C, Safety event does not hold}}\nonumber\\
    &\overset{(a)}{\leq} \frac{M^2(s^1_1)}{n}+\frac{8A M^2(s^1_1)}{n^2} + \frac{16A^2 M^2(s^1_1)}{n^3} + \dfrac{M^2(s^1_1)}{n}\left(32 M LSA^2 + H_{*, (2)}\right)^2  + 2\sum_{t=1}^n\dfrac{2\eta + 4\eta^2}{n^2}\label{eq:mdp-oracle-loss}
\end{align}
where, $(a)$ follows as $\E_{\D}\left[\left(Y_{n}(s^1_1)-V_\pi(s^1_1)\right)^2\indic{\xi^C_{c,K}}\right] \leq 2\eta + 4\eta^2$ and using the low error probability of the constraint event from \Cref{lemma:conc-mu}.
The claim of the proposition follows.
\end{proof}

\subsection{Tree Regret Corollary}
\label{app:mdp-regret-corollary}

\begin{customcorollary}{1}
\label{corollary:mdp-regret}
Under \Cref{assm:tractable-MDP} the constraint regret in the Tree MDP is given by $\overline{\cR}^c_n \leq O\left(\frac{\log(n)}{\bb^{3/2}_{*,\min}n^{3/2}}\right)$ and the regret is given by $\overline{\cR}_n \leq O\left(\frac{\log(n)}{\bb^{3/2}_{*,\min}n^{3/2}}\right)$.
\end{customcorollary}

\begin{proof}
    The upper bound to the safe oracle constraint is given by \eqref{eq:mdp-oracle-11} as follows
    \begin{align*}
        \C^*_n(\pi, \bb^k_*) \leq \dfrac{H_{*, (2)}}{2} \dfrac{n}{M_{\min}} + 16  LSA^2. 
    \end{align*}
    The upper bound to the constraint violation of \sav\ is given by \eqref{eq:mdp-agnostic-11}
    \begin{align*}
        \C_n(\pi, \wb^k) \leq \dfrac{H_{*, (2)}}{2} \dfrac{n}{M_{\min}} + 16  LSA^2 +  O\left(\dfrac{(2\eta + 4\eta^2)L^2S^2A^4 H_{*, (2)}^2 M^2\sqrt{\log (SA n(n+1) / \delta)}}{\min_{s}\bb^{*,k,(3/2)}(s)n^{3/2}}\right).
    \end{align*}
    Hence, from the constraint regret definition, we can show that
    \begin{align*}
        \overline{\cR}^c_n = \C_n(\pi, \wb^k) - \overline{\C}^*_n(\pi, \bb_*) \leq O\left(\frac{\log n}{\bb^{3/2}_{*,\min}n^{3/2}}\right).
    \end{align*}
    Observe that the loss of the agnostic algorithm \sav\ is given by \eqref{eq:mdp-agnostic-loss} and the upper bound to the oracle loss is given by \eqref{eq:mdp-oracle-loss}. Comparing these two losses directly leads to the regret as follows:
    \begin{align*}
    \overline{\cR}_n = \L_n(\pi, \wb^k) - \overline{\L}^*_n(\pi, \bb^k_*) = 
O\left(\dfrac{\log(n)}{\bb^{3/2}_{*,\min}n^{3/2}}\right).
\end{align*}
The claim of the corollary follows.
\end{proof}


\begin{customcorollary}{2}
\label{corollary:bandit-regret}
Under \Cref{assm:tractable-MDP} the constraint regret in the bandit setting is given by $\overline{\cR}^c_n \leq O\left(\frac{\log(n)}{\bb^{3/2}_{*,\min}n^{3/2}}\right)$ and the regret is given by $\overline{\cR}_n \leq O\left(\frac{\log(n)}{\bb^{3/2}_{*,\min}n^{3/2}}\right)$.
\end{customcorollary}

\begin{proof}
The bandit setting consists of a single state, and so we can define the quantity $H_{*, (2)} = \frac{1}{\alpha\mu(0)}\sum_{a\in\A\setminus\{0\}}\pi(a)\sigma(a)\min^+\{\Delta^c(a),\Delta^c(0)-\Delta^c(a) \}$
    The upper bound to the oracle constraint is given by \eqref{eq:mdp-oracle-11} as follows
    \begin{align*}
        \C^*_n(\pi, \bb^k_*) \leq \dfrac{H_{*, (2)}}{2} \dfrac{n}{M_{\min}} + 16  A^2 .
    \end{align*}
    The upper bound to the constraint violation of \sav\ is given by \eqref{eq:mdp-agnostic-11}
    \begin{align*}
        \C_n(\pi, \wb^k) \leq \dfrac{H_{*, (2)}}{2} \dfrac{n}{M_{\min}} + 16  A^2 +  O\left(\dfrac{(2\eta + 4\eta^2)A^4 H_{*, (2)}^2 M^2\sqrt{\log (A n(n+1) / \delta)}}{\min_{s}\bb^{*,k,(3/2)}(s)n^{3/2}}\right).
    \end{align*}
    Hence, from the constraint regret definition, we can show that
    \begin{align*}
        \overline{\cR}^c_n = \C_n(\pi, \wb^k) - \overline{\C}^*_n(\pi, \bb^k_*) \leq O\left(\frac{\log n}{\bb^{3/2}_{*,\min}n^{3/2}}\right).
    \end{align*}
    Observe that the loss of the agnostic algorithm \sav\ is given by \eqref{eq:mdp-agnostic-loss} and the upper bound to the oracle loss is given by \eqref{eq:mdp-oracle-loss}. Comparing these two losses directly leads to the regret as follows:
    \begin{align*}
    \overline{\cR}_n = \L_n(\pi, \wb^k) - \overline{\L}^*_n(\pi, \bb^k_*) = 
O\left(\dfrac{\log(n)}{\bb^{3/2}_{*,\min}n^{3/2}}\right).
\end{align*}
The claim of the corollary follows.
\end{proof}


\section{Support Lemmas}
\label{app:bandit-lower-bound}


\begin{lemma}\textbf{(Hoeffding's Lemma)}\citep{massart2007concentration}
\label{lemma:hoeffding}
Let $Y$ be a real-valued random variable with expected value $\mathbb{E}[Y]= \mu$, such that $a \leq Y \leq b$ with probability one. Then, for all $\lambda \in \mathbb{R}$
$$
\mathbb{E}\left[e^{\lambda Y}\right] \leq \exp \left(\lambda \mu +\frac{\lambda^{2}(b-a)^{2}}{8}\right)
$$
\end{lemma}

\begin{lemma}\textbf{(Concentration lemma 1)}
\label{lemma:conc1}
\label{lemma:conc}
Let $V_{t} = R_t(s, a) - \E[R_t(s, a)]$ and be bounded such that $V_{t}\in[-\eta, \eta]$. Let the total number of times the state-action $(s,a)$ is sampled be $T$.
Then we can show that for an $\epsilon > 0$
\begin{align*}
    \Pb\left(\left|\frac{1}{T}\sum_{t=1}^T R_t(s, a) - \E[R_t(s, a)]\right| \geq \epsilon\right) \leq 2\exp\left(-\frac{2\epsilon^2 T}{\eta^2}\right).
\end{align*}
\end{lemma}

\begin{proof}
Let $V_{t} =R_t(s, a) - \E[R_t(s, a)]$. Note that $\E[V_{t}] = 0$. Hence, for the bounded random variable $V_{t}\in[-\eta, \eta]$  we can show from Hoeffding's lemma in \Cref{lemma:hoeffding} that
\begin{align*}
    \E[\exp\left(\lambda V_{t}\right)] \leq \exp\left(\dfrac{\lambda^2}{8}\left(\eta - (-\eta)\right)^2\right) \leq \exp\left(2\lambda^4\eta^2\right)
\end{align*}
Let $s_{t-1}$ denote the last time the state $s$ is visited and action $a$ is sampled. Observe that the reward $R_t(s,a)$ is conditionally independent and $\eta^2$-sub-Gaussian. 
Next we can bound the probability of deviation as follows:
\begin{align} 
\Pb\left(\sum_{t=1}^T \left(R_t(s, a) - \E[R_t(s, a)]\right) \geq \epsilon\right) &=\Pb\left(\sum_{t=1}^T V_{t} \geq \epsilon\right) \nonumber\\ 
&\overset{(a)}{=}\Pb\left(e^{\lambda \sum_{t=1}^T V_{t}} \geq e^{\lambda \epsilon}\right) \nonumber\\
&\overset{(b)}{\leq} e^{-\lambda \epsilon} \E\left[e^{-\lambda \sum_{t=1}^T V_{t}}\right] \nonumber\\
&=  e^{-\lambda \epsilon} \E\left[\E\left[e^{-\lambda \sum_{t=1}^T V_{t}}\big|s_{T-1}\right] \right]\nonumber\\
&\overset{(c)}{=} e^{-\lambda \epsilon} \E\left[\E\left[e^{-\lambda  V_{T}}|S_{T-1}\right]\E\left[e^{-\lambda \sum_{t=1}^{T-1} V_{t}} \big|s_{T-1}\right]  \right]\nonumber\\
&\leq e^{-\lambda \epsilon} \E\left[\exp\left(2\lambda^4\eta^2\right)\E\left[e^{-\lambda \sum_{t=1}^{T-1} V_{t}}\big |s_{T-1}\right]  \right]\nonumber\\
& \overset{}{=} e^{-\lambda \epsilon} e^{2\lambda^{2} \eta^{2}} \mathbb{E}\left[e^{-\lambda \sum_{t=1}^{T-1} V_{t}}\right] \nonumber\\ 
& \vdots \nonumber\\ 
& \overset{(d)}{\leq} e^{-\lambda \epsilon} e^{2\lambda^{2} T \eta^{2}} \nonumber\\
& \overset{(e)}{\leq} \exp\left(-\dfrac{2\epsilon^2}{T\eta^2}\right) \label{eq:vt0}
\end{align}
where $(a)$ follows by introducing $\lambda\in\mathbb{R}$ and exponentiating both sides, $(b)$ follows by Markov's inequality, $(c)$ follows as $V_{t}$ is conditionally independent given $s_{T-1}$, $(d)$ follows by unpacking the term for $T$ times and $(e)$  follows by taking $\lambda= \epsilon / 4T\eta^2$. Hence, it follows that
\begin{align*}
    \Pb\left(\left|\dfrac{1}{T}\sum_{t=1}^{T} R_t(s, a) - \E[R_t(s, a)]\right| \geq \epsilon\right) = \Pb\left(\sum_{t=1}^T \left(R_t(s, a) - \E[R_t(s, a)]\right) \geq T\epsilon\right) \overset{(a)}{\leq} 2\exp\left(-\frac{2\epsilon^2 T}{\eta^2}\right).
\end{align*}
where, $(a)$ follows by \eqref{eq:vt0} by replacing $\epsilon$ with $\epsilon T$, and accounting for deviations in either direction.
\end{proof}

\begin{lemma}\textbf{(Concentration lemma 2)}
\label{lemma:conc2}
Let $\mu^{2}(s, a)=\mathbb{E}\left[R_{t}^{2}(s, a)\right]$. Let $R_t(s,a)$ be $\eta^2$ sub-Gaussian. Let $n=KL$ be the total budget of state-action samples. Define the event
\begin{align}
\xi_{\delta}=\left(\bigcap_{s\in\S}\bigcap_{1 \leq a \leq A, T_n(s,a) \geq 1}\left\{\left|\frac{1}{T_n(s,a)}\sum_{t=1}^{T_n(s,a)} R_{t}^{2}(s, a)-\mu^{2}(s, a)\right| \leq (2\eta + 4\eta^2) \sqrt{\frac{\log (SA n(n+1) / \delta)}{2 T_n(s,a)}}\right\}\right) \bigcap \nonumber\\
\left(\bigcap_{s\in\S}\bigcap_{1 \leq a \leq A, T_n(s,a) \geq 1}\left\{\left|\frac{1}{T_n(s,a)}\sum_{t=1}^{T_n(s,a)} R_{t}(s, a)-\mu(s, a)\right| \leq (2\eta + 4\eta^2) \sqrt{\frac{\log (SA n(n+1) / \delta)}{2 T_n(s,a)}}\right\}\right)\label{eq:event-xi-delta}
\end{align}
Then we can show that $\Pb\left(\xi_{\delta}\right) \geq 1- 2\delta$.
\end{lemma}

\begin{proof}
First note that the total budget $n = KL$. Observe that the random variable $R^{k}_{t}(s, a)$ and $R^{(2), k}_{t}(s, a)$  
are conditionally independent given the previous state $S^k_{t-1}$. Also observe that for any $\eta>0$ we have that $R^{k}_{t}(s, a), R^{(2),k}_{t}(s, a) \leq 2\eta + 4\eta^2$, where $R^{(2),k}_{t}(s, a) = (R^k_t(s,a))^2$. 
Hence we can show that 
\begin{align*}
    \Pb&\left(\bigcap_{s\in\S}\bigcap_{1 \leq a \leq A, T_n(s,a) \geq 1}\left\{\left|\frac{1}{T_n(s,a)}\sum_{t=1}^{T_n(s,a)} R_{ t}^{2}(s, a)-\mu^{2}(s, a)\right| \geq (2\eta + 4\eta^2) \sqrt{\frac{\log (SA n(n+1) / \delta)}{2 T_n(s,a)}}\right\}\right)\\
    &\leq \Pb\left(\bigcup_{s\in\S}\bigcup_{1 \leq a \leq A, T_n(s,a) \geq 1}\left\{\left|\frac{1}{T_n(s,a)}\sum_{t=1}^{T_n(s,a)} R_{ t}^{2}(s, a)-\mu^{2}(s, a)\right| \geq (2\eta + 4\eta^2) \sqrt{\frac{\log (SA n(n+1) / \delta)}{2 T_n(s,a)}}\right\}\right)\\
    &\overset{(a)}{\leq} \sum_{s=1}^S\sum_{a=1}^A\sum_{t=1}^n\sum_{T_n(s,a)=1}^t2\exp\left(-\dfrac{2T_n}{4(\eta^2 + \eta)^2 }\cdot  \frac{4(\eta^2 + \eta)^2\log (SA n(n+1) / \delta)}{2 T_n(s,a)}\right) = \delta.
\end{align*}
where, $(a)$ follows from \Cref{lemma:conc}. 
Note that in $(a)$ we have to take a double union bound summing up over all possible pulls $T_n$ from $1$ to $n$ as $T_n$ is a random variable. Similarly we can show that
\begin{align*}
    \Pb&\left(\bigcap_{s\in\S}\bigcap_{1 \leq a \leq A, T_n(s,a) \geq 1}\left\{\left|\frac{1}{T_n(s,a)}\sum_{t=1}^{T_n(s,a)} R_{t}(s, a)-\mu(s, a)\right| \geq (2\eta + 4\eta^2) \sqrt{\frac{\log (SA n(n+1) / \delta)}{2 T_n(s,a)}}\right\}\right)\\
    &\overset{(a)}{\leq} \sum_{s=1}^S\sum_{a=1}^A\sum_{t=1}^n\sum_{T_n(s,a)=1}^{t}2\exp\left(-\dfrac{2 T_n}{4(\eta^2 + \eta)^2 }\cdot  \frac{4(\eta^2 + \eta)^2\log (SA n(n+1) / \delta)}{2 T_n(s,a)}\right) = \delta.
\end{align*}
where, $(a)$ follows from \Cref{lemma:conc}. 
Hence, combining the two events above we have the following bound 
$$
\Pb\left(\xi_{\delta}\right) \geq 1- 2\delta.
$$
\end{proof}

\begin{customcorollary}{3}
\label{corollary:conc-var}
Under the event $\xi_\delta$ in \eqref{eq:event-xi-delta} we have for any state-action pair in an episode $k$ the following relation with  probability greater than $1-\delta$
\begin{align*}
    |\wsigma^{k}_t(s,a) - \sigma(s,a)|\leq (2\eta + 4\eta^2) \sqrt{\frac{\log (SA n(n+1) / \delta)}{2 T^K_L(s,a)}}.
\end{align*}
where, $T^K_L(s,a)$ is the total number of samples of the state-action pair $(s,a)$ till episode $k$.
\end{customcorollary}
\begin{proof}
Observe that the event $\xi_\delta$ bounds the sum of rewards $R^k_t(s,a)$ and squared rewards $R^{k,(2)}_t(s,a)$ for any $T^K_L(s,a) \geq 1$. Hence we can directly apply the \Cref{lemma:conc2} to get the bound.
\end{proof}

\begin{lemma}
\label{lemma:conc-mu}
Let $\mu^{c}(s, a)=\mathbb{E}\left[C_{t}^{}(s, a)\right]$ and $C_t(s,a) \leq 2\eta$. Define the event
\begin{align}
    \overline{\xi}_{\delta} = \bigcap_{s\in\S}\bigcap_{1 \leq a \leq A, T_n(s,a) \geq 1}\left\{\left|\frac{1}{T_n(s,a)}\sum_{t=1}^{T_n(s,a)} C_{t}(s, a)-\mu^c(s, a)\right| \leq (2\eta + 4\eta^2) \sqrt{\frac{\log (SA n(n+1) / \delta)}{2 T_n(s,a)}}\right\}. \label{eq:cost-event}
\end{align}
Then we can show that $\Pb(\overline{\xi}_{\delta})\geq 1-\delta$.
\end{lemma}
\begin{proof}
We can show that
\begin{align*}
    \Pb&\left(\bigcap_{s\in\S}\bigcap_{1 \leq a \leq A, T_n(s,a) \geq 1}\left\{\left|\frac{1}{T_n(s,a)}\sum_{t=1}^{T_n(s,a)} C_{t}(s, a)-\mu^c(s, a)\right| \geq (2\eta + 4\eta^2) \sqrt{\frac{\log (SA n(n+1) / \delta)}{2 T_n(s,a)}}\right\}\right)\\
    &\overset{(a)}{\leq} \sum_{s=1}^S\sum_{a=1}^A\sum_{t=1}^n\sum_{T_n(s,a)=1}^{t}2\exp\left(-\dfrac{2 T_n(s,a)}{4(\eta^2 + \eta)^2 }\cdot  \frac{4(\eta^2 + \eta)^2\log (SA n(n+1) / \delta)}{2 T_n(s,a)}\right) = \delta.
\end{align*}
where, $(a)$ follows from \Cref{lemma:conc} when applied for cost. The claim of the lemma follows.
\end{proof}

\begin{customcorollary}{4}
\label{corollary:safe-conc}
Let the total exploration budget be $n_x = \frac{SA \log(SAn(n+1)/\delta) }{\min_{s,a}\Delta^{c,(2)}(s,a)}$. Define the event $\overline{\xi}_{\delta}$ as in \eqref{eq:cost-event}. Then using the exploration policy $\pi_x$ it can be shown that $\Pb(\overline{\xi}_{\delta})\geq 1 - \delta$.
\end{customcorollary}
\begin{proof}
Let $n_x = \frac{SA \log(SAn(n+1)/\delta) }{\min_{s,a}\Delta^{c,(2)}(s,a)}$ be the total samples taken for exploration. Let $\pi_e$ sample each action according to uniform random policy in each state $s\in[S]$. Then the result follows directly from \Cref{lemma:conc-mu} in
\begin{align*}
     \Pb&\left(\bigcap_{s\in\S}\bigcap_{1 \leq a \leq A, T_{n_x}(s,a) \geq 1}\left\{\left|\frac{1}{T_{n_x}(s,a)}\sum_{t=1}^{T_{n_x}(s,a)} C_{t}(s, a)-\mu^c(s, a)\right| \geq (2\eta + 4\eta^2) \sqrt{\frac{\log (SA n(n+1) / \delta)}{2 T_{n_x}}}\right\}\right) \overset{(a)}{\leq} \delta,
\end{align*}
where, $(a)$ follows as by noting $T_{n_x}\geq \frac{\log(SAn(n+1)/\delta) }{\min_{s,a}\Delta^{c,(2)}(s,a)}$.




\end{proof}

\section{Additional Experimental Details}
\label{app:addl-expt-1}
In this section we state additional experimental details.



\textbf{Experiment 1 (Bandit):} We implement a bandit environment for $A=11$ and show that our proposed solution outperforms the safe on-policy and SEPEC \citep{wan2022safe} algorithm. In this experiment we have the $\mu(0) = 0.5, \sigma^2(0) = 10^{-4}$, $\mu(1) = 0.9, \sigma^2(1) = 10^{-4}$ (optimal action), and the sub-optimal actions $a\in\{2,3,\ldots,11\}$ have means $\mu(a)\in[0.02,0.03]$ and high variance $\sigma^2(a) = 40$. Moreover, we set the constraint-value means $\mu^c(a)$ the same as the reward means. The target policy is initialized as $\pi(0) = \pi(1) = 0.4$ while the remaining arms have the $0.2$ density evenly distributed among them. So in this environment, the safe on-policy will select the sub-optimal actions less and so reduces MSE at a slower rate. Whereas the \sav, complies with the safety constraint and reduces MSE maximally as the number of rounds increases. The performance is shown in \Cref{fig:expt} (left). Again observe that in \Cref{app:fig-expt-2} (top-left), the oracle keeps the safety budget around $0$ and uses all the remaining samples to explore optimally. The \sav\ has a safety budget of almost around $0$ as they sample the high cost maximizing action $1$ a sufficient number of times to offset the unsafe action pulls. However, safe on-policy and SEPEC again explores the high variance (sub-optimal and unsafe) actions less and has a very high safety budget.

\textbf{Experiment 2 (Movielens):} We conduct this experiment on Movielens dataset for $A=30$ actions and show that our proposed solution outperforms safe on-policy and SEPEC algorithm. The Movielens dataset from February 2003 consist of 6k users who give 1M ratings to 4k movies. We obtain a rank-$4$ approximation of the dataset over $128$ users and $128$ movies such that all  users prefer either movies $7$, $13$, $16$, or $20$ ($4$ user groups). The movies are the actions and we choose $30$ movies that have been rated by all the users. Hence, this testbed consists of $30$ actions and the mean values $\mu(a)$ are the rating of the movies given by the users. and is run over $T = 8000$. The target policy is initialized as $\pi(0) = \pi(1) = 0.4$ while the remaining arms has the $0.2$ density evenly distributed among them. We set the cost means $\mu_c(a)$ such that high variance actions have high-cost means. So in this environment, the safe on-policy will select the sub-optimal cost actions less and so reduces MSE at a slower rate as the number of rounds increases. The SEPEC MSE also reduces slower than \sav\ as the number of rounds increases. This is because SEPEC uses an IPW estimator instead of tracking the optimal behavior policy like \sav. The \sav, complies with the safety constraint and reduces MSE maximally as the number of rounds increases. The performance is shown in \Cref{fig:expt} (middle-left). Again observe that in \Cref{app:fig-expt-2} (top-right), the oracle keeps the safety budget around $0$ and uses all the remaining samples to explore optimally. The \sav\ has a safety budget of almost around $0$ as they sample the high reward maximizing action $1$ a sufficient number of times to offset the unsafe action pulls. However, safe on-policy and SEPEC again explores the high variance (sub-optimal and unsafe) actions less and has a very high safety budget.

\textbf{Experiment 3 (Tree):} We experiment with a $4$-depth $2$-action deterministic tree MDP $\T$ consisting of $15$ states. 
In this setting, we have a $4$-depth $2$-action deterministic tree MDP $\T$ consisting of $15$ states. Each state has a low variance arm with $\sigma^2(s,1) = 0.01$ and high target probability $\pi(1|s) = 0.95$ and a high variance arm with $\sigma^2(s,1) = 20.0$ and low target probability $\pi(2|s) = 0.05$. Again we set the cost means $\mu^c(a)$ such that high variance actions have high-cost means. Hence, the safe on-policy sampling which samples according to $\pi$ will sample the second (high variance) arms less and suffer a high MSE. 
%
We set $\alpha = 0.25$.
We assume that the learner can directly access the $V^{\pi_0}(s^1_1)$ (without any noise) when its safety budget is negative. It can observe $V^{\pi_0}(s^1_1)$ without running any episodic interaction (like \citet{yang2021reduction}.
The oracle has access to the model and variances and performs the best. \sav\ lowers MSE comparable to safe onpolicy as the number of episodes increases and eventually matches the oracle's MSE in \Cref{fig:expt} (middle-right). The \sav, oracle, and on-policy have an almost equal safety budget as shown in \Cref{app:fig-expt-2} (bottom-left). Note that we do not run SEPEC in this experiment as it is a bandit algorithm, and the optimization problem of SEPEC do not have a closed form solution in the MDP setting.


\textbf{Experiment 4 (Gridworld):} In this setting we have a $4\times 4$ stochastic gridworld consisting of $16$ grid cells. Considering the current episode time-step as part of the state, this MDP is a DAG MDP in which there is multiple paths to a single state. There is a single starting location at the top-left corner and a single terminal state at the bottom-right corner. Let $\mathbf{L}, \mathbf{R}, \mathbf{D}, \mathbf{U}$ denote the left, right, down, and up actions in every state. Then in each state, the right and down actions have low variance arms with $\sigma^2(s,\mathbf{R}) = \sigma^2(s,\mathbf{D}) = 0.01$ and high target policy probability $\pi(\mathbf{R}|s) = \pi(\mathbf{D}|s) = 0.45$. The left and top actions have high variance arms with $\sigma^2(s,\mathbf{L}) = \sigma^2(s,\mathbf{U}) = 0.01$ and low target policy probability $\pi(\mathbf{L}|s) = \pi(\mathbf{U}|s) = 0.05$. We set the cost means $\mu^c(a)$ such that high variance actions have high-cost means. Hence, safe onpolicy which goes right and down with high probability (to reach the terminal state) will sample the low variance arms more and suffer a high MSE. 
We set $\alpha = 0.25$.
Again we assume that the learner can directly access the $V^{\pi_0}(s_1)$ (without any noise) when it's safety budget is negative. It can observe $V^{\pi_0}(s_1)$ without running any episodic interaction (like \citet{yang2021reduction}.
%
%
\sav\ lowers MSE faster compared to safe onpolicy and actually matches MSE compared to the oracle as well as maintains the safety constraint with increasing number of episodes. 
We point out that the DAG structure of the Gridworld violates the tree structure under which the oracle and \sav bounds were derived. Nevertheless, both methods lower MSE compared to safe onpolicy. Again observe that in \Cref{app:fig-expt-2} (bottom-right), the oracle keeps the safety budget around $0$ and uses all the remaining samples to explore optimally. The \sav has a safety budget of almost around $0$ as they sample the high reward maximizing action a sufficient number of times to offset the unsafe action pulls. However, safe on-policy again explores the high variance (sub-optimal and unsafe) actions less and has a very high safety budget.


\newpage
\section{Table of Notations}
\label{table-notations}

\begin{table}[!tbh]
    \centering
    \begin{tabular}{|p{10em}|p{33em}|}
        \hline\textbf{Notations} & \textbf{Definition} \\\hline
        $s^{\ell}_i$ & State $s$ in level $\ell$ indexed by $i$ \\\hline
        $\pi(a|s^{\ell}_i)$  & Target policy probability for action $a$ in $s^{\ell}_i$ \\\hline
        $b(a|s^{\ell}_i)$  & Behavior policy probability for action $a$ in $s^{\ell}_i$ \\\hline
        $\sigma^2(s^{\ell}_i, a)$  & Variance of action $a$ in $s^{\ell}_i$ \\\hline
        $\wsigma^{(2),k}_{t}(s^{\ell}_i, a)$  & Empirical variance of action $a$ in $s^{\ell}_i$ at time $t$ in episode $k$\\\hline
        $\usigma^{(2),k}_{t}(s^{\ell}_i, a)$  & UCB on variance of action $a$ in $s^{\ell}_i$ at time $t$ in episode $k$\\\hline
        $\mu(s^{\ell}_i, a)$  & Mean of action $a$ in $s^{\ell}_i$\\\hline
        $\wmu^{k}_{t}(s^{\ell}_i, a)$  & Empirical mean of action $a$ in $s^{\ell}_i$ at time $t$ in episode $k$\\\hline
        $\mu^{2}(s^{\ell}_i, a)$  & Square of mean of action $a$ in $s^{\ell}_i$\\\hline
        $\wmu^{(2),k}_{t}(s^{\ell}_i, a)$  & Square of empirical mean of action $a$ in $s^{\ell}_i$ at time $t$ in episode $k$\\\hline
        $T_n(s^{\ell}_i, a)$  & Total Samples of action $a$ in $s^{\ell}_i$ after $n$ timesteps\\\hline
        $T_n(s^{\ell}_i)$  & Total samples of actions in $s^{\ell}_i$ as $\sum_a T_n(s^{\ell}_i, a)$ after $n$ timesteps (State count)\\\hline
        $T^k_t(s^{\ell}_i,a)$ & Total samples of action $a$ taken till episode $k$ time $t$ in $s^{\ell}_i$\\\hline
        $T^k_t(s^{\ell}_i,a,s^{\ell+1}_j)$ & Total samples of action $a$ taken till episode $k$ time $t$ in $s^{\ell}_i$ to transition to $s^{\ell+1}_j$\\\hline
        $P(s^{\ell+1}_j|s^{\ell}_i,a)$  & Transition probability of taking action $a$ in state $s^{\ell}_i$ and transition to state $s^{\ell+1}_j$\\\hline
        & $ \sum_a\sqrt{\pi^2(a|s^{\ell}_{i})\sigma^2(s^{\ell}_{i}, a)}, \text{ if } \ell = L$\\ 
        $M(s^{\ell}_{i})\coloneqq\begin{cases}\vspace{3em}\end{cases}$  & $\sum_a \sqrt{\sum\limits_{s^{\ell+1}_j}\pi^2(a|s^{\ell}_{i})\left(\sigma^2(s^{\ell}_{i}, a) + P(s^{\ell+1}_j|s^{\ell}_i, a) B^2(s^{\ell+1}_{j})\right)}, \text{ if } \ell \!\!\neq\!\! L$\\\hline
        & $ \sum_a\sqrt{\pi^2(a|s^{\ell}_{i})\wsigma^{(2),k}_t(s^{\ell}_{i}, a)}, \text{ if } \ell = L$\\ 
        $\wM(s^{\ell}_{i})\coloneqq\begin{cases}\vspace{3em}\end{cases}$  & $\sum_a \sqrt{\sum\limits_{s^{\ell+1}_j}\pi^2(a|s^{\ell}_{i})\left(\wsigma^{(2),k}_t(s^{\ell}_{i}, a) +  P^{}(s^{\ell+1}_j|s^{\ell}_i, a) \wB^{(2),k}_{t}(s^{\ell+1}_j)\right)}, \text{ if } \ell \!\!\neq\!\! L$\\\hline
    \end{tabular}
    \vspace{1em}
    \caption{Table of Notations}
    \label{tab:my_label}
\end{table}

\end{document}